\documentclass{article}

\PassOptionsToPackage{numbers, compress}{natbib}
\usepackage[final]{neurips_2020}

\usepackage[utf8]{inputenc} % allow utf-8 input
\usepackage[T1]{fontenc}    % use 8-bit T1 fonts
\usepackage{hyperref}       % hyperlinks
\usepackage{url}            % simple URL typesetting
\usepackage{booktabs}       % professional-quality tables
\usepackage{amsfonts}       % blackboard math symbols
\usepackage[dvipdfmx,hiresbb]{graphicx} % more modern
\usepackage{subcaption}
\usepackage{color}

\usepackage{algorithm}
\usepackage{algorithmic}
\usepackage{comment}

\usepackage{hyperref}
\usepackage{amsmath,amssymb}
\usepackage{amsbsy}
\usepackage{bm}
\usepackage{amsthm}

\newcommand{\argmin}{\mathop{\mathrm{argmin}}}

\newcommand{\diag}[1]{\mathrm{diag}\left(#1\right)}

\newcommand{\Eqref}[1]{Eq. \eqref{#1}}

\newcommand{\boldone}{{\boldsymbol{1}}}

\newcommand{\calA}{{\mathcal A}}
\newcommand{\calB}{{\mathcal B}}
\newcommand{\calC}{{\mathcal C}}

\newcommand{\calE}{{\mathcal E}}
\newcommand{\calF}{{\mathcal F}}
\newcommand{\calG}{{\mathcal G}}
\newcommand{\calH}{{\mathcal H}}

\newcommand{\calL}{{\mathcal L}}

\newcommand{\calN}{{\mathcal N}}

\newcommand{\calX}{{\mathcal X}}

\newcommand{\fhat}{\widehat{f}}

\newcommand{\Px}{P_{X}}
\newcommand{\Py}{P_{Y}}
\newcommand{\LPi}{L_2} %(P_{\calX})}
\newcommand{\LPiPx}{L_2(\Px)}
\newcommand{\Real}{\mathbb{R}}
\newcommand{\Natural}{\mathbb{N}}

\newcommand{\EE}{\mathrm{E}}
\newcommand{\dd}{\mathrm{d}}
\newcommand{\fstar}{f^{\ast}}

\newcommand{\gstar}{g^*}

\def\I<#1>{\left\langle #1 \right\rangle}
\def\i<#1>{\left\langle #1 \right\rangle}

\newcommand{\Var}{\mathrm{Var}}

\newcommand{\supp}{\mathrm{supp}}

\newcommand{\calLhat}{\widehat{\calL}}
\newcommand{\idmap}{\mathbb{I}}
\newcommand{\cH}{\mathcal{H}}
\newcommand{\cHK}{\mathcal{H}_K}
\newcommand{\calW}{\mathcal{W}}
\newcommand{\norm}[1]{\|#1\|}
\newcommand{\sign}{\mathrm{sign}}
\newcommand{\Wstar}{W^*}
\newcommand{\betas}{s}
\newcommand{\cHKtil}{\mathcal{H}_{\tilde{K}}}
\newcommand{\nutil}{\tilde{\nu}_{\beta}}
\newcommand{\cHKtilt}{\mathcal{H}_{\tilde{K}^\theta}}
\newcommand{\alphatil}{\tilde{\alpha}}

\newcommand{\ustar}{u^*}
\newcommand{\lambdabeta}{\lambda_\beta}
\newcommand{\epsilonstar}{{\epsilon^*}}

\newcommand{\calWt}{\widetilde{\calW}}
\newcommand{\hstar}{{h^*}}
\newcommand{\CLRD}{C_{(L,R,D)}}

\allowdisplaybreaks
\newtheorem{Theorem}{Theorem}
\newtheorem{Lemma}{Lemma}

\newtheorem{Assumption}{Assumption}

\newtheorem{Proposition}{Proposition}
\newtheorem{Remark}{Remark}

\usepackage{enumitem}
\newlist{assumenum}{enumerate}{1} % also creates a counter called 'assumenum'
\setlist[assumenum]{label={\rm (\roman*)}, ref=(\roman*)}%ref=\thetheorem~(\roman*)}
\newcommand{\Tr}{\mathrm{Tr}}

\newcommand{\Rbar}{\bar{R}}

\usepackage[textwidth=2.5cm, textsize=tiny]{todonotes}
\title{
Generalization bound of globally optimal non-convex neural network training: Transportation map estimation by infinite dimensional Langevin dynamics
}
\author{
Taiji Suzuki \\
The University of Tokyo, Tokyo, Japan \\
RIKEN Center for Advanced Intelligence Project, Tokyo, Japan \\
\texttt{taiji@mist.i.u-tokyo.ac.jp} \\
}

\begin{document}

\maketitle

%!TEX root = NIPS2020_supplementary.tex

\begin{abstract}
We introduce a new theoretical framework to analyze deep learning optimization with connection to its generalization error.
Existing frameworks such as mean field theory and neural tangent kernel theory 
for neural network optimization analysis 
typically require taking limit of infinite width of the network to show its global convergence.
This potentially makes it difficult to directly deal with finite width network; especially in the neural tangent kernel regime, we cannot reveal favorable properties of neural networks beyond kernel methods. 
To realize more natural analysis, we consider a completely different approach in which 
we formulate the parameter training as a transportation map estimation and show its global convergence via the theory of the {\it infinite dimensional Langevin dynamics}.
This enables us to analyze narrow and wide networks in a unifying manner.
Moreover, we give generalization gap and excess risk bounds for the solution obtained by the dynamics.
The excess risk bound achieves the so-called fast learning rate. 
In particular, we show an exponential convergence for a classification problem and a minimax optimal rate for a regression problem.

%for deep learning optimization theory
%not only the global optimality but also a nice generalization ability of trained network. 
%Indeed, we present generalization gap and excess risk bounds.

\end{abstract}

\section{Introduction}

Despite the extensive empirical success of deep learning, there are several missing issues in theoretical understanding of its optimization and generalizations. Even though there are several theoretical analyses on its generalization error and  representation ability \cite{COLT:Neyshabur+Tomioka+Srebro:2015,bartlett2017spectrally,pmlr-v80-arora18b,suzuki2020compression,AoS:Schmidt-Hieber:2020},
they are not necessarily well connected with an optimization procedure. 
The biggest difficulty in neural network optimization lies in its non-convexity.
Recently, this difficulty of non-convexity is partly resolved by considering infinite width limit of networks as performed in 
{\it mean field theory} \cite{sirignano2018mean,MeiE7665} and {\it Neural Tangent Kernel} (NTK) \cite{jacot2018neural,du2019gradient}.
These analyses deal with different scaling of parameters for taking the limit of the width, but they share a similar spirit that an appropriate gradient descent direction can be found in an over-parameterized setting until convergence.
%Anyway, to make them valid, the width must diverge against the sample size.

The mean field analysis formulates the neural network training as a gradient flow in the space of probability measures over the weights.
The gradient flow corresponding to a deterministic dynamics of the weights can be analyzed 
%in a framework so called {\it particle gradient descent} 
as an interacting particle system
\cite{nitanda2017prticle,chizat2018note,NIPS:Rotskoff&Vanden-Eijnden:2018,rotskoff2019trainability}. 
%This type of analysis characterizes the local optimality of the current solution by inclusive relationship between the support of the corresponding distribution  and the optimal solution.
On the other hand, a stochastic dynamics of an interacting particle system can be formulated as McKean–Vlasov dynamics,
and convergence to the global optimal is ensured by the ergodicity of this dynamics
\cite{MeiE7665,pmlr-v99-mei19a}.
Intuitively, inducing stochastic noise makes the solution easier to get out of local optimal and facilitates convergence to the global optimal.
%The McKean-Vlasov type analysis can be conducted not only just an ordinal stochastic gradient descent but also can be adapted to the optimal control theory. 
%\cite{weinan2019mean,tzen2020meanfield,lu2020meanfield} analyzed the optimal solution path derived from the Hamilton-Jacobi-Bellman equation for training 2 layer neural network or ResNet.

The second regime, NTK, deals with larger scaling than the mean field regime, and the gradient descent dynamics is approximated by that in the tangent space at the initial solution \cite{jacot2018neural,du2018gradient,allen2019convergence,du2019gradient,arora2019fine}.
That is, in the wide limit of the neural network, the gradient descent can be seen as that in an reproducing kernel Hilbert space (RKHS) corresponding to the neural tangent kernel, which resolves the difficulty of non-convexity.  
Actually, it is shown that the gradient descent converges to the zero error solution exponentially fast for a sufficiently large width network \cite{du2018gradient,allen2019convergence,du2019gradient}. %\cite{arora2019exact}
In addition to the optimization, its generalization error has been also extensively studied in the NTK regime \cite{du2018gradient,allen2019convergence,du2019gradient,weinan2019comparative,cao2019generalization,cao2019generalization_b,zou2019improved,oymak2020towards,nitanda2019refined,ji2019polylogarithmic}.
%\cite{weinan2019comparative} analyzed NTK to derive a generalization error bound for noisy observation in a student-teacher setting. 
%Weinan,...: Regression
%As for the classification, the requirement on the width can be much milder \cite{cao2019generalization,cao2019generalization_b}.
%Actually, \cite{nitanda2019refined,ji2019polylogarithmic} showed overparameterization is not necessary for classification with positive margin (poly-log of sample size is sufficient).
%However, all of them requires the width must to ensure the global convergence.
On the other hand, \cite{HAYAKAWA2020343} pointed out that non-convexity of a deep neural network model is essential to show superiority of deep learning over linear estimators such as kernel methods as in the analysis of \cite{suzuki2018adaptivity,pmlr-v89-imaizumi19a,suzuki2019deepIntrinsicDim}. Therefore, the NTK regime would not be  appropriate to show superiority of deep learning over other methods such as kernel methods. 

The above mentioned researches opened up new directions for analyzing deep learning optimization.
However, all of them require that the width should diverge as the sample size goes up to show the global convergence and obtain generalization error bounds.
%This is because considering asymptotics is  analysis alleviate the difficulty of non-convexity. 
On the other hand, a convergence guarantee for ``fixed width'' training is still difficult and we have not obtained a satisfactory result that can bridge both of under-parameterized and over-parameterized settings in a {\it unifying manner}.
%This limitation stems from non-convex optimization.
One way to tackle non-convexity in a finite width situation would be stochastic gradient Langevin dynamics (SGLD) \cite{Welling_Teh11,Raginsky_Rakhlin_Telgarsky2017,NIPS2018_8175}.
This would be useful to show the global convergence for the non-convex optimization in deep leaning.
However, the convergence rate depends exponentially to the dimensionality, which is not realistic to analyzing neural network training that typically requires huge parameter size.

%\paragraph{Our contribution}
{\bf Our contribution:} In this paper, we resolve these difficulties such as (i) diverging width against sample size and (ii) curse of dimensionality for analyzing Langevin dynamics in neural network training by 
formulating the neural network training as a {\it transport map} estimation problem of the parameters.
%\citep{DaPrato_Zabczyk92,Maslowski:1989,Sowers:1992} %existence uniqueness of invariant measure
%\citep{Jacquot+Gilles:1995,Shardlow:1999,Hairer:2002} % exponential convergence
By doing so, we can deal with finite width and infinite width in a unifying manner.
We also give a generalization error bound for the solution obtained by our optimization formulation and further show %generalization is guaranteed and 
that it achieves {\it fast learning rate} in a well-specified setting.
The preferable generalization error heavily relies on similarity between a {\it nonparametric Bayesian Gaussian process estimator} and the Langevin dynamics. More details are summarized as follows:

{\vspace{-0.15cm}
\setlength{\leftmargini}{0.8cm}
\begin{itemize}
\item {\bf (formulation)} We formulate neural network training as a transportation map learning of weights (parameters) and solve this problem by infinite dimensional gradient Langevin dynamics in RKHS \cite{da_prato_zabczyk_2014,muzellec2020dimensionfree}.
This formulation has a wide range of applications including two layer neural network, ResNet, Wasserstein optimal transportation map estimation and so on. 
\item {\bf (optimization)} Based on this formulation, we show its global convergence for finite width and infinite width in a unifying manner. 
We give its size independent convergence rate. % of the dynamics. %it. %the optimization.
%\item {\bf (optimization)} 
\item {\bf (generalization)} We derive the generalization error bound of the estimator obtained by our optimization framework. We also derive the fast learning rate in a student-teacher setup. 
%and show superiority of deep learning to linear estimators. 
Especially, we show exponential convergence for classification. % and fast rate for regression. 
%This is the first result to show a fast learning rate without any scale limit.
\end{itemize}
%Our analysis is also used for training Wasserstein optimal transportation map.
%\TS{Strengthen justification, superiority over linear estimator}\\
%\TS{Related work}
%\subsection{Other related work}
%Entropy SGD, PAC-Bayes
%
%chizat sparse
}

\section{Problem setting and model: Training parameter transportation map}
In this section, we give the problem setting and notations that will be used in the theoretical analysis.
%\subsection{Problem setting}
Basically, we consider the standard supervised leaning where data consists of input-output pairs $z=(x,y)$ where $x \in \Real^d$ is an input  and $y \in \Real$ is an output (or label). 
%A function $f:\Real^d \to \Real$ is trained to predict the output $y$ well.
%We consider a single output setting, i.e., the output $y$ is a 1-dimensional real value, but
%it is straight forward to generalize the result to a multiple output case.
We may also consider a unsupervised learning setting, but just for the presentation simplicity, we consider a supervised learning.
Suppose that we are given $n$ i.i.d. observations $D_n = (x_i,y_i)_{i=1}^n$ distributed from a probability distribution $P$,
the marginal distributions of which with respect to $x$ and $y$ are denoted by $\Px$ and $P_Y$ respectively.
We denote $\calX = \supp(\Px)$.
%the marginal distribution of $x$ is denoted by $\Px$ and the one corresponding to $y$ is denoted by $\Py$.
To measure the performance of a trained function $f$, we use a loss function $\ell:\Real \times \Real \to \Real~((y,f)\mapsto \ell(y,f))$ and 
define the expected risk and the empirical risk as 
$
\calL(f) := \EE_{Y,X}[\ell(Y,f(X))]$ and $\calLhat(f) := \frac{1}{n} \sum_{i=1}^n \ell(y_i,f(x_i))$ respectively.
As in the standard deep learning, we optimize the training risk $\calLhat$. 
Our theoretical interest is to bound the following errors for an estimator $\fhat$:  
\begin{equation*}
\text{Excess risk:}~~\calL(\fhat) - \inf_{f:\text{measurable}} \calL(f),~~~~~\text{ Generalization gap:}~~\calL(\fhat) - \calLhat(\fhat).
\end{equation*}
In a typical situation, the generalization gap is bounded as $O(1/\sqrt{n})$ via VC-theory type analysis \cite{mohri2012foundations}, for example.
On the other hand, the excess risk can be faster than $O(1/\sqrt{n})$, which is known as a {\it fast learning rate}
\citep{IEEEIT:Mendelson:2002,LocalRademacher,Koltchinskii,gine2006concentration}.
The population $L_2$-norm with respect to $P$ is denoted by $\|f\|_{\LPi} := \sqrt{\EE_{Z\sim P}[f(Z)^2]}$
and the sup-norm on the domain of the input distribution $\Px$
is denoted by $\|f\|_\infty := \sup_{x \in \supp(\Px)}|f(x)|$.
%We denote the empirical $L_2$-norm by $\|f\|_n := \sqrt{\sum_{i=1}^n f(z_i)^2/n}$ for an empirical observation $z_i = (x_i,y_i)~(i=1,\dots,n)$.
%The population $L_2$-norm is denoted by $\|f\|_{\LPi}:=\sqrt{\EE_{Z\sim P}[f(Z)^2]}$.

%Definitions of $\|\cdot\|_\infty$ and $\|\cdot\|_p$

\subsection{Introductory setting: mean field training of two layer neural network}
Here, we explain the motivation of our theoretical framework by introducing mean field analysis of two layer neural networks. Let us consider the following two layer neural network model: 
\begin{align} \label{eq:finitewidthmodel}
\textstyle
f_{\Theta}(x) =  \frac{1}{M} \sum_{m=1}^M a_m \sigma(w_m^\top x).
\end{align}
where $\sigma:\Real \to \Real$ is a smooth activation function, $(a_m)_{m=1}^M \subset \Real$ is the set of weights in the second layer which we assume is fixed for simplicity, and $\Theta = (w_m)_{m=1}^M \subset \Real^d$ is the set of weights in the first layer. We aim to minimize the following regularized empirical risk with respect to $\Theta$ and analyze the dynamics of gradient descent updates:
\begin{align*}
\textstyle
\min_{\Theta}~~\calLhat(f_{\Theta}) + \frac{\lambda}{2M} \sum_{m=1}^M \|w_m\|^2.
\end{align*}
The stochastic gradient descent (SGD) update for optimizing $\calLhat(f_\Theta)$ with respect to $\Theta$ is reduced to 
\begin{align}\label{eq:DiscreteTimeDiscreteSpaceEvolution}
%\textstyle 
w_m^{(t+1)}  = w_m^{(t)} - \eta( \tfrac{\lambda}{M} w_m^{(t)} + \nabla_{w_m} \calLhat(f_{\Theta^{(t)}}) )+ %\sqrt{\frac{\eta_t}{\beta}} 
\sqrt{2\eta/\beta} \epsilon_t^{(m)},
% \\
%a_m^{(t+1)} & = a_m^{(t)} - \eta_t \nabla_{a_m} \calLhat(f_{\Theta^{(t)}}) + \sqrt{\frac{\eta_t}{\beta}} \epsilon_t^{(a_m)}. % \\
\end{align}
where $\nabla_{w_m} \calLhat(f_{\Theta^{(t)}}) = \frac{a_m}{M}  \frac{1}{n} \sum_{i=1}^n x_i \sigma'(w_m^{(t)\top} x_i) \ell'(y_i,f_{\Theta^{(t)}}(x_i))$ and $\epsilon_t^{(m)}$ is an i.i.d. Gaussian noise mimicking the deviation of the stochastic gradient. 
Here, $\eta > 0$ is a step size and $\beta > 0$ is an inverse temperature parameter.
This could be time discretized version of the following continuous time stochastic differential equation (SDE): 
%\begin{align*}
%& 
$$
\dd w_m(t) = - \big( \tfrac{\lambda}{M} w_m(t) + \nabla_{w_m(t)} \calLhat(f_{\Theta^{(t)}})\big) \dd t + \sqrt{2\eta/\beta} \dd B_t^{(m)},
$$
%\\
%& \dd a_m(t)  = - \nabla_{w_m(t)} \calLhat(f_{\Theta^{(t)}}) \dd t + \sqrt{\frac{\eta_t}{\beta}} \dd B_t^{(a_m)},
%\end{align*}
where $(B_t^{(m)})_{t}$ %and $(B_t^{(a_m)})_{t}$ are 
is a $d$-dimensional Brownian motion. % and $1$-dimensional one respectively.
In the mean field analysis, this optimization process is casted to an optimization of probability distribution over the parameters \cite{MeiE7665,pmlr-v99-mei19a,nitanda2017prticle,chizat2018note} based on the following integral representation of neural networks:
\begin{align}\label{eq:IntegralRep}
%\textstyle
f_\rho(x) := \int_{\Real^{d}} a \sigma(w^\top x)  \dd \rho(w),
\end{align}
where $\rho$ is a Borel probability measure defined on the parameter space $\Real^{d}$
and the parameter in the second layer is fixed to a constant $a \in \Real$ just for presentation simplicity.
%The finite sum dynamics is a space discretized version of so called {\it McKean–Vlasov dynamics}.
The time evolution of the distribution $\rho$ is deduced from the optimization dynamics with respect to each ``particle'' given by 
\begin{align*}
& %\textstyle 
\dd W(t) = -\Big(\lambda W(t) + a \frac{1}{n} \sum_{i=1}^n x_i \sigma'(W(t)^\top x_i)  \ell'(y_i, f_{\rho_t}(x_i))\Big) \dd t + \sqrt{\beta^{-1}} \dd B_t,
%\nabla_{w_m(t)} \calLhat(f_{\Theta^{(t)}}) \dd t + \sqrt{\frac{\eta_t}{\beta}} \dd B_t^{(m)}
%\nabla_{w_m(t)} \calLhat(f_{\Theta^{(t)}}) \dd t + \sqrt{\frac{\eta_t}{\beta}} \dd B_t^{(m)}, %\\
%& \dd a_m(t)  = - \nabla_{w_m(t)} \calLhat(f_{\Theta^{(t)}}) \dd t + \sqrt{\frac{\eta_t}{\beta}} \dd B_t^{(a_m)},
\end{align*}
where $\rho_t$ is the probability law of $W(t) \in \Real^d$ with an initial distribution $W(0) \sim \rho_0$,
which is one of the {\it McKean-Vlasov} processes.
We can see that this equation is space-time continuous limit of the update \Eqref{eq:DiscreteTimeDiscreteSpaceEvolution}.
Importantly, $\rho_t$ admits a density function $\pi_t$ obeying the so-called continuity equation \cite{MeiE7665,pmlr-v99-mei19a}.
%and obeys the %Fokker-Plank equation gives the following 
%continuity equation:
%$
%\partial_t \pi_t(w) = \nabla_w \cdot (\pi_t(w) v_t(w) ) + \frac{\eta/\beta}{2} \Delta_w \pi_t(w)
%$
%where $v_t(w) = - \lambda w - a  \frac{1}{n} \sum_{i=1}^n x_i \sigma'(w^\top x_i)  \ell'(y_i, f_{\rho_t}(x_i))$.
%The mean field analysis analyzes the convergence of this dynamics. 
The usual finite width network is regarded as a finite sum approximation of the integral representation (\Eqref{eq:IntegralRep}). 
As a consequence, the convergence analysis needs to take limit of infinite width to approximate the absolutely continuous distribution $\rho_t$.
%because $\rho_t$ is absolutely continuous.  %with respect to the Lebesgue measure with density $\pi_t$.  % has the full support.
Hence, a finite width dynamics is outside the scope of mean field analysis.
This is due to the fact that an independent noise is injected to each particle regardless its location;
%The McKean-Vlasov process assumes 
the diffusion $B_t$ is independently and identically applied to each realized path $\{W(t)\mid t\geq 0\}$
(interaction between particles is induced only through gradient). 
However, in a real neural network training, the noise induced by stochastic gradient has high correlation between each node. Thus, we need a different approach.
%, which is different from the mean field analysis.
%The noise is not injected independently to each parameter. 
%As a result, even if a particle $W^{(1)}(t)$ is close to $W^{(2)}(t)$, a completely independent noise is added. 
%This is quite different from the practical deep learning in which the stochastic gradients are close enough if two particles are close.

\paragraph{Lift of McKean-Vlasov process}
Our core idea is to ``lift'' the stochastic process $W(t)$ as a process of a function with the initial value $W(0)$.
For each $W(0) = w_0$, the particle's location at time $t$ is determined by $W(t) = W(t,w_0)$. This means that the process generates a function $w_0 \mapsto W(t,w_0)$ with respect to the initial solution $w_0$.
By considering the stochastic process of this function itself directly, 
the dynamics is transformed to an {\it infinite dimensional stochastic differential equation}, which has been studied especially in the stochastic partial differential equation \cite{da_prato_zabczyk_2014}.
In other words, we try to estimate a map from the initial parameters to the solution at time $t$ instead of analyzing each particle's behavior.
%estimating a density over parameters.

From this %point of view, 
perspective, we can directly regularize the smoothness of the trajectory, 
especially, we can incorporate a smoothed noise of the dynamics 
%We can describe the correlation of noise between nodes as 
by utilizing a spatially correlated Gaussian process in the space of functions on parameters. 
Let $W_t(w) = W(t,w)$ and we regard $W_t$ as a member of $L_2(\rho_0)$ space. 
Then, $f_{\rho_t}$ can be rewritten by 
\begin{equation}\label{eq:fWtdefinition}
%\textstyle
f_{W_t}(x) := \int_{\Real^d} a \sigma(W_t(w)^\top x) \dd \rho_0(w) = \int_{\Real^d} a \sigma(w^\top x) \dd W_{t}\sharp \rho_0(w),
\end{equation}
where $W_{t}\sharp \rho_0$ is the pushforward of the measure $\rho_0$ by the map $W_t$, i.e., $f \sharp \mu(B) := \mu \circ f^{-1} (B) = \mu(f^{-1}(B))$ for a Borel measurable map $f: \Real^d \to \Real^d$, a Borel measure $\mu$, and a Borel set $B \subset \Real^d$.
By using this notation,  the stochastic process we consider can be written as  
\begin{align}
\dd W_t = - (A W_t + \nabla_W \calLhat(f_{W_t}) ) \dd t+ \sqrt{2\beta^{-1}} \dd \xi_t,
\label{eq:InfGLDin2NN}
\end{align} 
where $A: L_2(\rho_0) \to L_2(\rho_0)$ is an unbounded linear operator corresponding to a regularization %smoothing operator 
(which will be explained later in more details),
$\nabla_W \calLhat(f_{W})$ is the Frechet derivative of $\calLhat(f_{W} )$ with respect to $W$ in the space of $L_2(\rho_0)$, in our setting, which is given by $\nabla_W \calLhat(f_{W})(w) = a \frac{1}{n} \sum_{i=1}^n x_i \sigma'(W(w)^\top x_i) \ell'(y_i, f_W(x_i))$.
$(\xi_t)_t$ is a {\it cylindric Brownian motion} in $L_2(\rho_0)$ \cite{da_prato_zabczyk_2014}, which is an infinite dimensional Brownian motion and will be defined rigorously later on.
In practical deep learning, the regularization term $A W_t$ is induced by several mechanism such as weight decay \citep{krogh1992simple}, dropout \citep{srivastava2014dropout,NIPS2013_4882}, batch-normalization \citep{pmlr-v37-ioffe15}. 
%a derivative of $\frac{1}{2}\langle W_t, A W_t\rangle_{\cH}$ 
As a result, the regularization term $AW_t$ introduces spatial correlation between particles unlike the McKean-Vlasov process.

Then, training two layer neural networks is formulated as optimizing the map $W: w \in \Real^d \mapsto W(w) \in \Real^d$ with the initial condition $W_0 = \idmap$ (identity map). 
%The dynamics \eqref{eq:InfGLDin2NN} is called infinite dimensional gradient Langevin dynamics (GLD).
This dynamics is well analyzed and guaranteed to converge to at least a stationary distribution (a.k.a., invariant measure) under mild assumptions  \cite{DaPrato_Zabczyk92,Maslowski:1989,Sowers:1992,Jacquot+Gilles:1995,Shardlow:1999,Hairer:2002} which is useful to show convergence  %of $W_t$ 
to a (near) global optimal.

\begin{Remark}
\label{rem:FiniteWidth}
We would like to emphasize that our formulation admits a finite width neural network training by setting the initial distribution $\rho_0$ as a discrete distribution $\rho_0 = \frac{1}{M} \sum_{m=1}^M \delta_{w_m}$ for a Dirac measure $\delta_{w_m}$ which has probability 1 on a point $w_m$. 
In this situation, optimizing the map $W_t$ corresponds to optimizing the finite width model \eqref{eq:finitewidthmodel} because $\rho_t = W_t \sharp \rho_0 = \frac{1}{M}\sum_{m=1}^M \delta_{W_t(w_m)}$ which is still a discrete distribution throughout entire $t \in \Real_+$.
This is remarkably different from both mean field analysis and NTK analysis that essentially take infinite width limits:
%Actually,
mean field analysis in \cite{MeiE7665,pmlr-v99-mei19a} requires $M = \Omega(e^T)$ for a time horizon $T$ and NTK requires $M = \Omega(\mathrm{poly}(n))$ \cite{zou2019improved}.
\end{Remark}

%\subsection
\paragraph{General formulation of our optimization problem}

Here, we describe mathematical details of optimizing the transportation map in a more general setting
 and give a practical algorithm of the corresponding GLD.
We assume that the map $W_t(\cdot)$ is included in a separable Hilbert space $\cH$ with norm $\|\cdot\|_{\cH}$ and an inner product $\langle \cdot,\cdot\rangle_{\cH}$ (in the previous section, $\cH = L_2(\rho_0)$).
The Hilbert space $\cH$ consists of functions whose domain is a set $\calW$ and whose range is $\calWt$ (in the previous example, $\calW = \Real^d$ amd $\calWt = \Real^d$).
%In the above example, $\calW = \Real^d$ and $\cH = L_2(\rho_0)$.
Since a function $w \in \cH$ has no smoothness condition in typical settings, we consider a more ``regulated'' subspace of $\cH$.
Such a subspace is denoted by $\cHK$ and given by
$
%\textstyle 
\cH_K := \left\{\sum_{k=0}^\infty \alpha_k e_k \mid \sum_{k=0}^\infty \alpha_k^2/\mu_k < \infty \right\},
$
where $(e_k)_{k=0}^\infty$ is an orthonormal basis of $\cH$ and $(\mu_k)_{k=0}^\infty$ is a non-increasing non-negative sequence.
We equip an inner product $\langle \cdot, \cdot \rangle_{\cHK}$ to the space $\cH_K$ defined by 
$\langle f,g \rangle_{\cH_K} = \sum_{k=0}^\infty \alpha_k \beta_k/\mu_k$ for $f = \sum_{k=0}^\infty \alpha_k e_k \in \cH_K$ and $g = \sum_{k=0}^\infty \beta_k e_k \in \cH_K$. Correspondingly, the norm $\|\cdot\|_{\cH_K}$ is defined from the inner product.
%Under some boundedness assumption, 
When $\calH = L_2(\rho_0)$,  %defined on $\Real^d$ and some boundedness condition is satisfied,  
$\cH_K$ becomes a {\it reproducing kernel Hilbert space} (RKHS) corresponding to a kernel function $K(x,y) = \sum_{k=0}^\infty \mu_k e_k(x) e_k(y)$ where $x,y \in \Real^d$ under an appropriate convergence condition.
That is, we have the reproducing property $\langle K(x,\cdot), W \rangle_{\cH_K} = W(x)$ for each $W \in \cH_K$.
Based on the norm $\|\cdot\|_{\cHK}$, we define an unbounded linear operator $A:\cH \to \cH$ as 
$
A f = \lambda \sum_{k=0}^\infty \frac{\alpha_k}{\mu_k}e_k,
$
for $f =  \sum_{k=0}^\infty \alpha_k e_k \in \cH$.
We note that $Af =\frac{\lambda}{2} \nabla_f \|f\|_{\cH_K}^2$ which is a Frechet derivative of $\lambda \|\cdot\|_{\cH_K}^2$ in $\cH$ (which is the derivative of the RKHS norm, if $\cHK$ is an RKHS).
We assume that for each $W \in \cH$, there exits a function $f_W:\Real^d \to \Real$ as in \Eqref{eq:fWtdefinition}, 
and we basically aim to minimize the regularized empirical risk  
\begin{align*}
\calLhat(f_{W}) + \tfrac{\lambda}{2} \|W\|_{\cH_K}^2.
\end{align*}
By abuse of notation, we denote by $\calLhat(W)$ indicating $\calLhat(f_W)$. 
To execute this non-convex optimization, 
we use the GLD in the infinite dimensional Hilbert space $\cH$ as introduced in \Eqref{eq:InfGLDin2NN}.
Here, $(\xi_t)_{t \geq 0}$ in \Eqref{eq:InfGLDin2NN} is the { cylindrical Brownian motion} defined as 
$
\xi_t = \sum_{k \geq 0} B^{(k)}_t e_k
$
where $(B^{(k)}_t)_{t \geq 0}$ is a real valued standard Brownian motion and they are independently identical for $k=0,1,2,\dots$\footnote{More natural modeling would be that the regularization $A$ and the covariance of $\xi_t$ depend on the current solution $W_t$, but we consider this simplest model for technical tractability.}.
%Then, we can formally define the GLD in the infinite dimensional Hilbert space $\cH$ as \Eqref{eq:InfGLDin2NN}.
%$$
%\dd W_t = - (A W_t + \nabla_W \calLhat(W_t)) \dd t + \sqrt{\beta^{-1}} \dd \xi_t.
%$$
%This is the more rigorous and generalized formulation of \Eqref{eq:InfGLDin2NN}.
Since this is defined on a continuous time domain, we introduce a discrete time {\it implicit Euler scheme} for practical implementation:
\begin{align}\label{eq:IGLDDiscreteDynamics}
& W_{k+1} \!\!=\! W_{k} \!\!-\! \eta ( A W_{k+1} \!+\!\nabla_W \calLhat(W_k)) \! + \!\! \sqrt{\tfrac{2 \eta}{\beta}} \epsilon_k %\\
\Leftrightarrow%& 
W_{k+1} \!=\! S_\eta \!  \left( \! W_{k} \!\!-\! \eta  \nabla_W \calLhat(W_k) \!+\! \!\sqrt{\tfrac{2 \eta}{\beta}} \epsilon_k\!\right)\!, 
\end{align}
where $\eta > 0$ is the step size and $S_\eta = (\idmap + \eta A)^{-1}$.
We can see that the ``regularization effect'' $A W$ induces the spacial smoothness of the noise of the gradient.
It is known \cite{brehier:cel-01633504} that under some assumption (Assumption \ref{ass:IGLDConvCond} below is sufficient), 
the process \eqref{eq:InfGLDin2NN} has a unique invariant measure $\pi_\infty$ given by 
\begin{align*}
%\textstyle 
\frac{\dd \pi_\infty}{\dd \nu_\beta}(W) \propto \exp(- \beta \calLhat(W)), 
\end{align*}
where $\nu_\beta$ is the Gaussian measure in $\cH$ with mean 0 and covariance $(\beta A)^{-1}$ (see \citet{da_prato_zabczyk_2014} for the rigorous definition of the Gaussian measure on a Hilbert space and related topics about existence of invariant measure).
In a special situation where $\beta = n,~\lambda=1/n$ and $\beta \calLhat(W)$ is a log-likelihood function of some model,
this invariant measure is nothing but the {\it Bayes posterior distribution} for a Gaussian process prior corresponding to the RKHS $\cHK$.
Remarkably, this formulation can be applied to several problems other than training two layer neural networks: 
%There are many examples that can be solved in this framework:

{
\vspace{-0.3cm}
\setlength{\leftmargini}{0.5cm}
\begin{itemize}%[itemsep=0.05cm]
\setlength{\itemsep}{0pt}
\item Ordinary nonparametric regression model:
$\calW = \Real^d$, $\calWt = \Real$ and $f_W(x) = W(x)$. 
\item Two layer neural networks (continuous topology): 
$\calW = \calWt  = \Real^{d}$ and $f_W = \int_{\Real^d} a(w) \sigma(W(w)^\top x) \dd \rho_0(w)$.
\item Two layer neural networks (discrete topology): $\calW = \{1,2,3,\dots\}$, $\calWt = \Real^d$ and $
f_W = \sum_{m=1}^\infty a_m \sigma(W(m)^\top x).
$
\item Two layer neural networks (discrete topology): $\calW = \{1,2,3,\dots\}$, $\calWt = \Real^d$ and $
f_W = \sum_{m=1}^\infty a_m \sigma(W(m)^\top x).
$
\item Deep neural networks (continuous topology): 
%$\calW = \calWt  = \Real^{d}$ 
$\calW = \Real^d \times \{1,\dots,L\}$, $\calWt = \Real^d$
and 
\begin{align*}
\textstyle
f_W(x) = u^\top \left(\int_{\Real^d}  a_{w,L}\sigma (W(w,L)^\top \cdot) \dd \rho_0(w)\right) \circ \dots \circ 
\left(\int_{\Real^d}  a_{w,1}\sigma(W(w,1)^\top x) \dd \rho_0(w)\right),
\end{align*}
where $u \in \Real^d$ and $a_{w,\ell} \in \Real^{d}$ for $w\in \Real^d$ and $\ell \in \{1,\dots,L\}$.
\item ResNet: 
$\calW = \Real^d \times \{1,\dots,T\}$, $\calWt = \Real^d$ and 
\begin{align*}
\textstyle
f_W(x) \!=\! u^\top \!\left(\idmap \!+\! \int_{\Real^d} \! a_{w,T}\sigma (W(w,T)^\top \cdot) \dd \rho_0(w)\right) \circ \dots \circ 
\left(\idmap \!+\! \int_{\Real^d} \! a_{w,1}\sigma(W(w,1)^\top x) \dd \rho_0(w)\right),
\end{align*}
where $u \in \Real^d$ and $a_{w,t} \in \Real^{d}$ for $w\in \Real^d$ and $t \in \{1,\dots,T\}$.
\item Wasserstein optimal transportation map: $\calW = \calWt = \Real^d$ and $f_W(x) = W(x)$. 
For random variables $X$ and $Y$ obeying distributions $P$ and $Q$ respectively: 
$
\calW^2(P,Q) = \min_{W:  Q = f_W\sharp P} \EE_{X\sim P}[\|X - f_W(X)\|^2].
$
\end{itemize}
}

\section{Optimization error bound of transportation map learning}
To show convergence of the dynamics \eqref{eq:IGLDDiscreteDynamics}, we utilize the recent result given by \cite{muzellec2020dimensionfree}.  
%Recently, the convergence of the dynamics  is investigated by \cite{muzellec2020dimensionfree}.  
%We use their result to ensure convergence. % under the following assumptions.
%Here, we discuss convergence of the discrete time dynamics \eqref{eq:IGLDDiscreteDynamics} and introduce some assumptions to ensure convergence. 
Let $\norm{W}_\varepsilon := \big(\sum_{k \geq 0} (\mu_k)^{2\varepsilon} \langle W, e_k \rangle_{\cH}^2\big)^{1/2}$ and $P_N W := \sum_{k=0}^{N-1} \langle W,e_k \rangle_{\cH} e_k$ for $W \in \cH$ where $(e_k)_k$ is the orthonormal system of $\cH$. Accordingly, let $\cH_N$ be the image of $P_N$: $\cH_N = P_N \cH$.

 \begin{Assumption}\label{ass:IGLDConvCond}~
{\vspace{-0.15cm}
\setlength{\leftmargini}{0.8cm}
\begin{assumenum}
\item {\rm (Eigenvalue condition)} \label{assum:eigenvalue_cvg}
     There exists a constant $c_\mu$ such that $\mu_k \leq c_\mu (k+1)^{-2}$.
\item {\rm (Boundedness and Smoothness)} \label{assum:eigenvalue_cvg}\label{assum:smoothness}
There exist $B,M > 0$ such that the gradient of the empirical risk  is bounded by $B$ and is $M$-Lipschitz continuous
with $\alpha \in (1/4,1)$ almost surely:   
\begin{align*}   
\norm{\nabla  \calLhat(W)}_{\cH} \leq B~(\forall W \in \cH),
\quad \norm{\nabla \calLhat(W) - \nabla \calLhat(W')}_{\cH} \leq L \norm{W - W'}_{\alpha} ~(\forall W, W' \in \cH).
\end{align*}
\item {\rm (Third order smoothness \cite[Assumption 2.7]{Brehier16})} \label{assum:C2_boundedness}
Let $\calLhat_N: \cH_N \to \Real$ be $\calLhat_N = \calLhat(P_N W)$. $\calLhat$ is three times differentiable, and there exists $\alpha' \in [0, 1), C_{\alpha'} \in (0, \infty)$ such that for all $N \in \mathbb{N}$ and $\forall W, h, k \in \cH_N,$
%\footnote{\textcolor{red}{We indicate a modification from the main text by red.}}
%
%\begin{align*} &
$
    \norm{\nabla^3 \calLhat_N(W) \cdot (h, k)}_{\alpha'} \leq C_{\alpha'} \norm{h}_{\cH}\norm{k}_{\cH},~~\norm{\nabla^3 \calLhat_N(W) \cdot (h, k)}_{\cH} \leq C_{\alpha'} \norm{h}_{-\alpha'}\norm{k}_{\cH}~~\text{(a.s.)},
$
%\end{align*}
where $\nabla^3 \calLhat_N(W)$ is the third-order derivative, we identify it with third-order linear form, and we also write $\nabla^3 \calLhat_N(W) \cdot (h, k)$ for the Riesz representor of  $l \in \cH \mapsto \nabla^3 \calLhat_N(W) \cdot (h, k, l)$.
%
%and $\calLhat$ is $M$-smooth:
%  \begin{equation}\label{eq:smoothness}
%     \forall x, y\in \cH,\quad \norm{\nabla \calLhat(x) - \nabla \calLhat(y)}_{\alpha} \leq L \norm{x - y} .
%  \end{equation}
\end{assumenum}}
\end{Assumption}

%\begin{Assumption}[Boundedness and Smoothness]\label{assum:smoothness}
%The gradient is bounded: 
%$$   
%\norm{\nabla \calLhat(\cdot)} \leq B.
%$$
%$\calLhat$ is $M$-smooth:
%  \begin{equation}\label{eq:smoothness}
%     \forall x, y\in \cH,\quad \norm{\nabla \calLhat(x) - \nabla \calLhat(y)}_{\alpha} \leq L \norm{x - y} .
%  \end{equation}
%\end{Assumption}

%\begin{Assumption}[{\cite[Assumption 2.7]{Brehier16}}]\label{assum:C2_boundedness}
%Let $\calL_N: \cH_N \to \Real$ be $\calL_N = \calL(P_N x)$. $L$ is three times differentiable, and there exists $\alpha' \in [0, 1), C_{\alpha'} \in (0, \infty)$ such that for all $N \in \mathbb{N}$ and $\forall x, h, k \in \cH_N,$
%%\footnote{\textcolor{red}{We indicate a modification from the main text by red.}}
%%
%\begin{align*}
%    & \norm{D^3 \calL_N(x) \cdot (h, k)}_{\alpha'} \leq C_{\alpha'} \norm{h}_0\norm{k}_0,~~~~\norm{D^3 \calL_N(x) \cdot (h, k)}_{0} \leq C_{\alpha'} \norm{h}_{-\alpha'}\norm{k}_0.
%\end{align*}
%%where $\norm{x}_\varepsilon \eqdef \left(\sum_{k \geq 0} (\mu_k)^{2\varepsilon} |\scal{x}{f_k}|^2\right)^{1/2}$. 
%\end{Assumption}

%\begin{Assumption}\label{assum:dissipative}
%    It either holds that
%    \begin{assumenum}
%        \item  $\lambda  > M\mu_0$ (Strict Dissipativity), or  \label{assum:strict_diss}
%        \item  $\norm{\nabla L(\cdot)} \leq B, \quad B >0$ (Bounded gradients). \label{assum:bounded_grad}
%    \end{assumenum}
%\end{Assumption}

The first condition controls the strength of the regularization term.
%Thanks to this condition, the solution of the gradient Langevin dynamics can remain a bounded region with high probability.
%The decreasing rate $(k+1)^{-2}$  can be 
The second condition ensures the smoothness of the loss function
that yields the \emph{disspativity} condition of the objective combined with the regularization term.
That is, the solution of the gradient Langevin dynamics can remain a bounded region with high probability.
The Lipschitz continuity of the gradient is a bit strong condition because the right hand side appears a weaker norm $\|\cdot\|_\alpha$ than the canonical norm $\|\cdot\|_{\cH}$.
However, this gives the geometric ergodicity (exponential convergence to the stationary distribution) of the discrete time dynamics.
The third condition is more technical assumption. 
This condition is used for bounding the continuous time dynamics and discrete time dynamics.
Intuitively, a smoother loss function makes the two dynamics closer.
In particular, $\eta^{1/2-a}$ term appearing in the following bound can be shown by this condition.

Then, we can show the following weak convergence rate.
Let $\pi_k$ be the probability measure on $\cH$ corresponding to the distribution of $W_k$.
\begin{Proposition}\label{prop:WeakConvergence}
Assume Assumption \ref{ass:IGLDConvCond} holds and $\beta > \eta$. % \ref{assum:eigenvalue_cvg}, \ref{assum:smoothness} and \ref{assum:C2_boundedness}.
%Suppose that $0 \leq \calLhat(W_0) \leq \Rbar$~(a.s.) for $\Rbar > 0$.
Suppose that $\exists \Rbar > 0$, $0 \leq \ell(Y,f_{W}(X)) \leq \Rbar$ for any $W \in \cH$~(a.s.).
Let $\rho = \frac{1}{1 + \lambda\eta/\mu_0}$ and $b = \frac{\mu_0}{\lambda}B + \frac{c_\mu}{\beta \lambda}$. %k(1)$.
%For all $\phi: \cH \to \mathbb{R}$ with $|\phi(\cdot)| \leq V(\cdot)$ and $\|\phi(x) - \phi(y)\| \leq M' \|x - y\|~(x,y \in \cH)$, 
Then, %there exist $\Lambda^*_\eta > 0$ and $C_{W_0} > 0$ such that
%\begin{align*}
for 
$
\textstyle \Lambda^*_\eta = \frac{\min\left(\frac{\lambda}{2 \mu_0}, \frac{1}{2} \right)}{4 \log(\kappa (\bar{V} + 1)/(1-\delta)) } \delta$ and 
$\textstyle C_{W_0} = \kappa [\bar{V} + 1] + \frac{\sqrt{2} (\Rbar + b)}{\sqrt{\delta}}
$
%\gamma = 1 - \eta \min\{ \exp\{- \beta C_0 ( 1 + 2 {c'}^2(\mu_0/\lambda)^2 b^2 ) - \beta \frac{B^2}{2} \},  1 - \alpha_N\}$ with constants $C_0,c'$ 
%and $C_{\lambda,\mu_0}$ is a constant depending on $\lambda^{-1},\mu_0$ polynomially.
%\end{align*}
where $0 < \delta< 1$ satisfying $\delta = \Omega(\exp(-\Theta(\mathrm{poly}(\lambda^{-1})\beta)))$, $\bar{b} = \max\{b,1\}$, $\kappa = \bar{b} + 1$ and 
%$\bar{V} = \tfrac{4 \bar{b}}{\sqrt{(1+\rho^{1/\eta})/2} - \rho^{1/\eta}}$, 
$\bar{V} = 4 \bar{b}/{\scriptstyle (\sqrt{(1+\rho^{1/\eta})/2} - \rho^{1/\eta})}$
(where $\bar{V} = 4\bar{b}/(\scriptstyle \sqrt{(1+\exp(-\frac{\lambda}{\mu_1}))/2} - \exp(-\frac{\lambda}{\mu_1}))$
for $\eta = 0$), 
%\footnote{More detailed evaluation of $\delta$ can be found in the proof.}.
and for any $0 < a < 1/4$, the following convergence bound holds for almost sure observation $D_n$:  
for either $L = \calL$ or $L = \calLhat$,
%\ts{It also holds for $\calL$.}
\begin{align} %\label{eq:informal_final_rate}
%&\EE_{W_k \sim \pi_k}[\calLhat(W_k)] - \EE_{W\sim \pi_\infty}[\calLhat(W)] %\\ & 
%&\EE_{W_k \sim \pi_k,W\sim \pi_\infty}[\calL^*(W_k) -\calL^*(W)] 
&| \EE_{W_k \sim \pi_k}[L(W_k) ] 
-
\EE_{W\sim \pi_\infty}[L(W)] |
%\EE_{W_k,W^{\pi_\infty}}[\calL(W_k) -\calL(W^{\pi_\infty})]
%\\ & 
\leq
%\phi(\tilde{x}) - L(x^*)   \\ & 
%C_1 %e^{\tfrac{c\beta}{\eta}} \gamma^n ~~+ 
C_1 \left[
C_{W_0}  \exp(- \Lambda_\eta^* \eta k )   + 
\frac{\sqrt{\beta}}{\Lambda^*_0}\eta^{1/2-a} \right] =: \Xi_k,
\end{align}
where $C_1$ is a constant depending only on $c_\mu,B,L,C_{\alpha'},a,\Rbar$ (independent of $\eta,k,\beta,\lambda$).
\end{Proposition}
%This is a direct consequence form 
We utilized the theories of \cite{muzellec2020dimensionfree} as the core technique to show this proposition. % is coming from . 
Its complete proof is given in Appendix \ref{sec:ProofOfWeakConv}.
We can see that as $k$ goes to infinity the first term of the right hand side converges exponentially,
and as the step size $\eta$ goes to 0, the second term converges arbitrary close to the rate of $\sqrt{\eta}$.
It is known that the convergence rate with respect to $\eta$ is optimal \cite{brehier2020influence}.
Therefore, if we choose sufficiently small $\eta$ and sufficiently large $k$, we can sample $W_k$ that obeys nearly the invariant measure $\pi_\infty$. As we will see later, sample from $\pi_\infty$ has a nice property in terms of generalization.  %error.
As we have remarked in  Remark \ref{rem:FiniteWidth}, the convergence is guaranteed even for the finite width neural network setting, i.e., $\rho_0$ is a discrete distribution in the model \eqref{eq:fWtdefinition}. This is much advantageous against existing framework such as mean field analysis and NTK.

The above proposition gives a bound on the expectation of the loss of the solution $W_k$ 
instead of a high probability bound.
However, due to the geometric ergodicity of the dynamics, by running the algorithm for sufficiently large steps, 
we can show that the probability that there {\it does not} appear $W_k$ in the trajectory that has a loss such that $
L(W_k)- \EE_{W\sim \pi_\infty}[L(W)] \leq  O(\Xi_k)$ approaches 0 with exponential rate.
Since this direction requires much more involved mathematics, we consider a simpler one as described above.

\section{Generalization error analysis}

%\subsection{Generalization gap bound}
\paragraph{Generalization gap bound}
Here, we analyze the generalization error of the solution of $W_k$ obtained by the dynamics \eqref{eq:IGLDDiscreteDynamics}. %First, we analyze 
%The following theorem gives the generalization gap ($\calL(W_k) - \calLhat(W_k)$).
%\begin{Assumption}\label{ass:LossBoundedness}
%\end{Assumption}

\begin{Theorem}\label{thm:PAC-BayesGenBound}
%Let 
%$
%\Xi_{k} := C_{W_0}  \exp(- \Lambda_\eta^*(\eta k - 1))   + %C_2 
%\frac{\sqrt{\beta}}{\Lambda^*_0}\eta^{1/2-\kappa}.
%$
Assume Assumption \ref{ass:IGLDConvCond} holds with $\beta > \eta$, and assume that the loss function is bounded, i.e., there exits $\Rbar > 0$ such that 
$\forall W \in \cH,~0 \leq \ell(Y,f_W(X)) \leq \Rbar~(\mathrm{a.s.})$.
%\footnote{Note that under this condition, we also have $0\leq \calLhat(W_0) \leq \Rbar$ (a.s.).}.
Then, for any $1 > \delta > 0$, with probability $1 -\delta$, the generalization error is bounded by
\begin{align*}
%\textstyle
\EE_{W_k}[\calL(W_k)]  \leq \EE_{W_k}[\calLhat(W_k)] + 
\frac{\Rbar^2}{\sqrt{n}}
\left[2\left(1 + \frac{2\beta}{\sqrt{n}}\right) + \log\left(\frac{1 + e^{\Rbar^2/2}}{\delta}\right)\right] + 2 \Xi_{k}.
%O_p(1/\sqrt{n}).
\end{align*}
%with probability $1 -\delta$.
\end{Theorem}
The proof is given in Appendix \ref{sec:ProofOfPAC-BayesBound}.
To prove this, we used a PAC-Bayes stability bound \cite{NIPS:Rivasplata:2020}.
%\cite{NIPSWS:Rivasplata:2019}.
From this theorem, we have that the generalization error is bounded by $O(1/\sqrt{n})$ and the optimization error $\Xi_{k}$.
The $O(1/\sqrt{n})$ term is the generalization gap for the stationary distribution, and as $k$ goes to infinity, the total generalization gap converges to this one.
\cite{pmlr-v75-mou18a} also showed a PAC-Bayesian stability bound for a finite dimensional Langevin dynamics
(roughly speaking, their bound is $O(\sqrt{\beta B^2/(n\lambda)})$),
but their proof technique is quite different from ours.  
Our proof analyzes the generalization error under the stationary distribution of the dynamics
and bounds the gap between the stationary distribution and the current solution,
while \cite{pmlr-v75-mou18a} evaluated the bound by ``accumulating'' the error through the updates without analyzing the stationary distribution.
%The proof technique is 
%Note that the generalization error does not increase as the number $k$ of updates increases unlike existing analysis %\cite{pmlr-v75-mou18a} ($O(\sqrt{k/n})$ or $O(\sqrt{\beta \sum_i \eta_i/n})$ for a step size $\eta_i$ in the $i$-th step) %which is for a finite dimensional situation.
%This is because we bounded the generalization error based on the stationary distribution of the dynamics,
%whereas \cite{pmlr-v75-mou18a} did not make use of the stationary distribution and computed the bound by ``accumulating'' the error through the updates.

%\subsection{Excess risk bound: fast learning rate}
\paragraph{Excess risk bound: fast learning rate}
%\section{Fast learning rate: Student-teacher model}

Next, we bound the excess risk. 
Unlike the $O(1/\sqrt{n})$ convergence rate of the generalization gap bound, we can derive a fast learning rate which is faster than $O(1/\sqrt{n})$ in a setting of realizable case, i.e., a student-teacher model, %by analyzing 
for the excess risk instead of the generalization gap.
As a concrete example, we keep the following two layer neural network model in our mind. 
For a map $W:\Real^{d_1} \to \Real^{d_2}$, let a ``clipped map'' $\bar{W}$ be 
$
\bar{W}(w) := R \times \tanh(W(w)/R),
$
where $R \geq 1$ is a constant and $\tanh$ is applied elementwise.
% and $\mathrm{sig}(\cdot)$ is the sigmoid function $1/(1+e^{-x})$ applied elementwise.
Then, the following two layer neural network model falls into our analysis:
\begin{align}\label{eq:2layerClippedModel}
\textstyle
f_W(x) := \int_{\Real \times \Real^{d}} \bar{W}_2(a) \sigma(\bar{W}_1(w)^\top x)  \dd \rho_0(a,w)
\end{align}
%\ts{!!!$\bar{W}_1$!!!!!!!!!!!!!!!}
for a measurable map $W = (W_1,W_2): \Real^{d}\times \Real \to \Real^{d} \times \Real$ and an activation function $\sigma$ that is 1-Lipschitz continuous and included in a {\it H\"older class} $\calC^3(\Real)$. Here, we used the clipping operation only for a technical reason because 
%the boundedness condition assumed in 
the current convergence analysis of the infinite dimensional Langevin dynamics requires a boundedness condition. 
%&is satisfied.
This could be removed if we could show its convergence under more relaxed conditions.
The fast learning rate analysis is not restricted to the two layer model, but it can be applied as long as the following statement is satisfied (e.g., ResNet).

\begin{Lemma}\label{lem:LipshitzInfty}
%$L_\infty$-distance <= Wasserstein distance 
For the model \eqref{eq:2layerClippedModel}, if $\|x\| \leq D$ for any $x \in \supp(P_X)$, then it holds that 
$
\|f_{W} - f_{W'}\|_{\infty} %\lesssim \calW_2(W \sharp \rho_0, W'\sharp \rho_0) \lesssim 
\leq (1 + RD) \|W - W'\|_{L_2(\rho_0)}
$
where $\|W - W'\|_{L_2(\rho_0)}^2 := \int \|W((a,w)) - W'((a,w))\|^2 \dd \rho_0(a,w)$.
\end{Lemma}
The proof is given in Appendix \ref{sec:ProofOfLipshitzInfty}.
This lemma indicates that to estimate a function $f_{W^*}$, its estimation error can be bounded by the estimation error of the parameter $W$.
%Therefore, to estimate a function $f_{W^*}$, then 
To ensure the smooth gradient assumption (Assumption \ref{ass:IGLDConvCond}-\ref{assum:smoothness}) and precisely  characterize the estimation accuracy by the model complexity, we consider an RKHS with ``smoothness'' parameter $\gamma$ as the model of $W$.
Let $T_K: \cH \to \cH$ be a linear bounded operator such that $\langle T_K h, h'\rangle_{\cH} = \sum_{k=0}^\infty \mu_k \alpha_k \alpha'_k $ 
for $h = \sum_{k} \alpha_k e_k$ and $h' = \sum_{k} \alpha'_k e_k$.
Let the range of power of $T_K$ be $\cH_{K^\gamma} = \{f = T_K^{\gamma/2} h\mid  h \in \cH\}$
for $\gamma > 0$ which is equipped with the inner product $\langle h,h' \rangle_{\cH_{K^\gamma}} = \sum_{k=0}^\infty \mu_k^{-\gamma} \alpha_k \alpha'_k$.
%We consider a situation where the model of $W$ is $\cH_{K^\gamma}$.
We can see that $\gamma=1$ corresponds to $\cH_K$ and $\gamma$ controls the ``complexity'' of $\cH_{K^\gamma}$,
that is, if $\gamma < 1$, then $\cH_K \hookrightarrow  \cH_{K^\gamma}$, and otherwise, $\cH_{K^\gamma}\hookrightarrow  \cH_K $.
We consider a problem of %optimization to 
%of $\calLhat(W)$ and $\calL(W)$ with respect to $W \in \cH$ 
optimizing $\calLhat(f_W)$ or $\calL(f_W)$ with respect to $W$ in the model $\cH_{K^\gamma}$.
%corresponds to optimizing $\calLhat(f)$ and $\calL(f)$ with respect to $f$ in the model $\cH_{K^\gamma}$ respectively.
To so so, by noticing that any $g \in \cH_{K^\gamma}$ can be written as $g = T_K^{\gamma/2}W$ for $W \in \cH$, we write the empirical and population risk with respect to $W \in \cH$ as 
%%Let the loss function be 
%$
%\ell(W,Z) := \ell(Y, f_{T_K^{\gamma/2}W}(X)),
%$
%(note that .
% Accordingly the empirical and population risk are defined as 
$
\calLhat(W) = \calLhat(f_{T_K^{\gamma/2}W} ),~%\frac{1}{n} \sum_{i=1}^n \ell(W,z_i),~~\calL(W) = \EE_Z[\ell(W,Z)].
\calL(W) = \calL(f_{T_K^{\gamma/2}W}).
$
%By abuse of notation, we also use the same notation for risks of a function $f$ as $\calLhat(f) = \frac{1}{n}\sum_{i=1}^n \ell(y_i,f(x_i))$ and $\calL(f) = \EE[\ell(Y,f(X))]$.
%Note that optimizing $\calLhat(W)$ and $\calL(W)$ with respect to $W \in \cH$
%corresponds to optimizing $\calLhat(f)$ and $\calL(f)$ with respect to $f$ in the model $\cH_{K^\gamma}$ respectively.
Let  $\fstar \in \argmin_{f} \calL(f)$ where min is taken over all measurable functions and we assume the existence of the minimizer.

\begin{Assumption}[Bernstein condition and predictor condition \cite{van2015fast,bartlett2006empirical}]\label{ass:BernsteinLightTail}
The Bernstein condition is satisfied: there exist $C_B > 0$ and $\betas \in (0,1]$ such that 
for any $f _W~(W \in \calH)$, 
\begin{align*}
\EE[(\ell(Y,f_W(X)) - \ell(Y,\fstar(X)))^2] \leq C_B (\calL(f_W) - \calL(\fstar))^{\betas}.
\end{align*}
Moreover, we assume that, for any $h:\Real^d \to \Real$ and $x \in \supp(P_X)$, it holds that 
\begin{align*}
\textstyle
\EE_{Y|X=x}\big[\exp\big(- \frac{\beta}{n}(\ell(Y,h(x)) - \ell(Y,\fstar(x)))\big)\big] \leq 1.
\end{align*}
%$$
%\fstar \in 
%$$
\end{Assumption}
The first assumption is called {\it Bernstein condition}.
We can show that this condition is satisfied by the logistic loss and the squared loss with bounded $f_W$ and $\fstar$ (Theorem \ref{thm:ClassificationFastRate}).
%for a classification task, and 
%with the label noise is relatively low \cite{van2015fast,bartlett2006empirical} where $s$ represents strength of noise (small $s$ corresponds to large noise, %and large $s$ corresponds to low noise) (see also {\it Tsybakov's low noise condition} \cite{tsybakov2004optimal}). 
%For a regression case, 
%the squared loss also satisfies it with $s=1$ as far as $f_W$ and $\fstar$ are bounded.
The second assumption is called {\it predictor condition} \cite{van2015fast} and can be satisfied if $\ell$ is a log-likelihood function and the model is correctly specified (that is, the true conditional probability density (or probability mass) $p(y|x)$ is expressed as $p(y|x) \simeq \exp(-\ell(y,\fstar(x)))$).
To extend the theory to misspecified situations, we need the second assumption. For example, 
if we use a squared loss in a regression problem whereas the label noise is {\it not} Gaussian, then it is a misspecified situation but if the noise has a light tail (such as sub-Gaussian), then the assumption can be satisfied \cite{van2015fast}.

Our analysis is valid even if $\fstar$ cannot be represented by $f_W$ for $W \in \cH$.
%However, we assume that $\fstar$ can be approximated by $f_W$ for any precision with respect to $\LPiPx$
%so that we can show convergence of the excess risk.
This model misspecification can be incorporated as bias-variance trade-off in the excess risk bound.
This trade-off can be captured by the following {\it concentration function}. 
Let 
$
\cHKtil = \calH_{K^{\gamma+1}},
$
and the Gaussian process law of $T_K^{\gamma/2}W$ for $W \sim\nu_\beta$ be $\nutil$.
Then, define the concentration function as 
\begin{align*}
& \phi_{\beta,\lambda}(\epsilon) := \inf_{h \in \cHKtil: \calL(h) - \calL(\fstar) \leq \epsilon^2} \beta \lambda \|h\|^2_{\cHKtil}
- \log \nutil(\{h \in \cH : \|h\|_{\cH} \leq \epsilon \}) + \log(2),
\end{align*}
where, if there does not exist $h \in \cHKtil$ satisfying the condition in $\inf$, then we set $\phi_{\beta,\lambda}(\epsilon) = \infty$.
\begin{Theorem}\label{thm:ExcessRiskConvRate}
Assume that %Assumptions \ref{ass:IGLDConvCond} and 
Assumption \ref{ass:BernsteinLightTail} holds, 
$\|x\| \leq D~(\forall x \in \calX)$,
 $\gamma > 1/2$, 
$\beta > \eta$ and $\beta \leq n$. 
Assume that the loss function $\ell(y,\cdot)$ is included in $\calC^3(\Real)$ for any $y \in \supp(\Py)$ and 
there exists $B>0$ such that $|\tfrac{\partial^k}{\partial u^k}\ell(y,u)| \leq B ~(\forall u \in \Real~\mathrm{s.t.}~|u| \leq R,~\forall y \in \supp(\Py),~k=1,2,3)$.
Assume also %$\lambda \beta \leq 1$ and 
%the loss function $f \mapsto \ell(y,f)$ is convex~($\forall y \in \supp(P_Y)$), 
that $0 \leq \ell(Y,f(X)) \leq \Rbar~(a.s.)$ for any $f = f_W~(W \in \cH)$ and $f = \fstar$,
and $\bar{\ell}_x(u) := \EE_{Y|X=x}[\ell(Y,u)]$ satisfies $|\frac{\dd \bar{\ell}_x}{\dd u}(u) - \frac{\dd \bar{\ell}_x}{\dd u}(u')| \leq L |u - u'|~(\forall u,u' \in \Real,\forall x \in \calX)$ for a constant $L > 0$.
%Then, if 
Let 
$\alphatil := 1/\{2(\gamma+1)\}$ and $\theta$ be an arbitrary real number satisfying $0 < \theta < 1 - \alphatil$.
We define 
$
\epsilonstar := \inf\{\epsilon > 0 :  \phi_{\beta,\lambda}(\epsilon) \leq \beta \epsilon^2\} \vee n^{-\frac{1}{2-s}}.
$
Then, the expected excess risk is bounded as 
\begin{align}\label{eq:ExpectedLossConvRate}
%\EE_{D^n}\left[\int \calL(W) - \calL(\fstar) \dd \pi_\infty(W) \right]
\EE_{D^n}\left[ \EE_{W_k}[\calL(W_k)] - \calL(\fstar) \right]
\leq  
 C \Big[ %\epsilon^2 + \frac{1}{\beta}\phi_{\beta,\lambda}(\epsilon) 
\epsilonstar^2 \vee 
\big(\tfrac{ \beta }{n} \epsilonstar^2  + n^{-\frac{1}{1+\alphatil/\theta}} (\lambda\beta)^{\frac{2\alphatil/ \theta}{1+\alphatil/\theta}} \big)^{\frac{1}{2-s}}
\vee \frac{1}{n}
\Big] + \Xi_k,
\end{align}
where $C$ is a constant independent of $n,\beta,\lambda,\eta,k$.

\end{Theorem}
The proof is given in Appendix \ref{sec:ProofOfExcessRisk}.
It is proven by using the technique of nonparametric Bayes contraction rate analysis \cite{AS:Ghosal+Ghosh+Vaart:2000,AS:Vaart&Zanten:2008,JMLR:Vaart&Zanten:2011}. However, we cannot adapt these existing techniques because (i) the loss function is not necessarily the log-likelihood function, (ii) the inverse temperature is generally different from the sample size.
In that sense, our proof is novel to derive an excess risk for (i) a misspecified model, and (ii) a randomized estimator with a general inverse temperature parameter. 

The bound is about expectation of the excess risk instead of high probability bound.
However, a high probability bound is also provided in the proof and the expectation bound is derived from the high probability bound.

If the bias is not zero, i.e., $\inf_{W \in \cH} \calL(W) - \calL(\fstar) = \delta_0 > 0$, then we may choose $\epsilonstar^2 = \Theta(\delta_0)$ 
because %but we cannot make $\epsilonstar \to 0$ as 
$\phi_{\beta,\lambda}(\epsilon)$ is finite for $\epsilon^2 > \delta_0$
and infinite for $\epsilon^2 < \delta_0$.
Thus, a misspecified setting is covered.
% but  $\phi_{\beta,\lambda}(\epsilon) is finite for $\epsilon^2 > \delta_0$.
%. Using this argument, we can apply our bound also to a misspecified model setting.

\paragraph{(i) Example of fast rate: Regression}
Here, we apply our general result to a nonparametric regression problem by the neural network model.
We consider the following nonparametric regression model:  
$
y_i = f_{W^*}(x_i) + \epsilon_i,
$
for $W^* \in \cH$  where $\epsilon_i$ is an i.i.d. noise with mean 0 and $|\epsilon_i| \leq C <\infty$ (a.s.).
%We assume there exists $W^* \in \cH_{K^{(1+\gamma)/4}}$.
To estimate $f_{W^*}$, we employ the squared loss $\ell(y,f) = (y - f)^2$. Then, we can easily confirm that $\fstar$ is achieved by $f_{W^*}$ via a simple calculation: $\argmin_{f}\calL(f) = f_{W^*}$. Moreover, for the squared loss, $s=1$ is satisfied as remarked just after Assumption \ref{ass:BernsteinLightTail}.
%We assume that there exists $\Wstar \in \cH$ such that $\fstar = f_{\Wstar}$.
Moreover, we further assume that 
$
\Wstar \in \calH_{K^{\theta(\gamma+1)}} %\cHKtilt
$
for $\theta < 1-\alphatil$. 
%We employ a squared loss function $\ell(y,f) = (y - f)^2$. Then, $s=1$. 
%By noticing that Lemma \ref{lem:LipshitzInfty} gives 
%$
%\calL(h) - \calL(\fstar) \leq \gamma \|h - \Wstar\|_{\cH}^2
%$
%for some $\gamma > 0$, 
Then, the ``bias'' and ``variance'' terms can be evaluated as 
$
\inf_{h \in \cHKtil: \calL(h) - \calL(\fstar) \leq \epsilon^2} \lambdabeta \|h\|^2_{\cHKtil}
\lesssim \lambda\beta \epsilon^{- \frac{2(1-\theta)}{\theta}}$ and
$- \log \nutil(\{h \in \cH: \|h\|_{\cH} \leq \epsilon \})
\lesssim (\epsilon/(\lambda\beta)^{1/2} )^{-\frac{2\alphatil}{1 - \alphatil}}
$.
%Therefore, we obtain the following excess risk bound.
%The covering number of $\calB_{\cHKtil}$ is bounded by 
%$$\tilde{\phi}(\epsilon') \lesssim (\epsilon'/\lambdabeta)^{-2\alphatil /\theta}.$$
%$$
%\ustar \lesssim
%\max \left\{  \lambdabeta^{\frac{2\alphatil/\theta}{1 + \alphatil/\theta}} n^{-\frac{1}{1 + \alphatil/\theta}},
%\epsilonstar^2 \right\}.
%$$
%Finally, we have that 
Accordingly, we can show the following excess risk bound: 
\begin{align}\label{eq:RegressionFastRate}
%\EE_{D_n}\left[ \int (\calL(h) - \calL(\fstar)) \dd \pi_\infty(h) \right]
\EE_{D^n}\left[ \EE_{W_k}[\calL(W_k)] - \calL(\fstar) \right]
\lesssim 
\max \big\{  (\lambda\beta)^{\frac{2\alphatil/\theta}{1 + \alphatil/\theta}} n^{-\frac{1}{1 + \alphatil/\theta}},
\lambda^{-\alphatil} \beta^{-1}, \lambda^{\theta}, 1/n \big\} + \Xi_k,
\end{align}
(see Appendix \ref{App:RegressionExcessRiskDerivation} for the derivation).
In particular, if $\beta = \lambda^{-1} = n$, then this convergence rate can be rewritten as
$
%\EE_{D_n}\left[ \int (\calL(h) - \calL(\fstar)) \dd \pi_\infty(h) \right] \lesssim  
\max\{n^{-\frac{1}{1 + \alphatil/\theta}}, n^{-\theta}\} = n^{-\theta}~(\because \theta < 1-\alphatil),
$
which can be faster than $1/\sqrt{n}$ and is controlled by the ``difficulty'' of the problem $\alphatil$ and $\theta$.
\begin{Remark}
As an example, if the RKHS $\cHK$ is a Sobolev space $W_2^{a + d/2}(\Real^d)$ with regularity parameter $a + d/2$ (more precisely, each output $W_i(\cdot)$ is a member of a Sobolev space) and $\cH$ is $L_2(\rho_0)$,
then we can set $\alphatil = \frac{d}{2a + d}$.
If the true parameter $W^*$ is included in another Sobolev space $W_2^b(\Real^d)$ for $b \leq a$, then we may choose $\theta = 2b/(2a + d)$ % < \frac{2\alpha}{2\alpha + d}$
and the convergence rate is bounded by
$
n^{-2 b/(2a + d) },
$
 %the kernel function is the Mat\'ern kernel 
which coincides with the posterior contraction rate of Gaussian process estimator derived in \cite{JMLR:Vaart&Zanten:2011}. 
It is known that, if $a = b$, this achieves the {\rm minimax optimal rate} \cite{AS:Yang+Barron:99}.
%Hence, our bound includes analysis of Gaussian process regression as a special case.
\end{Remark}

%More generally, if $s\neq 1$, then $\ustar$ is modified as 
%$$
%\ustar \lesssim
%\max \left\{  \lambdabeta^{\frac{2\alphatil/\theta}{2 - s(1 - \alphatil/\theta)}} n^{-\frac{1}{2 - s(1 - \alphatil/\theta)}},
%\epsilonstar^2 \right\}.
%$$

\paragraph{(ii) Example of fast rate: Classification (exponential convergence)}
%\paragraph{Fast rate  (Exponential convergence): Classification}
%$$
%\eta(x) := P(Y=1|X=x).
%$$
Here, we consider a binary classification problem $y \in \{\pm 1\}$.
We employ the logistic loss function $\ell(y,f) = \log(1 + \exp(-yf))$ for $y \in \{\pm 1\}$ and $f\in \Real$.
Corresponding to the loss function, we define the expected loss conditioned by $X = x$ as 
$
h(u|x) = \EE[\ell(Y, u) |X=x].
$
Note that $h(0|x) = \log(2)$.  % and $\fstar(x) = \argmin_{u \in \Real}h(u|x)$. 
We assume that the strength of noise of  this binary classification problem is low as follows.
\begin{Assumption}[Strong low noise condition]\label{ass:StrongLowNoise}
%$$
%|\eta(x) - 1/2| \geq \delta.
%$$
Let $\hstar(x) := \inf_{u \in \Real} h(u|x)$.
Assume that there exists $\delta > 0$ such that 
$\hstar(x) \leq \log(2) - \delta~(\forall x \in \calX)$.
Moreover, there exists $\Wstar \in \calH$ such that $\fstar = f_{\Wstar}$,
that is, 
$\sup_{x \in \supp(\Px)} |h(f_{\Wstar}(x) |x) - \hstar(x)| =0.$
%for any $\epsilon > 0$, there exists $W \in \cH$ such that %\cH_{K^{1/4}}$ such that 
%$\sup_{x \in \supp(\Px)} |h(f_{W}(x) |x) - \hstar(x)| \leq \epsilon.$
\end{Assumption}
The first assumption is satisfied if the label probability is away from the even probability $1/2$: $|P(Y|X=x) - 1/2| > \Omega(\sqrt{\delta})$. This condition means that the class label has less noisy than completely random labeling. In that sense, we call this assumption the {\it strong low noise condition}, which has been analyzed in \cite{ExpRate:Koltchinskii+Beznosova:2005,audibert2007fast,pmlr-v89-nitanda19a}. A weaker low noise condition was introduced by \cite{tsybakov2004optimal} as Tsybakov's low noise condition.
The second assumption can be relaxed to the existence of $W$ only for some $\epsilon > c_0 \delta$ with sufficiently small $c_0$, but we don't pursuit this direction for simplicity.
\begin{Assumption}\label{ass:ClassSupport}
Assume $\calX (= \supp(\Px)) \subset [0,1]^d$ 
and $\calX$ is a {\it minimally smooth domain} in a sense of \cite{Stein:Book:1970}. 
$\Px$ has a density $p(x)$ which is lower bounded as 
$
p(x) \geq c_0~~(\forall x \in \supp(\Px))
$
on its support.
For $2m > d$ and $m \geq 3$, 
the activation function satisfies %sufficiently smooth:
$
\sigma \in \calC^m(\Real)
$
%(H\"older class)
and 
$\fstar$ is included in the Sobolev space $W_2^{m}(\calX)$ defined on $\calX$ (see \cite{devore1993besov} for its definition). 
\end{Assumption}

The following theorem gives an upper-bound of the probability of ``perfect classification'' for the estimator.
%We have %can show %the Bayes classifier can be obtained with high probability. Especially, 
More specifically, it shows the error probability converges in an {\it exponential rate}.
\begin{Theorem}\label{thm:ClassificationFastRate}
%Suppose that we execute the optimization upto $k$ steps over $n$ times repetition and choose the best one.
Under Assumptions \ref{ass:StrongLowNoise} and \ref{ass:ClassSupport}, the convergence in Theorem \ref{thm:ExcessRiskConvRate} holds for $s=1$.
%Let the convergence rate in the right hand side of \Eqref{eq:ExpectedLossConvRate} in Theorem \ref{thm:ExcessRiskConvRate} be $\ustar$, and 
Let $\gstar(x) = \sign(P(Y=1|X=x)-1/2)$ be the Bayes classifier.
If the sample size $n$ is sufficiently large and $\lambda, \beta$ are appropriately chosen, %, % and the parameters $\beta$ and $\lambda$ are appropriately chosen,
%so that $\ustar \leq \delta^{2m/(2m-d)}$, 
then the classification error converges exponentially with respect to $\beta$ and $k$: 
%then, there exists $\beta^*, \lambda^* > 0$  (depending on $\delta$ but independent of $n$) such that 
\begin{align*}
%& P( \EE[\boldone[\sign(f_{W_k}(X)) \neq Y]] - \EE[\boldone[\gstar(X) \neq Y]] > 0 ) \\
&  \EE[\pi_k(\{ W_k \in \cH \mid P_X(\sign(f_{W_k}(X)) = \gstar(X)) \neq 1\})]  %\\
%& 
\lesssim
%\underbrace{
\frac{\Xi_k}{\delta^{2m/(2m-d)}}
%\Xi_k \delta^{- \tfrac{2m}{2m-d}}
%}_{\text{Optimization error}} 
+ 
%\underbrace{
\exp(- c' \beta \delta^{\tfrac{2m}{2m-d}}).
 %\delta^{2m/(2m-d)}
%}_{\text{Statistical error}}.
%\exp(- C  n ),
\end{align*}
%for $k \geq \Omega(\log(n))$ and sufficiently small $\eta$.
\end{Theorem}
The proof is given in Appendix \ref{sec:ProofOfClassificationError}.  
This theorem states that if we choose the step size $\eta$ sufficiently small, then the error probability converges exponentially as $k$ and $\beta$ increase. %the iteration number $k$ and the inverse temperature $\beta$ increase.
Even if the first term of the right hand side is larger than the second term, 
we can make this  %the first term as small 
as small as the second term by running the algorithm several times and picking up the best one with respect to validation error if $\Xi_k \ll 1$ (see Appendix \ref{sec:ProofOfClassificationError} for this discussion).

\section{Conclusion}
In this paper, we have formulated the deep learning training as a transportation map estimation and analyzed its convergence and generalization error through the infinite dimensional Langevin dynamics.
Unlike exiting analysis, our formulation can incorporate spatial correlation of noise and achieve global convergence without taking the limit of infinite width.
The generalization analysis reveals the dynamics achieves a stable estimator with $O(1/\sqrt{n})$ convergence of generalization error and shows fast learning rate of the excess risk.
Finally, we have shown a convergence rate of excess risk for regression and classification.
The rate for regression recovers the minimax optimal rate known in Bayesian nonparametrics 
and that for classification achieves exponential convergence under the strong low noise condition.

\section*{Broader impact}
{\bf Benefit} Since deep learning is used in several applications across broad range of areas, our theoretical analysis about optimization of deep learning would influence wide range of areas in terms of understanding of the algorithmic behavior. 
One of the biggest criticisms on deep learning is its poor explainability and interpretability.
Our work on optimization analysis of deep learning can much improve explainability and would facilitate its usage. This is quite important step toward trustworthy machine learning.

{\bf Potential risk} On the other hand, this is purely theoretical work and thus would not directly bring on severe ethical issues.
However, misunderstanding of theoretical work would cause misuse of its statement to conduct an intensional opinion making.
To avoid such a potential risk, we made our best effort to minimize technical ambiguity in our paper presentation.

\section*{Acknowledgment}

I would like to thank Atsushi Nitanda for insightful comments. 
%, and would like to thank Boris Muzellec, Kanji Sato, and Mathurin Massias for their help for developing the theory of the inifinite dimensional Langevin dynamis.
 TS was partially supported by JSPS KAKENHI (18K19793,18H03201, and 20H00576), Japan Digital Design, and JST CREST.

\bibliography{main} 

\begin{thebibliography}{79}
\providecommand{\natexlab}[1]{#1}
\providecommand{\url}[1]{\texttt{#1}}
\expandafter\ifx\csname urlstyle\endcsname\relax
  \providecommand{\doi}[1]{doi: #1}\else
  \providecommand{\doi}{doi: \begingroup \urlstyle{rm}\Url}\fi

\bibitem[Allen-Zhu et~al.(2019)Allen-Zhu, Li, and Song]{allen2019convergence}
Z.~Allen-Zhu, Y.~Li, and Z.~Song.
\newblock A convergence theory for deep learning via over-parameterization.
\newblock In \emph{Proceedings of International Conference on Machine
  Learning}, pp.\  242--252, 2019.

\bibitem[Arora et~al.(2018)Arora, Ge, Neyshabur, and Zhang]{pmlr-v80-arora18b}
S.~Arora, R.~Ge, B.~Neyshabur, and Y.~Zhang.
\newblock Stronger generalization bounds for deep nets via a compression
  approach.
\newblock In \emph{Proceedings of the 35th International Conference on Machine
  Learning}, volume~80 of \emph{Proceedings of Machine Learning Research}, pp.\
   254--263. PMLR, 2018.

\bibitem[Arora et~al.(2019)Arora, Du, Hu, Li, and Wang]{arora2019fine}
S.~Arora, S.~Du, W.~Hu, Z.~Li, and R.~Wang.
\newblock Fine-grained analysis of optimization and generalization for
  overparameterized two-layer neural networks.
\newblock In \emph{Proceedings of the 36th International Conference on Machine
  Learning}, volume~97 of \emph{Proceedings of Machine Learning Research}, pp.\
   322--332. PMLR, 2019.

\bibitem[Audibert et~al.(2007)Audibert, Tsybakov, et~al.]{audibert2007fast}
J.-Y. Audibert, A.~B. Tsybakov, et~al.
\newblock Fast learning rates for plug-in classifiers.
\newblock \emph{The Annals of statistics}, 35\penalty0 (2):\penalty0 608--633,
  2007.

\bibitem[Bartlett et~al.(2005)Bartlett, Bousquet, and
  Mendelson]{LocalRademacher}
P.~Bartlett, O.~Bousquet, and S.~Mendelson.
\newblock Local {R}ademacher complexities.
\newblock \emph{The Annals of Statistics}, 33\penalty0 (4):\penalty0
  1487--1537, 2005.

\bibitem[Bartlett et~al.(2017{\natexlab{a}})Bartlett, Foster, and
  Telgarsky]{bartlett2017spectrally:arXiv}
P.~Bartlett, D.~J. Foster, and M.~Telgarsky.
\newblock Spectrally-normalized margin bounds for neural networks.
\newblock \emph{arXiv preprint arXiv:1706.08498}, 2017{\natexlab{a}}.

\bibitem[Bartlett \& Mendelson(2006)Bartlett and
  Mendelson]{bartlett2006empirical}
P.~L. Bartlett and S.~Mendelson.
\newblock Empirical minimization.
\newblock \emph{Probability theory and related fields}, 135\penalty0
  (3):\penalty0 311--334, 2006.

\bibitem[Bartlett et~al.(2017{\natexlab{b}})Bartlett, Foster, and
  Telgarsky]{bartlett2017spectrally}
P.~L. Bartlett, D.~J. Foster, and M.~J. Telgarsky.
\newblock Spectrally-normalized margin bounds for neural networks.
\newblock In \emph{Advances in Neural Information Processing Systems}, pp.\
  6241--6250, 2017{\natexlab{b}}.

\bibitem[Bennett \& Sharpley(1988)Bennett and
  Sharpley]{Book:Bennett+Sharpley:88}
C.~Bennett and R.~Sharpley.
\newblock \emph{Interpolation of Operators}.
\newblock Academic Press, Boston, 1988.

\bibitem[Borell(1975)]{borell1975brunn}
C.~Borell.
\newblock The {B}runn-{M}inkowski inequality in gauss space.
\newblock \emph{Inventiones mathematicae}, 30\penalty0 (2):\penalty0 207--216,
  1975.

\bibitem[Boucheron et~al.(2013)Boucheron, Lugosi, and
  Massart]{boucheron2013concentration}
S.~Boucheron, G.~Lugosi, and P.~Massart.
\newblock \emph{Concentration Inequalities: A Nonasymptotic Theory of
  Independence}.
\newblock OUP Oxford, 2013.

\bibitem[Bousquet(2002)]{BousquetBenett}
O.~Bousquet.
\newblock A {B}ennett concentration inequality and its application to suprema
  of empirical process.
\newblock \emph{Comptes Rendus de l'Acad{\'e}mie des Sciences - Series I -
  Mathematics}, 334:\penalty0 495--500, 2002.

\bibitem[Br{\'e}hier \& Kopec(2016)Br{\'e}hier and Kopec]{Brehier16}
C.-E. Br{\'e}hier and M.~Kopec.
\newblock Approximation of the invariant law of {SPDEs}: error analysis using a
  {Poisson} equation for a full-discretization scheme.
\newblock \emph{IMA Journal of Numerical Analysis}, 37\penalty0 (3):\penalty0
  1375--1410, 07 2016.

\bibitem[Br{\'e}hier(2017)]{brehier:cel-01633504}
C.-E. Br{\'e}hier.
\newblock Lecture notes: Invariant distributions for parabolic {SPDEs} and
  their numerical approximations, November 2017.
\newblock {HAL ID: cel-01633504}.

\bibitem[Br{\'e}hier(2020)]{brehier2020influence}
C.-E. Br{\'e}hier.
\newblock Influence of the regularity of the test functions for weak
  convergence in numerical discretization of spdes.
\newblock \emph{Journal of Complexity}, 56:\penalty0 101424, 2020.

\bibitem[Cao \& Gu(2019{\natexlab{a}})Cao and Gu]{cao2019generalization}
Y.~Cao and Q.~Gu.
\newblock A generalization theory of gradient descent for learning
  over-parameterized deep {ReLU} networks.
\newblock \emph{arXiv preprint arXiv:1902.01384}, 2019{\natexlab{a}}.

\bibitem[Cao \& Gu(2019{\natexlab{b}})Cao and Gu]{cao2019generalization_b}
Y.~Cao and Q.~Gu.
\newblock Generalization bounds of stochastic gradient descent for wide and
  deep neural networks.
\newblock \emph{arXiv preprint arXiv:1905.13210}, 2019{\natexlab{b}}.

\bibitem[Chizat \& Bach(2018)Chizat and Bach]{chizat2018note}
L.~Chizat and F.~Bach.
\newblock A note on lazy training in supervised differentiable programming.
\newblock \emph{arXiv preprint arXiv:1812.07956}, 2018.

\bibitem[Da~Prato \& Zabczyk(1992)Da~Prato and Zabczyk]{DaPrato_Zabczyk92}
G.~Da~Prato and J.~Zabczyk.
\newblock Non-explosion, boundedness and ergodicity for stochastic semilinear
  equations.
\newblock \emph{Journal of Differential Equations}, 98:\penalty0 181--195,
  1992.

\bibitem[Da~Prato \& Zabczyk(2014)Da~Prato and Zabczyk]{da_prato_zabczyk_2014}
G.~Da~Prato and J.~Zabczyk.
\newblock \emph{Stochastic Equations in Infinite Dimensions}.
\newblock Encyclopedia of Mathematics and its Applications. Cambridge
  University Press, 2 edition, 2014.

\bibitem[DeVore \& Sharpley(1993)DeVore and Sharpley]{devore1993besov}
R.~A. DeVore and R.~C. Sharpley.
\newblock {Besov} spaces on domains in $\mathbb{R}^d$.
\newblock \emph{Transactions of the American Mathematical Society},
  335\penalty0 (2):\penalty0 843--864, 1993.

\bibitem[Du et~al.(2019{\natexlab{a}})Du, Lee, Li, Wang, and
  Zhai]{du2019gradient}
S.~Du, J.~Lee, H.~Li, L.~Wang, and X.~Zhai.
\newblock Gradient descent finds global minima of deep neural networks.
\newblock In \emph{International Conference on Machine Learning}, pp.\
  1675--1685, 2019{\natexlab{a}}.

\bibitem[Du et~al.(2019{\natexlab{b}})Du, Zhai, Poczos, and
  Singh]{du2018gradient}
S.~S. Du, X.~Zhai, B.~Poczos, and A.~Singh.
\newblock Gradient descent provably optimizes over-parameterized neural
  networks.
\newblock \emph{International Conference on Learning Representations 7},
  2019{\natexlab{b}}.

\bibitem[Erdogdu et~al.(2018)Erdogdu, Mackey, and Shamir]{NIPS2018_8175}
M.~A. Erdogdu, L.~Mackey, and O.~Shamir.
\newblock Global non-convex optimization with discretized diffusions.
\newblock In \emph{Advances in Neural Information Processing Systems 31}, pp.\
  9671--9680. 2018.

\bibitem[Ghosal et~al.(2000)Ghosal, Ghosh, and {van der
  Vaart}]{AS:Ghosal+Ghosh+Vaart:2000}
S.~Ghosal, J.~K. Ghosh, and A.~W. {van der Vaart}.
\newblock Convergence rates of posterior distributions.
\newblock \emph{The Annals of Statistics}, 28\penalty0 (2):\penalty0 500--531,
  2000.

\bibitem[Ghosal \& van~der Vaart(2017)Ghosal and van~der
  Vaart]{ghosal_van_der_vaart:Book:2017}
S.~Ghosal and A.~van~der Vaart.
\newblock \emph{Fundamentals of Nonparametric Bayesian Inference}.
\newblock Cambridge Series in Statistical and Probabilistic Mathematics.
  Cambridge University Press, 2017.

\bibitem[Gin{\'e} \& Koltchinskii(2006)Gin{\'e} and
  Koltchinskii]{gine2006concentration}
E.~Gin{\'e} and V.~Koltchinskii.
\newblock Concentration inequalities and asymptotic results for ratio type
  empirical processes.
\newblock \emph{The Annals of Probability}, 34\penalty0 (3):\penalty0
  1143--1216, 2006.

\bibitem[Hairer(2002)]{Hairer:2002}
M.~Hairer.
\newblock Exponential mixing properties of stochastic {PDE}s through asymptotic
  coupling.
\newblock \emph{Probability Theory and Related Fields}, 124\penalty0
  (3):\penalty0 345--380, 2002.

\bibitem[Hayakawa \& Suzuki(2020)Hayakawa and Suzuki]{HAYAKAWA2020343}
S.~Hayakawa and T.~Suzuki.
\newblock On the minimax optimality and superiority of deep neural network
  learning over sparse parameter spaces.
\newblock \emph{Neural Networks}, 123:\penalty0 343--361, 2020.

\bibitem[Imaizumi \& Fukumizu(2019)Imaizumi and Fukumizu]{pmlr-v89-imaizumi19a}
M.~Imaizumi and K.~Fukumizu.
\newblock Deep neural networks learn non-smooth functions effectively.
\newblock In K.~Chaudhuri and M.~Sugiyama (eds.), \emph{Proceedings of Machine
  Learning Research}, volume~89 of \emph{Proceedings of Machine Learning
  Research}, pp.\  869--878. PMLR, 16--18 Apr 2019.

\bibitem[Ioffe \& Szegedy(2015)Ioffe and Szegedy]{pmlr-v37-ioffe15}
S.~Ioffe and C.~Szegedy.
\newblock Batch normalization: Accelerating deep network training by reducing
  internal covariate shift.
\newblock In F.~Bach and D.~Blei (eds.), \emph{Proceedings of the 32nd
  International Conference on Machine Learning}, volume~37 of \emph{Proceedings
  of Machine Learning Research}, pp.\  448--456, Lille, France, 07--09 Jul
  2015. PMLR.

\bibitem[Jacot et~al.(2018)Jacot, Gabriel, and Hongler]{jacot2018neural}
A.~Jacot, F.~Gabriel, and C.~Hongler.
\newblock Neural tangent kernel: Convergence and generalization in neural
  networks.
\newblock In \emph{Advances in Neural Information Processing Systems 31}, pp.\
  8580--8589, 2018.

\bibitem[Jacquot \& Royer(1995)Jacquot and Royer]{Jacquot+Gilles:1995}
S.~Jacquot and G.~Royer.
\newblock Ergodicit\'{e} d'une classe d'\'{e}quations aux d\'{e}riv\'{e}es
  partielles stochastiques.
\newblock \emph{C. R. Acad. Sci. Paris S\'{e}r. I Math.}, 320\penalty0
  (2):\penalty0 231--236, 1995.

\bibitem[Ji \& Telgarsky(2019)Ji and Telgarsky]{ji2019polylogarithmic}
Z.~Ji and M.~Telgarsky.
\newblock Polylogarithmic width suffices for gradient descent to achieve
  arbitrarily small test error with shallow {ReLU} networks.
\newblock \emph{arXiv preprint arXiv:1909.12292}, 2019.

\bibitem[Koltchinskii(2006)]{Koltchinskii}
V.~Koltchinskii.
\newblock Local {R}ademacher complexities and oracle inequalities in risk
  minimization.
\newblock \emph{The Annals of Statistics}, 34:\penalty0 2593--2656, 2006.

\bibitem[Koltchinskii \& Beznosova(2005)Koltchinskii and
  Beznosova]{ExpRate:Koltchinskii+Beznosova:2005}
V.~Koltchinskii and O.~Beznosova.
\newblock Exponential convergence rates in classification.
\newblock In P.~Auer and R.~Meir (eds.), \emph{Learning Theory}, pp.\
  295--307, Berlin, Heidelberg, 2005. Springer Berlin Heidelberg.

\bibitem[Krogh \& Hertz(1992)Krogh and Hertz]{krogh1992simple}
A.~Krogh and J.~A. Hertz.
\newblock A simple weight decay can improve generalization.
\newblock In \emph{Advances in neural information processing systems}, pp.\
  950--957, 1992.

\bibitem[Lata{\l}a \& Matlak(2017)Lata{\l}a and Matlak]{Latala2017}
R.~Lata{\l}a and D.~Matlak.
\newblock \emph{Royen's Proof of the Gaussian Correlation Inequality}, pp.\
  265--275.
\newblock Springer International Publishing, 2017.

\bibitem[Maslowski(1989)]{Maslowski:1989}
B.~Maslowski.
\newblock Strong {F}eller property for semilinear stochastic evolution
  equations and applications.
\newblock In \emph{Stochastic systems and optimization ({W}arsaw, 1988)},
  volume 136 of \emph{Lect. Notes Control Inf. Sci.}, pp.\  210--224. Springer,
  Berlin, 1989.

\bibitem[Mei et~al.(2018)Mei, Montanari, and Nguyen]{MeiE7665}
S.~Mei, A.~Montanari, and P.-M. Nguyen.
\newblock A mean field view of the landscape of two-layer neural networks.
\newblock \emph{Proceedings of the National Academy of Sciences}, 115\penalty0
  (33):\penalty0 E7665--E7671, 2018.

\bibitem[Mei et~al.(2019)Mei, Misiakiewicz, and Montanari]{pmlr-v99-mei19a}
S.~Mei, T.~Misiakiewicz, and A.~Montanari.
\newblock Mean-field theory of two-layers neural networks: dimension-free
  bounds and kernel limit.
\newblock In A.~Beygelzimer and D.~Hsu (eds.), \emph{Proceedings of the
  Thirty-Second Conference on Learning Theory}, volume~99 of \emph{Proceedings
  of Machine Learning Research}, pp.\  2388--2464, Phoenix, USA, 25--28 Jun
  2019. PMLR.

\bibitem[Mendelson(2002)]{IEEEIT:Mendelson:2002}
S.~Mendelson.
\newblock Improving the sample complexity using global data.
\newblock \emph{IEEE Transactions on Information Theory}, 48:\penalty0
  1977--1991, 2002.

\bibitem[Mohri et~al.(2012)Mohri, Rostamizadeh, and
  Talwalkar]{mohri2012foundations}
M.~Mohri, A.~Rostamizadeh, and A.~Talwalkar.
\newblock Foundations of machine learning.
\newblock 2012.

\bibitem[Mou et~al.(2018)Mou, Wang, Zhai, and Zheng]{pmlr-v75-mou18a}
W.~Mou, L.~Wang, X.~Zhai, and K.~Zheng.
\newblock Generalization bounds of {SGLD} for non-convex learning: Two
  theoretical viewpoints.
\newblock In \emph{Proceedings of the 31st Conference On Learning Theory},
  volume~75 of \emph{Proceedings of Machine Learning Research}, pp.\  605--638.
  PMLR, 2018.

\bibitem[Muzellec et~al.(2020)Muzellec, Sato, Massias, and
  Suzuki]{muzellec2020dimensionfree}
B.~Muzellec, K.~Sato, M.~Massias, and T.~Suzuki.
\newblock Dimension-free convergence rates for gradient {Langevin} dynamics in
  {RKHS}.
\newblock \emph{arXiv preprint 2003.00306}, 2020.

\bibitem[Neyshabur et~al.(2015)Neyshabur, Tomioka, and
  Srebro]{COLT:Neyshabur+Tomioka+Srebro:2015}
B.~Neyshabur, R.~Tomioka, and N.~Srebro.
\newblock Norm-based capacity control in neural networks.
\newblock In \emph{Proceedings of The 28th Conference on Learning Theory}, pp.\
   1376--1401, Montreal Quebec, 2015.

\bibitem[Nitanda \& Suzuki(2017)Nitanda and Suzuki]{nitanda2017prticle}
A.~Nitanda and T.~Suzuki.
\newblock Stochastic particle gradient descent for infinite ensembles.
\newblock \emph{arXiv preprint arXiv:1712.05438}, 2017.

\bibitem[Nitanda \& Suzuki(2019{\natexlab{a}})Nitanda and
  Suzuki]{nitanda2019refined}
A.~Nitanda and T.~Suzuki.
\newblock Refined generalization analysis of gradient descent for
  over-parameterized two-layer neural networks with smooth activations on
  classification problems.
\newblock \emph{arXiv preprint arXiv:1905.09870}, 2019{\natexlab{a}}.

\bibitem[Nitanda \& Suzuki(2019{\natexlab{b}})Nitanda and
  Suzuki]{pmlr-v89-nitanda19a}
A.~Nitanda and T.~Suzuki.
\newblock Stochastic gradient descent with exponential convergence rates of
  expected classification errors.
\newblock In K.~Chaudhuri and M.~Sugiyama (eds.), \emph{Proceedings of Machine
  Learning Research}, volume~89 of \emph{Proceedings of Machine Learning
  Research}, pp.\  1417--1426. PMLR, 16--18 Apr 2019{\natexlab{b}}.

\bibitem[Oymak \& Soltanolkotabi(2020)Oymak and
  Soltanolkotabi]{oymak2020towards}
S.~Oymak and M.~Soltanolkotabi.
\newblock Towards moderate overparameterization: global convergence guarantees
  for training shallow neural networks.
\newblock \emph{IEEE Journal on Selected Areas in Information Theory}, 2020.

\bibitem[Raginsky et~al.(2017)Raginsky, Rakhlin, and
  Telgarsky]{Raginsky_Rakhlin_Telgarsky2017}
M.~Raginsky, A.~Rakhlin, and M.~Telgarsky.
\newblock Non-convex learning via stochastic gradient {Langevin} dynamics: a
  nonasymptotic analysis.
\newblock volume~65 of \emph{Proceedings of Machine Learning Research}, pp.\
  1674--1703. PMLR, 2017.

\bibitem[Rivasplata et~al.(2020)Rivasplata, Kuzborskij, Szepesv{\'a}ri, and
  Shawe-Taylor]{NIPS:Rivasplata:2020}
O.~Rivasplata, I.~Kuzborskij, C.~Szepesv{\'a}ri, and J.~Shawe-Taylor.
\newblock {PAC}-{Bayes} analysis beyond the usual bounds.
\newblock In \emph{Advances in Neural Information Processing Systems 34}. 2020.
\newblock to appear.

\bibitem[Rotskoff \& Vanden-Eijnden(2018{\natexlab{a}})Rotskoff and
  Vanden-Eijnden]{NIPS:Rotskoff&Vanden-Eijnden:2018}
G.~Rotskoff and E.~Vanden-Eijnden.
\newblock Parameters as interacting particles: long time convergence and
  asymptotic error scaling of neural networks.
\newblock In S.~Bengio, H.~Wallach, H.~Larochelle, K.~Grauman, N.~Cesa-Bianchi,
  and R.~Garnett (eds.), \emph{Advances in Neural Information Processing
  Systems 31}, pp.\  7146--7155. Curran Associates, Inc., 2018{\natexlab{a}}.

\bibitem[Rotskoff \& Vanden-Eijnden(2018{\natexlab{b}})Rotskoff and
  Vanden-Eijnden]{rotskoff2019trainability}
G.~M. Rotskoff and E.~Vanden-Eijnden.
\newblock Trainability and accuracy of neural networks: An interacting particle
  system approach.
\newblock \emph{arXiv preprint arXiv:1805.00915}, 2018{\natexlab{b}}.

\bibitem[Royen(2014)]{royen2014simple}
T.~Royen.
\newblock A simple proof of the gaussian correlation conjecture extended to
  multivariate gamma distributions.
\newblock \emph{arXiv preprint arXiv:1408.1028}, 2014.

\bibitem[Schmidt-Hieber(2020)]{AoS:Schmidt-Hieber:2020}
J.~Schmidt-Hieber.
\newblock Nonparametric regression using deep neural networks with {ReLU}
  activation function.
\newblock \emph{The Annals of Statistics}, 48\penalty0 (4), 2020.

\bibitem[Shardlow(1999)]{Shardlow:1999}
T.~Shardlow.
\newblock Geometric ergodicity for stochastic {PDE}s.
\newblock \emph{Stochastic Analysis and Applications}, 17\penalty0
  (5):\penalty0 857--869, 1999.

\bibitem[Sirignano \& Spiliopoulos(2018)Sirignano and
  Spiliopoulos]{sirignano2018mean}
J.~Sirignano and K.~Spiliopoulos.
\newblock Mean field analysis of neural networks.
\newblock \emph{arXiv preprint arXiv:1805.01053}, 2018.

\bibitem[Sowers(1992)]{Sowers:1992}
R.~Sowers.
\newblock Large deviations for the invariant measure of a reaction-diffusion
  equation with non-{G}aussian perturbations.
\newblock \emph{Probability Theory and Related Fields}, 92\penalty0
  (3):\penalty0 393--421, 1992.

\bibitem[Srivastava et~al.(2014)Srivastava, Hinton, Krizhevsky, Sutskever, and
  Salakhutdinov]{srivastava2014dropout}
N.~Srivastava, G.~Hinton, A.~Krizhevsky, I.~Sutskever, and R.~Salakhutdinov.
\newblock Dropout: a simple way to prevent neural networks from overfitting.
\newblock \emph{The Journal of Machine Learning Research}, 15\penalty0
  (1):\penalty0 1929--1958, 2014.

\bibitem[Stein(1970)]{Stein:Book:1970}
E.~M. Stein.
\newblock \emph{Singular Integrals and Differentiability Properties of
  Functions}.
\newblock Princeton University Press, 1970.

\bibitem[Steinwart(2019)]{steinwart2019convergence}
I.~Steinwart.
\newblock Convergence types and rates in generic karhunen-lo{\`e}ve expansions
  with applications to sample path properties.
\newblock \emph{Potential Analysis}, 51\penalty0 (3):\penalty0 361--395, 2019.

\bibitem[Steinwart \& Christmann(2008)Steinwart and
  Christmann]{Book:Steinwart:2008}
I.~Steinwart and A.~Christmann.
\newblock \emph{Support Vector Machines}.
\newblock Springer, 2008.

\bibitem[Steinwart et~al.(2009)Steinwart, Hush, and
  Scovel]{COLT:Steinwart+etal:2009}
I.~Steinwart, D.~Hush, and C.~Scovel.
\newblock Optimal rates for regularized least squares regression.
\newblock In \emph{Proceedings of the Annual Conference on Learning Theory},
  pp.\  79--93, 2009.

\bibitem[Suzuki(2019)]{suzuki2018adaptivity}
T.~Suzuki.
\newblock Adaptivity of deep re{LU} network for learning in besov and mixed
  smooth besov spaces: optimal rate and curse of dimensionality.
\newblock In \emph{International Conference on Learning Representations}, 2019.

\bibitem[Suzuki \& Nitanda(2019)Suzuki and Nitanda]{suzuki2019deepIntrinsicDim}
T.~Suzuki and A.~Nitanda.
\newblock Deep learning is adaptive to intrinsic dimensionality of model
  smoothness in anisotropic besov space.
\newblock \emph{arXiv preprint arXiv:1910.12799}, 2019.

\bibitem[Suzuki et~al.(2020)Suzuki, Abe, and Nishimura]{suzuki2020compression}
T.~Suzuki, H.~Abe, and T.~Nishimura.
\newblock Compression based bound for non-compressed network: Unified
  generalization error analysis of large compressible deep neural network.
\newblock In \emph{International Conference on Learning Representations}, 2020.

\bibitem[Talagrand(1996)]{Talagrand2}
M.~Talagrand.
\newblock New concentration inequalities in product spaces.
\newblock \emph{Inventiones Mathematicae}, 126:\penalty0 505--563, 1996.

\bibitem[Triebel(1983)]{triebel1983theory}
H.~Triebel.
\newblock \emph{Theory of Function Spaces}.
\newblock Monographs in Mathematics. Birkh{\"a}user Verlag, 1983.

\bibitem[Tsybakov et~al.(2004)]{tsybakov2004optimal}
A.~B. Tsybakov et~al.
\newblock Optimal aggregation of classifiers in statistical learning.
\newblock \emph{The Annals of Statistics}, 32\penalty0 (1):\penalty0 135--166,
  2004.

\bibitem[{van der Vaart} \& {van Zanten}(2008){van der Vaart} and {van
  Zanten}]{AS:Vaart&Zanten:2008}
A.~W. {van der Vaart} and J.~H. {van Zanten}.
\newblock Rates of contraction of posterior distributions based on {G}aussian
  process priors.
\newblock \emph{The Annals of Statistics}, 36\penalty0 (3):\penalty0
  1435--1463, 2008.

\bibitem[{van der Vaart} \& {van Zanten}(2011){van der Vaart} and {van
  Zanten}]{JMLR:Vaart&Zanten:2011}
A.~W. {van der Vaart} and J.~H. {van Zanten}.
\newblock Information rates of nonparametric gaussian process methods.
\newblock \emph{Journal of Machine Learning Research}, 12:\penalty0 2095--2119,
  2011.

\bibitem[van Erven et~al.(2015)van Erven, Gr{\"u}nwald, Mehta, Reid, and
  Williamson]{van2015fast}
T.~van Erven, P.~D. Gr{\"u}nwald, N.~A. Mehta, M.~D. Reid, and R.~C.
  Williamson.
\newblock Fast rates in statistical and online learning.
\newblock \emph{Journal of Machine Learning Research}, 16:\penalty0 1793--1861,
  2015.

\bibitem[Wager et~al.(2013)Wager, Wang, and Liang]{NIPS2013_4882}
S.~Wager, S.~Wang, and P.~S. Liang.
\newblock Dropout training as adaptive regularization.
\newblock In C.~J.~C. Burges, L.~Bottou, M.~Welling, Z.~Ghahramani, and K.~Q.
  Weinberger (eds.), \emph{Advances in Neural Information Processing Systems
  26}, pp.\  351--359. Curran Associates, Inc., 2013.

\bibitem[Wainwright(2019)]{wainwright2019high}
M.~Wainwright.
\newblock \emph{High-Dimensional Statistics: A Non-Asymptotic Viewpoint}.
\newblock Cambridge Series in Statistical and Probabilistic Mathematics.
  Cambridge University Press, 2019.

\bibitem[Weinan et~al.(2019)Weinan, Ma, and Wu]{weinan2019comparative}
E.~Weinan, C.~Ma, and L.~Wu.
\newblock A comparative analysis of optimization and generalization properties
  of two-layer neural network and random feature models under gradient descent
  dynamics.
\newblock \emph{Science China Mathematics}, pp.\  1--24, 2019.

\bibitem[Welling \& Teh(2011)Welling and Teh]{Welling_Teh11}
M.~Welling and Y.-W. Teh.
\newblock Bayesian learning via stochastic gradient {L}angevin dynamics.
\newblock In \emph{{ICML}}, pp.\  681--688, 2011.

\bibitem[Yang \& Barron(1999)Yang and Barron]{AS:Yang+Barron:99}
Y.~Yang and A.~Barron.
\newblock Information-theoretic determination of minimax rates of convergence.
\newblock \emph{The Annals of Statistics}, 27\penalty0 (5):\penalty0
  1564--1599, 1999.

\bibitem[Zou \& Gu(2019)Zou and Gu]{zou2019improved}
D.~Zou and Q.~Gu.
\newblock An improved analysis of training over-parameterized deep neural
  networks.
\newblock In \emph{Advances in Neural Information Processing Systems}, pp.\
  2053--2062, 2019.

\end{thebibliography}
\bibliographystyle{icml2016_mod}%abbrvnat}%icml2016_mod}

%!TEX root = NIPS2020_supplementary.tex

\newpage 
\appendix

%\section*{Appendix}

~\\

\begin{center}
{\bf \LARGE ------Appendix------}
\end{center}

%\noindent {\bf \large Additional notation}
%
%For a measurable function $f:\calX \to \Real$, we define 
%$$
%\|f\|_\infty := \sup_{x \in \supp(\Px)} |f(x)|.
%$$

\section{Proof of Proposition \ref{prop:WeakConvergence}}
\label{sec:ProofOfWeakConv}

We apply the result \cite{muzellec2020dimensionfree}. Let $\{Z_n\}_{n \in \Natural}$ be a dynamics obeying 
\begin{align*}
Z_{n+1} = S_\eta Z_n+ \sqrt{\eta/\beta}S_\eta \varepsilon_n,
\end{align*}
with $Z_0 = 0$.
Let $k(p) := \sup_{n \geq 0}\EE (\norm{Z_n}^p)$ for $p > 0$, then it is known that $k(p) < \infty$ for any $p > 0$.
Let $\{Z_n\}_{n \in \Natural}$ solve $Z_0 = 0$ and 
with $\beta > \eta$. Then, we can show that $k(p) := \sup_{n \geq 0}\EE (\norm{Z_n}^p) < \infty$ $(\forall p > 0)$\cite{muzellec2020dimensionfree}.
Using $k(p)$, we define 
$
%\rho = \frac{1}{1 + \lambda\eta/\mu_0},~~
b' = \frac{\mu_0}{\lambda}B + k(1).
$
We will show that $k(1) \leq \frac{c_\mu}{\beta \lambda}$. Then, we can see that $b' \leq b$. % defined in the main text. 
Now, we show $k(1) \leq \frac{c_\mu}{\beta \lambda}$.
First, note that 
\begin{align*}
Z_n = \sqrt{\frac{\eta}{\beta}} \sum_{\ell=0}^{n-1} S_\eta^{n-\ell} \epsilon_\ell.
\end{align*}
Therefore, we have 
\begin{align*}
\EE[\|Z_n\|^2] 
& = \frac{\eta}{\beta} \sum_{\ell=0}^{n-1} \Tr[S_\eta^{2(n-\ell)}]
= \frac{\eta}{\beta} \Tr\left[(S_\eta^2 - S_\eta^{2n})(I - S_\eta^2)^{-1}\right]
\leq 
\frac{\eta}{\beta} \Tr\left[S_\eta^2(I - S_\eta^2)^{-1}\right] \\
& = \frac{\eta}{\beta} \sum_{k=0}^\infty \left(\frac{1}{1 + \eta \lambda/\mu_k} \right)^2 \left( 1 - \frac{1}{(1 + \eta \lambda/\mu_k)^2}\right)^{-1}
= \frac{\eta}{\beta} \sum_{k=0}^\infty \frac{1}{(1 + \eta \lambda/\mu_k)^2 -1} \\
& \leq 
\frac{\eta}{\beta} \sum_{k=0}^\infty \frac{1}{ 2 \eta \lambda/\mu_k} 
\leq \frac{1}{2\beta \lambda}\sum_{k=0}^\infty c_\mu (k+1)^{-2}
\leq \frac{c_\mu}{\beta \lambda}.
\end{align*}
Then, Jensen's inequality yields $k(1) = \EE[\|Z_n\|_{\cH}] \leq \sqrt{\EE[\|Z_n\|_{\cH}^2]} \leq \sqrt{\frac{c_\mu}{\beta \lambda}}$.

Let $\phi: \cH \to \mathbb{R}$ be a test function satisfying $|\phi(\cdot)| \leq V(\cdot)$ and $\|\phi(x) - \phi(y)\| \leq M \|x - y\|~(x,y \in \cH)$ for $M > 0$.
Then, \cite{muzellec2020dimensionfree} showed that there exists a unique invariant measure $\mu_{{\eta}}$ and 
the following exponential convergence of the expectation of $\phi$ holds:
\begin{align}\label{eq:geom_ergo}
& 
|\EE_{x_0}[\phi(X_n)]  -  \EE[\phi(X^{\mu_\eta})]| \leq 
%|P_n\phi(x_0) - \int\phi \,\mathrm{d}\mu^\eta| \leq %\frac{C}{\gamma^T(1 - \sqrt{1 - \delta})} V(x_0) 
C_{x_0} \exp(- \Lambda^*_\eta (\eta n - 1)).
%C_{\lambda,\mu} \gamma^{n} V(x_0),
\end{align}
%Let $\eta >0, \beta > \eta$ and 
%$P_n \phi(\cdot) \eqdef \int \phi(y) P_n^\eta(\cdot, \dy)$. 
%Let 
%$V(x) = \|x\| + 1$.
%\begin{proposition}[Case $\beta \neq 1$\label{prop:scheme_cvg_general}]
%Under \Cref{assum:smoothness,assum:C2_boundedness,assum:dissipative,assum:eigenvalue_cvg}, for any $0<\kappa<1/2$, $0 < \eta_0$,
where 
\begin{align*}
\Lambda^*_\eta = \frac{\min\left(\frac{\lambda}{2 \mu_0}, \frac{1}{2} \right)}{4 \log(\kappa (\bar{V} + 1)/(1-\delta)) } \delta, ~
%\textstyle 
C_{W_0} = \kappa [\bar{V} + 1] + \frac{\sqrt{2} (\Rbar + b)}{\sqrt{\delta}}
\end{align*}
with $0 < \delta< 1$ satisfying $\delta = \Omega(\exp(-C' \mathrm{poly}(\lambda^{-1})\beta))$, $\bar{b} = \max\{b,1\}$, $\kappa = \bar{b} + 1$ and $\bar{V} = \frac{4 \bar{b}}{\sqrt{(1+\rho^{1/\eta})/2} - \rho^{1/\eta}}$.
To show this we note that $\beta$ in \cite{muzellec2020dimensionfree} is $2\beta$ using $\beta$ in this paper.
The definition of $\delta$ is not explicitly shown in \cite{muzellec2020dimensionfree} (in particular, $\lambda$ is omitted), but we can recover our definition from the proof.
Moreover, \cite{muzellec2020dimensionfree} assumed that 
there exists %$\alpha \in (1/4,1)$ and 
$\lambda_0, C_{\alpha,2} \in (0,\infty)$ such that % $\forall x,h,k \in \cH$, 
%where $\norm{x}_\varepsilon \eqdef \left(\sum_{k \geq 0} (\mu_k)^{2\varepsilon} |\scal{x}{f_k}|^2\right)^{1/2}$.
\begin{align*}
& \norm{\nabla \calLhat(W) - \nabla \calLhat(W')}_{\cH} \leq L \norm{W - W'}_{\cH}~~(\forall W,W' \in \cH), \\
& |\nabla^2 \calLhat(W) \cdot (h,k) | \leq C_{\alpha,2} \|h\|_{\cH} \|k\|_{\alpha}~~(\forall W,h,k \in \cH),
%\norm{\nabla \calLhat(W) - \nabla \calLhat(W')}_{\alpha} \leq L \norm{W - W'}
\end{align*}
instead of our assumption $\norm{\nabla \calLhat(W) - \nabla \calLhat(W')}_{\cH} \leq L \norm{W - W'}_{\alpha}$.
However, we can see that their proof is valid even under our assumption.

Let $C_b^2$ be a set of functions $f: \calH \to \Real$ that is continuously twice differentiable with bounded derivatives.
Under the same setting above, \cite{muzellec2020dimensionfree} also showed that, 
for any $0<\kappa<1/2$, $0 < \eta_0$,
there exists a constant $C$ such that, if the test function $\phi$ satisfies $\phi\in C_b^2$, then for any $0<\eta <\eta_0$, it holds that
\begin{equation}
 \left| \EE[\phi(X^{\mu_\eta})] - \EE[\phi(X^{\pi_\infty})] \right| %\int \phi \mathrm{d}\pi\right| 
\leq C \frac{\norm{\phi}_{0,2}}{\Lambda^*_0} c_\beta \eta^{1/2 - \kappa},
%(1 + \beta) \left(\frac{\eta}{\beta}\right)^{1/2 - \kappa}.
\end{equation}
%\end{proposition} %\bm{To update. According to Taiji, there should be an $O(\exp(\beta))$ factor.} 
where $\|\phi\|_{0,2} := \max\{ \|\phi\|_{\infty}, \sup_{x \in \cH}\|\nabla \phi(x)\|_{\cH}, \sup_{x\in \cH}\|\nabla^2 \phi(x)\|_{\mathcal{B}(\cH)} \}$ for $\phi \in C_b^2$ where $\|\cdot\|_{\mathcal{B}(\cH)}$ is the norm as a linear operator.

Thus, if we let $\phi(\cdot) = \calLhat(\cdot)/\Rbar$, then $\phi$ satisfies the assumption with $M = B/\Rbar$.
Therefore, we obtain that
\begin{align*}
|\EE_{x_0}[\calLhat(X_n)]   - \EE[\calLhat(X^{\pi_\infty})] |
\leq \Rbar \left[C_{x_0} \exp(- \Lambda^*_\eta (\eta n - 1)) + C \frac{\norm{\phi}_{0,2}}{\Lambda^*_0} c_\beta \eta^{1/2 - \kappa}\right].
\end{align*}
This gives the assertion.

Finally, we would like to note that since the assumption is satisfied almost surely, $\calL$ also satisfies the assumption instead of $\calLhat$.
That means the same convergence rate holds also for $\phi(W) = \calL$.

\section{Proof of Theorem \ref{thm:PAC-BayesGenBound}} \label{sec:ProofOfPAC-BayesBound}

\begin{proof}
\cite{NIPS:Rivasplata:2020} proved that, 
for any probability measure $Q$ which is absolutely continuous to $\pi_\infty$,
it holds that 
\begin{align}\label{eq:PAC-Bayes-Basic}
\EE_{W \sim Q}[\calL(W)] \leq \EE_{W \sim Q}[\calLhat(W)] + 
\frac{1}{\sqrt{n}}\mathrm{KL}(Q||\pi_\infty)+
\frac{\Rbar^2}{\sqrt{n}}
\left[2\left(1 + \frac{2\beta}{\sqrt{n}}\right) + \log\left(\frac{1 + e^{\Rbar^2/2}}{\delta}\right)\right],
\end{align}
with probability $1- \delta$. 

On the other hand, Proposition \ref{prop:WeakConvergence} gives that
\begin{align*}
& |\EE_{W_k \sim \pi_k}[\calL(W_k)] - \EE_{W \sim \pi_\infty}[\calL(W)] |
\leq \Xi_{k}, \\
& |\EE_{W_k \sim \pi_k}[\calLhat(W_k)] - \EE_{W \sim \pi_\infty}[\calLhat(W)] |
\leq \Xi_{k}.
\end{align*}
%by adopting $\phi = \calL$ or $\phi = \calLhat$.
Then, by substituting $Q = \pi_\infty$ into \Eqref{eq:PAC-Bayes-Basic} and applying the two inequalities above, 
we obtain the assertion. 
\end{proof}

\section{Proof of Lemma \ref{lem:LipshitzInfty}} \label{sec:ProofOfLipshitzInfty}

%\begin{proof}
By the definition of $f_W$, we have 
\begin{align*}
& f_{W}(x) - f_{W'}(x)  \\
& \leq 
\int \left[ (\bar{W}_2(a) - \bar{W}'_2(a)) \sigma(\bar{W}_1(w)^\top x) +
\bar{W}'_2(a) (\sigma(\bar{W}_1(w)^\top x) - \sigma(\bar{W}'_1(w)^{\top} x)) \right] \dd \rho_0(a,w) \\
& \leq 
\sqrt{\int (\bar{W}_2(a) - \bar{W}'_2(a))^2   \dd \rho_0(a,w)}  % \sigma(\bar{W}_1^\top x) +
+ 
\sqrt{\int (\bar{W}'_2(a))^2 (\sigma(\bar{W}_1(w)^\top x) - \sigma(\bar{W}_1'(w)^{\top} x))^2  \dd \rho_0(a,w) } \\
& \leq 
\sqrt{\int (W_2(a) - W'_2(a))^2   \dd \rho_0(a,w)}  % \sigma(\bar{W}_1^\top x) +
+ 
R \sqrt{\int(\sigma(\bar{W}_1(w)^\top x) - \sigma(\bar{W}_1'(w)^{\top} x))^2  \dd \rho_0(a,w) },
\end{align*}
where we used $|\sigma(x)| \leq 0$, 1-Lipschitz continuity of the clipping operation 
and $|\bar{W}_2'(a)| \leq R$.
By noticing that $\|x\| \leq D$ and $\sigma$ and the clipping operation $W(w) \mapsto \bar{W}(w)$ are 1-Lipschitz continuous, the right hand side can be further bounded by 
\begin{align*}
& %\leq 
\sqrt{\int (W_2(a) - W'_2(a))^2   \dd \rho_0(a,w)}  % \sigma(\bar{W}_1^\top x) +
+ 
R \sqrt{\int(\bar{W}_1(w)^\top x - \bar{W}_1'(w)^{\top} x)^2  \dd \rho_0(a,w) } \\
& \leq 
\sqrt{\int (W_2(a) - W'_2(a))^2   \dd \rho_0(a,w)}  % \sigma(\bar{W}_1^\top x) +
+ 
R D \sqrt{\int \|\bar{W}_1(w) - \bar{W}_1'(w)\|^2  \dd \rho_0(a,w) } \\
& \leq 
\sqrt{\int (W_2(a) - W'_2(a))^2   \dd \rho_0(a,w)}  % \sigma(\bar{W}_1^\top x) +
+ 
R D \sqrt{\int \|W_1(w) - W_1'(w)\|^2  \dd \rho_0(a,w) } \\
& \leq 
(1+RD)
\sqrt{\int (W_2(a) - W'_2(a))^2   +  \|W_1(w) - W_1'(w)\|^2  \dd \rho_0(a,w) } \\
& \leq 
(1+RD)
\sqrt{\int  \|W((a,w)) - W'((a,w))\|^2  \dd \rho_0(a,w) } 
= (1+RD) \|W - W'\|_{L_2(\rho_0)}.
\end{align*}
%\begin{align*}
%& f_{W}(x) - f_{W'}(x)  \\
%& \leq 
%\int \left[ (\bar{W}_2(a) - \bar{W}'_2(a)) \sigma(W_1^\top x) +
%\bar{W}'_2(a) (\sigma(\bar{W}_1(w)^\top x) - \sigma(W'_1(w)^{\top} x)) \right] \dd \rho_0(a,w) \\
%& \leq 
%\sqrt{\int (\bar{W}_2(a) - \bar{W}'_2(a))^2   \dd \rho_0(a,w)}  % \sigma(\bar{W}_1^\top x) +
%+ 
%\sqrt{\int (\bar{W}'_2(a))^2 (\sigma(W_1(w)^\top x) - \sigma(W_1'(w)^{\top} x))^2  \dd \rho_0(a,w) } \\
%& \leq 
%\sqrt{\int (W_2(a) - W'_2(a))^2   \dd \rho_0(a,w)}  % \sigma(\bar{W}_1^\top x) +
%+ 
%R \sqrt{\int(\sigma(W_1(w)^\top x) - \sigma(W_1'(w)^{\top} x))^2  \dd \rho_0(a,w) },
%\end{align*}
%where we used $|\sigma(x)| \leq 0$, 1-Lipschitz continuity of the clipping operation 
%and $|\bar{W}_2'(a)| \leq R$.
%By noticing that $\|x\| \leq D$ and $\sigma$ is 1-Lipschitz continuous, the right hand side can be further bounded by 
%\begin{align*}
%& %\leq 
%\sqrt{\int (W_2(a) - W'_2(a))^2   \dd \rho_0(a,w)}  % \sigma(\bar{W}_1^\top x) +
%+ 
%R \sqrt{\int(W_1(w)^\top x - W_1'(w)^{\top} x)^2  \dd \rho_0(a,w) } \\
%& \leq 
%\sqrt{\int (W_2(a) - W'_2(a))^2   \dd \rho_0(a,w)}  % \sigma(\bar{W}_1^\top x) +
%+ 
%R D \sqrt{\int \|W_1(w) - W_1'(w)\|^2  \dd \rho_0(a,w) } \\
%& \leq 
%\sqrt{\int (W_2(a) - W'_2(a))^2   \dd \rho_0(a,w)}  % \sigma(\bar{W}_1^\top x) +
%+ 
%R D \sqrt{\int \|W_1(w) - W_1'(w)\|^2  \dd \rho_0(a,w) } \\
%& \leq 
%(1+RD)
%\sqrt{\int (W_2(a) - W'_2(a))^2   +  \|W_1(w) - W_1'(w)\|^2  \dd \rho_0(a,w) } \\
%& \leq 
%(1+RD)
%\sqrt{\int  \|W((a,w)) - W'((a,w))\|^2  \dd \rho_0(a,w) } 
%= (1+RD) \|W - W'\|_{L_2(\rho_0)}.
%\end{align*}
This gives the assertion.
%\end{proof}

\section{Proof of fast rate of excess risk bounds} %Theorem \ref{thm:ExcessRiskConvRate} }

\subsection{Gaussian correlation inequality}

\newcommand{\cC}{\calA}
\begin{Lemma}[Gaussian correlation inequality]
\label{eq:GaussianCorrelationInfinite}
Let  $\tilde{\cH}$ be a separable Hilbert space equipped with the complete orthonormal system $(e_i)_{i=1}^\infty$, and  
suppose that $\nu$ is a Gaussian measure in $\tilde{\cH}$ with mean 0 and covariance $\Sigma = \diag{\mu_1,\mu_2,\dots}$ with respect to CONS $(e_i)_i$ where $\sum_{i=1}^n \mu_i^2 < \infty$, 
that is, $\nu$ is the distribution corresponding to $\sum_{i=1}^\infty \xi_i \sqrt{\mu_i} e_i$ for $\xi_i \sim N(0,1)$ (i.id.).
 %such that $\xi_i \mu_i e_i$
Let $\cC^1 = \{\sum_{i=1}^\infty \alpha_i e_i \in \tilde{\cH} \mid \sum_{i=1}^\infty a_i \alpha_i^2 \leq 1, \alpha_i \in \Real\}$ for $a_i \geq 0~(i=1,2,\dots)$ and 
 $\cC^2 = \{\sum_{i=1}^\infty \alpha_i e_i \in \tilde{\cH} \mid \sum_{i=1}^\infty b_i \alpha_i^2 \leq 1, \alpha_i \in \Real\}$  for $b_i \geq 0~(i=1,2,\dots)$.
Then, we have
\begin{align*}
\nu(\cC^1 \cap \cC^2) \geq \nu(\cC^1) \nu(\cC^2).
\end{align*}

%
%Let $\nu_\infty$ be the Gaussian measure in $\cH$ corresponding to the distribution of a random variable 
%$\sum_{i=0}^\infty \xi_i f_i$ where $(\xi)_{i=0}^\infty$ is a sequence of i.i.d. standard normal variables.
%For two sets $\cC^1 = \{X = \sum_{i=0}^\infty \alpha_i f_i \in \cH \mid \sum_{i=0}^\infty \alpha_i^2 \mu_i^{(1)} \leq 1 \}$
%and $\cC^2 = \{X = \sum_{i=0}^\infty \alpha_i f_i \in \cH  \mid | \sum_{i=0}^\infty \alpha_i
% \mu_i^{(2)} |\leq 1 \}$
%where $(\mu_i^{(1)})_{i=1}^\infty$ is a fixed non-negative sequence and $(\mu_i^{(2)})_{i=1}^\infty$ is a fixed sequence of real numbers satisfying 
%$\sum_{i=0}^\infty (\mu_i^{(2)})^2 < \infty$,
%we have
%$$
%\nu_\infty(\cC^1 \cap \cC^2) \geq \nu_\infty(\cC^1) \nu_\infty(\cC^2).
%$$
\end{Lemma}

\begin{proof}
Let $\cC^1_n$ an $\cC^2_n$ be the cylinder set that ``truncates'' $\cC^1$ an $\cC^2$ up to index $n$:
$\cC^1_n = \{ \sum_{i=1}^\infty \alpha_i e_i \in \cH \mid \sum_{i=1}^n a_i \alpha_i^2  \leq 1 \}$
and $\cC^2_n = \{\sum_{i=1}^\infty \alpha_i e_i \in \cH \mid 
| \sum_{i=1}^n b_i \alpha_i^2  \leq 1 \}$.
By the Gaussian correlation inequality \cite{royen2014simple,Latala2017}, it holds that
\begin{align*}
\nu(\cC^1_n \cap \cC^2_n) \geq \nu(\cC^1_n) \nu(\cC^2_n).
\end{align*}
Note that we can apply the Gaussian correlation inequality for a finite dimensional Gaussian measure.
Next, we extend this inequality to the infinite dimensional space.
Since $(\cC^1_n)_{n}$ is a monotonically decreasing sequence, i.e., $\cC^1_n \subseteq \cC^1_m$ for $m < n$, and $\cap_{n=1}^\infty \cC^1_n = \cC^1$, 
the continuity of probability measure gives that $\lim_{n \to \infty} \nu(\cC^1_n) = \nu(\cC^1)$.
Similarly, it holds that $\lim_{n \to \infty} \nu(\cC^2_n) = \nu(\cC^2)$.

Since $\cC^2 \subset \cC^1_n$ and $\cC^2 \subset \cC^2_n$, it holds that 
$\nu(\cC^1 \cap \cC^2) \leq \nu(\cC^1_n \cap \cC^2_n)$.
On the other hand, we also have
\begin{align*}
\nu(\cC^1_n \cap \cC^2_n) 
& = \nu((\cC^1 \cup (\cC^1_n \backslash \cC^1)) \cap (\cC^2 \cup (\cC^2_n \backslash \cC^2))) \\
& \leq \nu((\cC^1 \cap \cC^2) \cup (\cC^1_n \backslash \cC^1) \cup  (\cC^2_n \backslash \cC^2)) \\
& \leq \nu(\cC^1 \cap \cC^2) + \nu(\cC^1_n \backslash \cC^1) + \nu(\cC^2_n \backslash \cC^2) \\
& \rightarrow \nu(\cC^1 \cap \cC^2).
\end{align*}
Therefore, we have that 
\begin{align*}
\lim_{n \to \infty}\nu(\cC^1_n \cap \cC^2_n)  = \nu(\cC^1 \cap \cC^2).
\end{align*}
Combining all these arguments, we finally have that 
\begin{align*}
\nu(\cC^1 \cap \cC^2) \geq \nu(\cC^1) \nu(\cC^2).
\end{align*}
\end{proof}

%{\bf Note:} By the Gaussian correlation inequality, we have that 
%\begin{align*}
%\phi_{\beta,\lambda}^{(0)}(\epsilon)&  := - \log \nutil(\{W \in \cH: W \in B_\epsilon \}) \\
%& \leq 
%- \log \nutil(\{h \in \cH: \|h\|_{\cH} \leq \epsilon\} \times 1/2) \\ 
%& = - \log \nutil(\{h \in \cH: \|h\|_{\cH} \leq \epsilon\}) + \log(2).
%\end{align*}
%Therefore, we may replace the definition of $\phi_{\beta,\lambda}^{(0)}(\epsilon)$ 
%by $ - \log \nutil(\{h \in \cH: \|h\|_{\cH} \leq \epsilon\}) + \log(2)$.
%{\bf End of Note.}

\subsection{Proof of general excess risk bound (Theorem \ref{thm:ExcessRiskConvRate})}
\label{sec:ProofOfExcessRisk}

\begin{proof}

Since $\gamma > 1/2$, $\sigma$ and $\ell(y,\cdot)$ are in $\calC^3(\Real)$ 
where $\ell$ has a bounded partial derivative on a bounded domain,  
we can easily verify that the empirical risk satisfies Assumption \ref{ass:IGLDConvCond}
by noticing the clipping operation in the model.

%We consider more general setting.
For $0 < \theta < 1$, let $\cHKtilt := \calH_{K^{\theta(\gamma+1)}}$.
It is known that if the natural inclusion $I_{\tilde{K},\tilde{K}^\theta}: \cHKtil \to \cHKtilt$ is Hilbert-Schmidt, then 
the sample path of $\nutil$ is included in $\cHKtilt$ probability 1 (Theorem 5.2 of \cite{steinwart2019convergence}).
In our case, since $\mu_k \lesssim 1/k^2$, the eigenvalues $(\mu_k(\tilde{K}))_{k=1}^\infty$ of $\tilde{K}$ satisfies $\mu_k(\tilde{K}) \lesssim 1/k^{2(\gamma+1)}$.
Theorem 5.2 of \cite{steinwart2019convergence} also states that $I_{\tilde{K},\tilde{K}^\theta}$ is Hilbert-Schmidt if and only if $\sum_{k=0}^\infty \mu_k(\tilde{K})^{1-\theta} < \infty$.
Therefore, by setting $\alphatil := 1/\{2(\gamma+1)\}$, $\theta < 1 - \alphatil$ is sufficient for this property.
From now on, we assume that $\theta < 1 - \alphatil$.
For notational simplicity, let $\lambdabeta := \lambda \beta$.

By definition, we have
\begin{align*}
\EE_{W\sim \nutil}[\|T_K^{\gamma/2} W\|_{\cHKtilt}^2 ]
= \mathrm{Tr}[ T_K^{\gamma/2} (\lambdabeta^{-1} T_K) T_K^{\gamma/2} T_K^{-\theta(\gamma+1)}]
= \lambdabeta^{-1} \mathrm{Tr}[ T_K^{(\gamma + 1) (1-\theta)}]
\end{align*}
Note that the assumption $\theta < 1-\alphatil$ ensures the right hand side is finite.
Therefore, we obtain that, for $\bar{R}_\theta > 0$, 
\begin{align*}
\nutil(\{h \in \cH \mid \|h\|_{\cHKtilt} \geq \lambdabeta^{-1/2}  \bar{R}_\theta \}) & \leq 
\frac{\EE_{W\sim \nutil}[\|T_K^{\gamma/2} W\|_{\cHKtilt}^2 ] }{\lambdabeta^{-1}\bar{R}_\theta^2}
\leq \frac{\lambdabeta^{-1}\mathrm{Tr}[ T_K^{(\gamma + 1) (1-\theta)}]}{\lambdabeta^{-1}\bar{R}_\theta^2} \\
& = \frac{\mathrm{Tr}[ T_K^{(\gamma + 1) (1-\theta)}]}{\bar{R}_\theta^2}.
\end{align*}
Hence, by setting $\bar{R}_\theta = \sqrt{2 \mathrm{Tr}[ T_K^{(\gamma + 1) (1-\theta)}]}$, 
we can guarantee that $\nutil(\{h \in \cH \mid \|h\|_{\cHKtilt} \leq \lambdabeta^{-1/2}  \bar{R}_\theta \}) \geq 1/2$.

Let $B_\epsilon = (\epsilon \calB_\calH) \cap (\lambdabeta^{-1/2} \bar{R}_\theta \calB_{\cHKtilt})$.
We define 
\begin{align*}
%& \phi_{\beta,\lambda}(\epsilon) := \inf_{h \in \cHKtil: \calL(h) - \calL(\fstar) \leq \epsilon^2} \lambdabeta \|h\|^2_{\cHKtil}
%- \log \nutil(\{h \in \cH : \|h\|_{\cH} \leq \epsilon \}) + \log(2), \\
& 
\phi_{\beta,\lambda}^{(0)}(\epsilon) := - \log \nutil(\{W \in \cH: W \in B_\epsilon \}).
\end{align*}

%!!!!!!!!!!!!!!!!!!!!!!

For any $\delta > 0$, pick up $\hstar \in \cHKtil$ that satisfies
\begin{align*}
&  \lambdabeta \|\hstar\|^2_{\cHKtil} \leq  (1+\delta) \inf_{h \in \cHKtil: \calL(h) - \calL(\fstar) \leq \epsilon^2} \lambdabeta \|h\|^2_{\cHKtil} \\
& \calL(f_\hstar) - \calL(\fstar) \leq \epsilon^2.
\end{align*}
Then, by Borel's inequality, it holds that 
\begin{align*}
-\log \nutil(\{h \in \cH :  \|h - \hstar\|_{\cH} \leq \epsilon \}) 
\leq (1+\delta) \phi_{\beta,\lambda}(\epsilon).
\end{align*}

%\Note{More verification. Smoothness L is assumed for}
By the smoothness of the expected loss function $\bar{\ell}_x$, it holds that,
for any $f,g \in \LPiPx$, 
\begin{align*}
& \calL(f) + \langle g - f, \nabla_f \calL(f) \rangle_{\LPi} + \frac{L}{2} 
\| g - f\|_{\LPi}^2 \\
= &
\EE_{X}\left[\bar{\ell}_X(f(X)) + (g(X) - f(X))\bar{\ell}'_X(f(X)) +  
\frac{L}{2} 
( g(X) - f(X))^2 \right] \\
\geq & 
\EE_{X}\left[\bar{\ell}_X(g(X))  \right] =\calL(g),
\end{align*}
where $\nabla_f \calL$ is the Fr{\'e}chet  derivative in $\LPiPx$ (note that this inequality holds even though $\bar{\ell}_x(\cdot)$ is not a convex function).
By substituting $g = \fstar$ and $f = \fstar + \frac{1}{L} \nabla_f\calL(f_\hstar)$, we obtain 
\begin{align*}
& \calL(\fstar) + \frac{1}{2L} \|\nabla_f \calL(f_{\hstar})\|_{\LPi}^2 \leq \calL(f_\hstar) \\
\Rightarrow~~ &
\|\nabla_f \calL(f_{\hstar})\|_{\LPi}^2 \leq 2L(\calL(f_\hstar) - \calL(\fstar)) \leq 2L \epsilon^2.
\end{align*}
%
%
%\begin{align*}
%& \calL(f) + \langle g - f, \nabla \calL(f) \rangle_{\LPi} + \frac{1}{2L} 
%\| \nabla \calL(g) -  \nabla  \calL(f)\|_{\LPi}^2 \\
%= &
%\EE_{X}\left[\bar{\ell}_X(f(X)) + (g(X) - f(X))\bar{\ell}'_X(f(X)) +  
%\frac{1}{2L} 
%( \bar{\ell}'_X(g(X)) - \bar{\ell}'_X(f(X)) )^2 \right] \\
%\leq & 
%\EE_{X}\left[\bar{\ell}_X(g(X))  \right] =\calL(g).
%\end{align*}
%
%Since the loss function is convex with smoothness parameter $L$, it holds that 
%\begin{align*}
%\calL(\fstar) + \langle f_\hstar - \fstar, \nabla \calL(\fstar) \rangle_{\LPi} + \frac{1}{2L} \| \nabla \calL(f_\hstar) -  \nabla \calL(\fstar)\|_{\LPi}^2 \leq \calL(f_\hstar),
%\end{align*}
%which implies
%\begin{align*}
%\| \nabla \calL(f_\hstar) \|_{\LPi}^2 \leq 2 L \calL(\hstar) \leq 2L\epsilon^2.
%\end{align*}
Therefore, for any $h \in \cH$ such that $\|h - \hstar\|_{\cH} \leq \epsilon$, it holds that 
\begin{align*}
\calL(h) 
& \leq \calL(f_\hstar) + \langle f_h - f_\hstar, \nabla \calL(f_\hstar)\rangle_{\LPi} + \frac{L}{2} \|f_h - f_\hstar\|_{\LPi}^2 \\
& \leq \calL(f_\hstar) + \frac{1}{2}\| f_h - f_\hstar \|_{\LPi}^2 + \frac{1}{2}\| \nabla \calL(f_\hstar)\|^2_{\LPi} 
+ \frac{L}{2} \|f_h - f_\hstar\|_{\infty}^2 \\
& \leq \calL(f_\hstar) + \frac{1}{2}\| f_h - f_\hstar \|_{\infty}^2 + \frac{1}{2} 2L\epsilon^2
+ \frac{L}{2} \|f_h - f_\hstar\|_{\infty}^2 \\
& \leq \epsilon^2 + 
\frac{1+L}{2} (1+RD)^2 \|h - \hstar\|_{\cH}^2 + L \epsilon^2  \\
& \leq \left(1 + \frac{(1+L)(1+RD)^2}{2} + L \right) \epsilon^2 =: \CLRD \epsilon^2.
\end{align*}
This yields that 
\begin{align*}
& -\log \nutil(\{h \in \cH :  \calL(h) - \calL(\fstar) \leq \CLRD \epsilon^2\}) \\
& \leq 
-\log \nutil(\{h \in \cH :  \calL(h)   - \calL(\fstar) \leq \CLRD \epsilon^2, 
\|h - \hstar\|_{\cH} \leq \epsilon \})  \\
& \leq (1+\delta) \phi_{\beta,\lambda}(\epsilon).
\end{align*}
Since $\delta$ is arbitrary, we obtain that 
\begin{align}\label{eq:SmallBallPhiLowerBound}
& -\log \nutil(\{h \in \cH :  \calL(h)  - \calL(\fstar)  \leq \CLRD \epsilon^2\}) 
\leq \phi_{\beta,\lambda}(\epsilon).
\end{align}

%!!!!!!!!!!!!!!!!!!!!!!

By the Gaussian correlation inequality (Lemma \ref{eq:GaussianCorrelationInfinite}), we have that 
\begin{align*}
\phi_{\beta,\lambda}^{(0)}(\epsilon)&  = - \log \nutil(\{W \in \cH \mid W \in B_\epsilon \}) \\
& \leq 
- \log[ \nutil(\{h \in \cH \mid \|h\|_{\cH} \leq \epsilon\}) \times \nutil(\{h \in \cH \mid \|h\|_{\cHKtilt} \leq \lambdabeta^{-1/2}  \bar{R}_\theta \})] \\ 
& \leq 
- \log[ \nutil(\{h \in \cH \mid \|h\|_{\cH} \leq \epsilon\}) \times 1/2] \\ 
& = - \log \nutil(\{h \in \cH \mid \|h\|_{\cH} \leq \epsilon\}) + \log(2).
\end{align*}
Therefore, we can see that 
\begin{align*}
\inf_{h \in \cHKtil: \calL(h) - \calL(\fstar) \leq \epsilon^2/2} \lambdabeta \|h\|^2_{\cHKtil} + \phi_{\beta,\lambda}^{(0)}(\epsilon)
\leq 
 \phi_{\beta,\lambda}(\epsilon).
\end{align*}
%may replace the definition of $\phi_{\beta,\lambda}^{(0)}(\epsilon)$ 
%by $ - \log \nutil(\{h \in \cH: \|h\|_{\cH} \leq \epsilon\}) + \log(2)$.
%{\bf End of Note.}
Here we define $\epsilonstar$ as
\begin{align*}
\epsilonstar := \max\{\inf\{\epsilon > 0 :  \phi_{\beta,\lambda}(\epsilon) \leq \beta \epsilon^2\},n^{-\frac{1}{2(2-s)}}\}.
\end{align*}
Note that, since $\phi_{\beta,\lambda}$ is monotonically non-increasing, $\epsilonstar$ satisfies
\begin{align}\label{eq:epsilonstarPhiBetaLamCond}
\phi_{\beta,\lambda}(\epsilonstar) \leq \beta \epsilonstar^2.
\end{align}

Let $r > 1$, and $M_r := - 2 \Phi^{-1}(e^{- \beta \epsilonstar^2 r})$ where $\Phi$ is the cumulative distribution function of the standard normal distribution.
Since $\Phi^{-1}(y) \geq - \sqrt{5/2\log(1/y)}$ for every $y \in (0,1/2)$
and \Eqref{eq:epsilonstarPhiBetaLamCond} implies $\log(2) \leq \beta \epsilonstar^2$ yielding $e^{-\beta \epsilonstar^2 r}< e^{-\beta \epsilonstar^2} \leq e^{-\log(2)} = 1/2$, $M_r$ is bounded by
\begin{align}\label{eq:MrepsBound}
- M_r/2 \geq - \sqrt{\frac{5}{2}\log(e^{\beta \epsilonstar^2 r})} = -  \sqrt{\frac{5}{2} \beta \epsilonstar^2 r} 
\Rightarrow 
M_r \leq \sqrt{10 \beta \epsilonstar^2 r}.
\end{align}
Let 
\begin{align*}
\calF_r:= B_\epsilonstar + M_r \lambdabeta^{-1/2} \calB_{\cHKtil}.
\end{align*}
%Since $B_\epsilon \supset  (\epsilon \calB_\calH) \cap (\epsilon  \lambdabeta \calB_{\cHKtilt}) \subset  $
By Borell's inequality (Theorem 3.1 of  \cite{borell1975brunn}), 
the prior probability mass of $\calF_r$ is lower bounded by 
\begin{align*}
\nutil(\calF_r) \geq \Phi(M_r + \alpha_r)
\end{align*}
where $\alpha_r \in \Real$ is determined by 
\begin{align*}
\alpha_r = \Phi^{-1}(\nutil(B_\epsilonstar)) = \Phi^{-1}(e^{-\phi_{\beta,\lambda}^{(0)}(\epsilonstar)}).
\end{align*}
Since $\phi_{\beta,\lambda}^{(0)}(\epsilonstar) \leq \phi_{\beta,\lambda}(\epsilonstar) + \log(2)
\leq \beta \epsilonstar^2 \leq \beta \epsilonstar^2 r$ (\Eqref{eq:epsilonstarPhiBetaLamCond}), 
we have $\Phi(\alpha_r) \geq e^{-\beta \epsilonstar^2r }$ by the definition of $\alpha_r$, which implies 
\begin{align*}
\alpha_r \geq \Phi^{-1}(e^{-\beta \epsilonstar^2 r}) = - \frac{1}{2} M_r,
\end{align*}
where the last equality is given by the definition of $M_r$.
Therefore, 
\begin{align}\label{eq:nuFrlowerbound}
\nutil(\calF_r) \geq \Phi(M_r + \alpha_r) \geq \Phi(M_r /2) \geq 1 - \exp(- \beta \epsilonstar^2 r).
\end{align}

By the proof of Theorem 2.1 in \cite{AS:Vaart&Zanten:2008}, we obtain that the metric entropy of $\calF_r$ is bounded by
\footnote{For a metric space $\tilde{\calF}$ equipped with a metric $\tilde{d}$, the $\epsilon$-covering number $\calN(\tilde{\calF},\tilde{d},\epsilon)$
is defined as the minimum number of balls with radius $\epsilon$ (measured by the metric $\tilde{d}$) to cover the metric space $\tilde{\calF}$. 
}
\begin{align*}
\log \calN(\calF_r,\|\cdot\|_{\cH},\epsilonstar) 
& \leq \frac{1}{2} M_r^2 + \phi_{\beta,\lambda}^{(0)}(\epsilonstar),
% \\
%& \leq 5 C \beta \epsilon^2 r + \phi_{\beta,\lambda}^{(0)}(\epsilon).
\end{align*}
and, more strongly,
there exist $h_1,\dots,h_{N_\epsilonstar} \in \calF_r$ for $N_{\epsilonstar} := \calN(\calF_r, \epsilonstar, \|\cdot\|_{\cH})$
such that 
%for all $W \in \calF_r$, $\min_{i} \|h_i - W\|_{\cH} \leq \epsilon$.
%Then, by noticing that 
\begin{align*}
\calF_r \subset B_\epsilonstar + \{h_1,\dots,h_{N_\epsilonstar}\}.
\end{align*}

This indicates that, even for smaller $\epsilon' \leq \epsilonstar$, it holds that 
\begin{align*}
\log \calN(\calF_r, \|\cdot\|_{\cH},\epsilon') 
\leq N_\epsilonstar + \log \calN(B_\epsilonstar,\|\cdot\|_{\cH},\epsilon').
\end{align*}
Then, if we let $\tilde{\phi}(\epsilon') = \log \calN((\lambdabeta^{-1/2}\bar{R}_\theta) \calB_{\cHKtilt}, \|\cdot\|_{\cH},\epsilon')$, then 
for $B_\epsilonstar \subset (\lambdabeta^{-1/2}\bar{R}_\theta) \calB_{\cHKtilt}$ by its definition,
\Eqref{eq:epsilonstarPhiBetaLamCond} and \Eqref{eq:MrepsBound} give
\begin{align}\label{FrCoveringNumberBound}
\log \calN(\calF_r, \|\cdot\|_{\cH},\epsilon')  \leq N_\epsilonstar + \tilde{\phi}(\epsilon') \leq 
\frac{1}{2} M_r^2 + \phi_{\beta,\lambda}^{(0)}(\epsilonstar)+ \tilde{\phi}(\epsilon')
\leq (5r + 1) \beta \epsilonstar^2 + \tilde{\phi}(\epsilon').
\end{align}
Note that if $\epsilon' > \epsilonstar$, then by using the fact $N_{\epsilon'} \leq N_\epsilonstar$, we can see that this inequality still holds: 
$\log \calN(\calF_r, \epsilon', \|\cdot\|_{\cH})  = N_{\epsilon'} \leq N_\epsilonstar \leq N_\epsilonstar + \tilde{\phi}(\epsilon')$.
The covering number of $\calB_{\cHKtilt}$ can be evaluated using the decay rate of the spectrum $(\mu_k(\tilde{K}^{\theta}))_{k}$ \cite{COLT:Steinwart+etal:2009}.
Indeed,  %it has been shown that 
 $\mu_k(\tilde{K}^{\theta}) \lesssim k^{- \frac{\theta}{\alphatil}}$ implies  
$\tilde{\phi}(\epsilon') \lesssim (\epsilon'/\lambdabeta^{1/2})^{-2\alphatil /\theta}$ \cite[Theorem 15]{COLT:Steinwart+etal:2009}.
Moreover, the small ball probability $\phi_{\beta,\lambda}^{(0)}(\epsilon)$ can be evaluated using the covering number.
First, notice that $\phi_{\beta,\lambda}^{(0)}(\epsilon)  = - \log \nutil(\{W \in \cH \mid W \in B_\epsilon \}) \leq  
- \log \nutil(\{W \in \cH \mid  \|W \|_{\cH} \leq \epsilon \}) = :\varphi(\epsilon)$,
and then \cite{ghosal_van_der_vaart:Book:2017} showed that 
\begin{align*}
\varphi(2 \epsilon) \lesssim \log \calN\left(\lambdabeta^{-1/2} \calB_{\cHKtil}, \|\cdot\|_{\cH},\frac{\epsilon}{\sqrt{2\varphi(\epsilon)}}\right)
\lesssim
\varphi(\epsilon).
\end{align*}
Here, since the entropy number in the middle is evaluated as $\log\calN\left(\lambdabeta^{-1/2} \calB_{\cHKtil}, \|\cdot\|_{\cH},\epsilon\right) \lesssim (\epsilon/\lambdabeta^{1/2})^{-2\alphatil}$, we obtain 
\begin{align}\label{eq:SmallBallConcentProb}
\phi_{\beta,\lambda}^{(0)}(\epsilon) \leq \varphi(\epsilon) \lesssim \left( \frac{\epsilon}{\lambdabeta^{1/2}} \right)^{- \frac{2\alphatil}{1-\alphatil}}.
\end{align}

%\phi_{\beta,\lambda}^{(0)}(\epsilon)&  = - \log \nutil(\{W \in \cH \mid W \in B_\epsilon \}) 

In this setting, we will show that, for any $r > 1$, there exists an event $\calE_r$ with respect to data generation $D_n$ and 
exists $\ustar > 0$ such that
\begin{align*}
\mathrm{A}:~& P (\calE_r^c) \leq 
 2 e^{- c' \min\{\beta \epsilonstar^2, n \epsilonstar^{2(2-s)}\}r}~~\text{for a constant $c' > 0$}, \\
%e^{-r \beta \epsilonstar^2}, \\
\mathrm{B}:~ & \calLhat(W) - \calLhat(\fstar) \geq 
\frac{1}{2} [ (\calL(h) - \calL(\fstar)) -  \ustar r]~~~~(\forall h \in \calF_r) \\
%\left(C r - \frac{1}{2}\right)  \epsilonstar^2 ~~~~(\forall W \in \calF_r~\text{s.t.}~\calL(W) - \calL(\fstar) \geq C r  \epsilonstar^2) \\
&~~~~~~~~~\text{under the event $\calE_r$}, \\
\mathrm{C}:~ & \EE[\pi_\infty(\calF_r^c) \boldone_{\calE}] \leq 
2 \exp\left[ - \tfrac{1}{2}(r - 2)\beta\epsilonstar^2\right], \\
%e^{-r \beta \epsilonstar^2}, \\
\mathrm{D}:~ & \EE[\pi_\infty(\{W \in \calF_r : \calL(W) - \calL(\fstar) \geq 3 r \ustar\}) \boldone_{\calE_r}] \leq 
\exp\left[ - \tfrac{1}{2} (r- 2) \beta \epsilonstar^2 \right].
%e^{-r \beta \epsilonstar^2}.
\end{align*}

%Let $N_\epsilon := \calN(\calF_r, \epsilon, \|\cdot\|_{\cH})$, 
%then there are $h_1,\dots,h_{N_\epsilon} \in \calF_r$ such that 
%for all $W \in \calF_r$, $\min_{i} \|h_i - W\|_{\cH} \leq \epsilon$.
%It should be noted that $\| f_{T_K^{\gamma/2} h_i} - f_{T_K^{\gamma/2} W} \|_{\infty} \leq C \epsilon$ for $i \in \{1,\dots,N_{\epsilon}\}$ achieving the minimum.
%Let $g_h := f_{T_K^{\gamma/2} h}$.

From now on, we will define $\ustar$ and $\calE_r$ and prove the conditions $\mathrm{A,B,C,D}$ one by one.

{\it \bf Step 1: Definitions of $\ustar$ and $\calE_r$, and proof of $\mathrm{A}$ and $\mathrm{B}$.}

For notational simplicity, we write $\ell(f,Z)$ to indicate $\ell(Y,f(X))$ for $Z=(X,Y)$.
By Talangrand's concentration inequality \cite{Talagrand2,BousquetBenett}, we have 
\begin{align*}
& P\Bigg( \sup_{h \in \calF_r} \frac{|\calL(h) - \calL(\fstar) - (\calLhat(h) - \calLhat(\fstar))|}{\calL(h) - \calL(\fstar) + u} 
\notag \\
& ~~~~~~~~~~~~~\geq 
2 \EE\left[\sup_{h \in \calF_r} \frac{|\calL(h) - \calL(\fstar) - (\calLhat(h) - \calLhat(\fstar))|}{\calL(h) - \calL(\fstar) + u}\right]
+ \sqrt{\frac{2 t}{n}V } + \frac{2 t U}{n} \Bigg) \\
%\sqrt{\frac{2 \Var[ \ell(g_{h_i},Z) - \ell(\fstar,Z)]  }{n} t} + \frac{M}{n}t \right) \leq 
& \leq \exp(-t),
\end{align*}
for any $t \geq 1$, where 
\begin{align*}
V & = \sup_{h \in \calF_r} \frac{\EE[(\ell(h,Z) - \ell(\fstar,Z) - \EE[\ell(h,Z) - \ell(\fstar,Z)])^2]}{(\calL(h) - \calL(\fstar) + u)^2}, \\
U & = \sup_{h \in \calF_r} \frac{\|\ell(h,\cdot) - \ell(\fstar,\cdot) - \EE[\ell(h,Z) - \ell(\fstar,Z)]\|_{\infty} }{\calL(h) - \calL(\fstar) + u}.
\end{align*}
By the Bernstein condition, it holds that 
\begin{align*}
\Var[ \ell(f_h,Z) - \ell(\fstar,Z)]  \leq \EE[(\ell(f_h,Z) - \ell(\fstar,Z))^2] \leq C_B(\calL(h) - \calL(\fstar))^{\betas},
\end{align*}
which gives 
\begin{align*}
V \leq C_B \sup_{h \in \calF_r} \frac{(\calL(h) - \calL(\fstar))^s}{(\calL(h) - \calL(\fstar) + u)^2}
\leq \frac{C_B}{u^{2-s}}.
\end{align*}
By the boundedness assumption of the loss function, we can see that 
\begin{align*}
U \leq \frac{\Rbar}{u}.
\end{align*}
Hence, we have that 
\begin{align}
& P\Bigg( \sup_{h \in \calF_r} \frac{|\calL(h) - \calL(\fstar) - (\calLhat(h) - \calLhat(\fstar))|}{\calL(h) - \calL(\fstar) + u} 
  \notag \\
&~~~~~~~~~~~~~\geq 2\EE\left[\sup_{h \in \calF_r} \frac{|\calL(h) - \calL(\fstar) - (\calLhat(h) - \calLhat(\fstar))|}{\calL(h) - \calL(\fstar) + u}\right]
+ \sqrt{\frac{2 C_B }{nu^{2-s}} t} + \frac{2 \Rbar}{n u}t \Bigg) \notag \\
%\sqrt{\frac{2 \Var[ \ell(g_{h_i},Z) - \ell(\fstar,Z)]  }{n} t} + \frac{M}{n}t \right) \leq 
& \leq \exp(-t),
\label{eq:TalagrandConc}
\end{align}
for any $t \geq 1$.

Hereafter, we bound the expectation of the supremum of the ratio type empirical process: $\EE\left[\sup_{h \in \calF_r} \frac{|\calL(h) - \calL(\fstar) - (\calLhat(h) - \calLhat(\fstar))|}{\calL(h) - \calL(\fstar) + u}\right]$.
Let the empirical $L_2$-norm be $\|h\|_n := \sqrt{\frac{1}{n}\sum_{i=1}^n h(z_i)^2}$. By the usual Rademacher complexity and covering number argument (Lemma 11.4 of \cite{boucheron2013concentration}, Theorem 5.22 of \cite{wainwright2019high} and Lemma A.5 of \cite{bartlett2017spectrally:arXiv} for example), the non-ratio-type empirical process can be bounded as 
\begin{align*}
% \psi(u)  
&%= 
\EE\left[\sup_{h \in \calF_r: \calL(h) - \calL(\fstar) \leq u} |\calL(h) - \calL(\fstar) - (\calLhat(h) - \calLhat(\fstar))|\right]
\\
& \leq 
2 \EE\left[\sup_{h \in \calF_r: \calL(h) - \calL(\fstar) \leq u} \Big| \frac{1}{n}\sum_{i=1}^n \epsilon_i ( \ell(f_h,z_i) - \ell(\fstar,z_i) - 
(\calL(h) - \calL(\fstar)))\Big|
\right] \\
& \leq 
C\EE\left[ \inf_{a > 0}\left\{ a + \int_a^{\hat{r}(u)} \sqrt{\frac{\log \calN(\{\ell(f_h,\cdot) - \ell(\fstar,\cdot) \mid h \in \calF_r,~\calL(h) - \calL(\fstar) \leq u \},\|\cdot\|_n,\epsilon')}{n}} \dd \epsilon' \right\}\right],
\end{align*}
where $\hat{r}(u) := \sup\{ \|\ell(f_h,\cdot) - \ell(\fstar,\cdot)\|_n \mid h \in \calF_r,~\calL(h) - \calL(\fstar) \leq u \}$
and $C$ is a universal constant.
The Dudley integral in the right hand side can be bounded by 
\begin{align*}
& \int_a^{\hat{r}(u)} \sqrt{\frac{\log \calN(\{\ell(f_h,\cdot) - \ell(\fstar,\cdot) \mid h \in \calF_r,~\calL(h) - \calL(\fstar) \leq u \},\|\cdot\|_n,\epsilon')}{n}} \dd \epsilon' \\
\stackrel{(1)}{\leq}
& \int_a^{\hat{r}(u)} \sqrt{\frac{\log \calN(\{ f_h \mid h \in \calF_r \},\|\cdot\|_\infty,\epsilon'/B)}{n}} \dd \epsilon' \\
%\stackrel{(1)}{\lesssim}
\stackrel{(2)}{\leq}
& 
\int_a^{\hat{r}(u)} \sqrt{\frac{\log \calN(\calF_r ,\|\cdot\|_{\cH},\epsilon'/(B(1+RD)))}{n}} \dd \epsilon'  \\
%\leq 
%\int_a^{\hat{r}(u)} \sqrt{\frac{\log \calN(\calF_r ,\|\cdot\|_{\cH},\epsilon')}{n}} \dd \epsilon' 
\stackrel{(3)}{\leq}
&
\int_a^{\hat{r}(u)} \sqrt{\frac{N_\epsilonstar + \tilde{\phi}(\epsilon'/(B(1+RD)))}{n}} \dd \epsilon', 
\end{align*}
where we used the bounded gradient condition on the loss function to show (1), used Lemma \ref{lem:LipshitzInfty} to show (2),  and used  \Eqref{FrCoveringNumberBound} to show (3).
If we let $a = 1/n$, then we have 
\begin{align*}
& \EE\left[\sup_{h \in \calF_r: \calL(h) - \calL(\fstar) \leq u} |\calL(h) - \calL(\fstar) - (\calLhat(h) - \calLhat(\fstar))|\right] 
\\ 
\leq 
& 
C\EE\left[\frac{1}{n} + \int_{1/n}^{\hat{r}(u)} \sqrt{\frac{N_\epsilonstar + \tilde{\phi}(\epsilon'/(B(1+RD)))}{n}} \dd \epsilon' \right]
=: \psi_{r,\epsilonstar}(u).
\end{align*}
Here, we assume that there exists an upper bound $\bar{\psi}_{r,\epsilonstar}(u)$ of $\psi_{r,\epsilonstar}(u)$ that satisfies
\begin{subequations}
\label{eq:barpsiConditions}
\begin{align}
& \bar{\psi}_{r,\epsilonstar}(4 u) \leq 2 \bar{\psi}_{r,\epsilonstar}(u)~~~(u > 0), \\
& \frac{\bar{\psi}_{r,\epsilonstar}(u r)}{u r} \leq \frac{\bar{\psi}_{1,\epsilonstar}(u)}{u}~~~(u >0,~r \geq 1).
\end{align} 
\end{subequations}
We will show these conditions in Step 5.
Then, the so called {\it peeling device} gives 
\begin{align*}
\EE\left[\sup_{h \in \calF_r} \frac{|\calL(h) - \calL(\fstar) - (\calLhat(h) - \calLhat(\fstar))|}{\calL(h) - \calL(\fstar) + u}\right]
\leq
\frac{4 \bar{\psi}_{r,\epsilonstar}(u)}{u}.
\end{align*}
(Theorem 7.7 and Eq.~(7.17) of \cite{Book:Steinwart:2008}).
Therefore, \Eqref{eq:TalagrandConc} can yields that  
\begin{align*}
& P\left( \sup_{h \in \calF_r} \frac{|\calL(h) - \calL(\fstar) - (\calLhat(h) - \calLhat(\fstar))|}{\calL(h) - \calL(\fstar) + u} 
\geq 
8\frac{\bar{\psi}_{r,\epsilonstar}(u)}{u}
+ \sqrt{\frac{2C_B }{nu^{2-s}} t} + \frac{2 \Rbar}{n u}t \right) 
\leq \exp(-t).
\end{align*}
Here, for $t_1 = \beta \epsilonstar^2$, let $\ustar$ be any real number satisfying
%Therefore, by setting $\ustar = \ustar(t) > 0$ as 
\begin{align*}
\ustar  \geq 
\max\left\{
\epsilonstar^2
,
\inf\left\{ u > 0 : 
8\frac{\bar{\psi}_{1,\epsilonstar}(u)}{u}
+ \sqrt{\frac{2 C_B }{n u^{2-s}} t_1} + \frac{2 \Rbar}{n u}t_1 \leq \frac{1}{2}
\right\}
\right\}.
\end{align*}
For more general $t = t_1 r =\beta \epsilonstar^2 r $, 
since $\frac{\bar{\psi}_{r,\epsilonstar}(\ustar r)}{\ustar r} \leq \frac{\bar{\psi}_{1,\epsilonstar}(\ustar)}{\ustar}$,
combining with the fact that $\sqrt{\frac{2 C_B }{n (\ustar r)^{2-s}} t_1r} \leq \sqrt{\frac{2 C_B }{n (\ustar )^{2-s}} t_1}$
and $\frac{2 \Rbar}{n \ustar r}t_1r = \frac{2 \Rbar}{n \ustar }t_1$, 
it also holds that 
$$
8 \frac{\bar{\psi}_{r,\epsilonstar}(\ustar r)}{\ustar r}
+ \sqrt{\frac{2C_B }{n(\ustar r)^{2-s}} t_1 r} + \frac{2 \Rbar}{n (\ustar r)}t_1 r \leq \frac{1}{2},
$$
for $r \geq 1$.
%We will show that we can obtain an upper bound $\bar{\psi}_{r,\epsilonstar}$ of $\psi_{r,\epsilonstar}$ such that 
%$\frac{\bar{\psi}_{r,\epsilonstar}(\ustar r)}{\ustar r} \leq \frac{\bar{\psi}_{1,\epsilonstar}(\ustar)}{\ustar}$ later on (Step 5).
%By replacing $\bar{\psi}_{r,\epsilonstar}$ with $\psi_{r,\epsilonstar}$ in the definition of $\ustar$,
Therefore, the following inequality holds:
%an appropriate $\ustar$, for which it holds that  
%if we let $\ustar_r := \ustar r$, then 
%Therefore,
%we obtain that 
\begin{align*}
& \calLhat(h) - \calLhat(\fstar) \geq \frac{1}{2} [ (\calL(h) - \calL(\fstar)) -  \ustar r], %\\
%& \calLhat(h) - \calLhat(\fstar) \leq \frac{1}{2} [ 3 (\calL(h) - \calL(\fstar)) +  \ustar], 
\end{align*}
uniformly over all $h \in \calF_r$ with probability $1 - e^{-\beta \epsilonstar^2 r }$.
We denote this event by $\calE_1$.

\Eqref{eq:SmallBallPhiLowerBound} gives that 
\begin{align*}
 -\log \nutil(\{h \in \cH :  \calL(h) - \calL(\fstar) \leq \CLRD \epsilonstar^2\}) 
\leq \phi_{\beta,\lambda}(\epsilonstar).
\end{align*}
Let the conditional probability measure of $\nutil$ conditioned on the set $A_\epsilonstar := \{h \in \cH :  \calL(h) - \calL(\fstar) \leq \CLRD \epsilonstar^2\}$ be 
\begin{align*}
\nutil(B|A_\epsilonstar) := \frac{\nutil(B \cap A_\epsilonstar) }{\nutil(A_\epsilonstar)},
\end{align*}
for a measurable set $B \subset \cH$.
Let $\bar{\ell}(Z) := \int \ell(h,Z) \nutil(\dd h | A_\epsilonstar)$. 
Then, we have that 
\begin{align*}
 \int \exp\left( - \beta (\calLhat(h) - \calLhat(\fstar)) \right) \dd \nutil(h | A_\epsilonstar)  &\geq \exp\left( - \int \beta (\calLhat(h) - \calLhat(\fstar)) \dd \nutil(h | A_\epsilonstar) \right) \\
&
 = \exp\left[ -  \beta  \left(\frac{1}{n} \sum_{i=1}^n \bar{\ell}(z_i) - \EE[\bar{\ell}(Z)] \right) \right].
% \\
%& = \exp\left[ -  \beta  \left(\frac{1}{n} \sum_{i=1}^n \bar{\ell}(z_i) - \EE[\bar{\ell}(Z)] \right) \right] \\
\end{align*}
Now, by the Bernstein's inequality,
\begin{align*}
P\left( \frac{1}{n} \sum_{i=1}^n \bar{\ell}(z_i) - \EE[\bar{\ell}(Z)] \geq  \sqrt{\frac{2C_B (\CLRD \epsilonstar^2)^{s} t}{n}} + \frac{\Rbar t}{n}\right) \leq e^{-t},
\end{align*}
for $t \geq 0$. 
Here, let $t = \frac{1}{8}\min\{\frac{1}{2C_B \CLRD^s },\frac{1}{\Rbar}\} n \min\{ \epsilonstar^{2(2-s)}, \epsilonstar^2\} r$, then it holds that
\begin{align*}
P\left( \frac{1}{n} \sum_{i=1}^n \bar{\ell}(z_i) - \EE[\bar{\ell}(Z)] \geq  \frac{1}{2}\epsilonstar^2 r \right)
 \leq e^{-\frac{1}{8}\min\{\frac{1}{2C_B \CLRD^s },\frac{1}{\Rbar}\} n \min\{\epsilonstar^{2(2-s)},\epsilonstar^2\} r}.
\end{align*}
%under a condition of $\epsilonstar \leq 1$.
%
%By the definition of $\ustar$, this also gives 
%\begin{align*}
%P\left( \frac{1}{n} \sum_{i=1}^n \bar{\ell}(z_i) - \EE[\bar{\ell}(Z)] \geq  \frac{\CLRD^{s/2}}{2} \ustar(t) \right) \leq e^{-t},
%\end{align*}
Therefore, this and the definition of $\nutil(\cdot |A_\epsilonstar)$ give that 
\begin{align}
 \int \exp\left( - \beta (\calLhat(h) - \calLhat(\fstar)) \right) \dd \nutil(h ) 
&
\geq  \exp\left( - \tfrac{1}{2}\beta \epsilonstar^2 r %\tfrac{\CLRD^{s/2}}{2} \beta \ustar(t)
\right) \nutil(A_\epsilonstar)  \notag \\
%& 
%\geq  \exp\left( - \tfrac{\CLRD^{s/2}}{2} \beta \ustar(t) - \phi_{\beta,\lambda}(\epsilonstar)\right)  \\
&
\geq  \exp\left( -\tfrac{1}{2}\beta\epsilonstar^2 r  % \tfrac{\CLRD^{s/2}}{2} \beta \ustar(t) 
- \beta \epsilonstar^2\right)   ~~~(\text{$\because$ Eqs. \eqref{eq:SmallBallPhiLowerBound} and  \eqref{eq:epsilonstarPhiBetaLamCond}}) \notag \\
&
\geq  
\exp\left( -(\tfrac{r}{2} + 1)\beta\epsilonstar^2   % \tfrac{\CLRD^{s/2}}{2} \beta \ustar(t) 
%- \beta \epsilonstar^2
\right).
\label{eq:DenominatorLowerBound}
%  \\
%\exp\left( - \tfrac{\CLRD^{s/2}}{2} \beta \ustar(t) - \beta \ustar \right)  \\
%&
%=  \exp\left( - \tfrac{\CLRD^{s/2} + 2}{2} \beta \ustar(t)\right).
\end{align}
with probability $1 - \exp[-\frac{1}{8}\min\{\frac{1}{2C_B \CLRD^s },\frac{1}{\Rbar}\} n 
\min\{\epsilonstar^{2(2-s)},\epsilonstar^2 \}r]$. %e^{-t}$.
We define this event as $\calE_2$.

Combining $\calE_1$ and $\calE_2$, we define $\calE_r = \calE_1 \cap \calE_2$, then $P(\calE_r) \geq 1 - 
e^{-\beta \epsilonstar^2 r } - e^{-\frac{1}{8}\min\{\frac{1}{2C_B \CLRD^s },\frac{1}{\Rbar}\} n \min\{\epsilonstar^{2(2-s)}, \epsilonstar^2\} r}
\geq 1- 2 e^{- c' \min\{\beta \epsilonstar^2, n \epsilonstar^{2(2-s)}\}r}$ for a constant $c' > 0$,
where we used $\beta \leq n$.
% e^{-t}.$

%In the event $\calE_r$, it holds that
%\begin{align*}
% \int \exp(-\beta (\calLhat(h) - \calLhat(\fstar)))\nutil(\dd h) 
%& \geq 
%\int_{h: \calL(h) - \calL(\fstar) \leq \ustar}\exp(-\beta (\calLhat(h) - \calLhat(\fstar))) \nutil(\dd h) \\
%& \geq \exp(- 2 \beta \ustar ) \nutil(\{ h \in \cH \mid h \in h^* + B_{\sqrt{u^*/2\gamma}} \}) \\
%& \geq \exp(- 2 \beta \ustar - \phi_{\beta,\lambda}(\sqrt{\ustar})).
%\end{align*}

{\it \bf Step 2: Proof of $\mathrm{C}$.}

Next, we evaluate the condition $\mathrm{C}$.
\Eqref{eq:DenominatorLowerBound} gives that, on the event $\calE_r$, the Bayes posterior probability of $\calF_r^c$ is upper bounded by   
\begin{align*}
\pi_\infty(\calF_r^c) & = 
\frac{ \int_{h \in \calF_r^c} \exp\left( - \beta (\calLhat(h) - \calLhat(\fstar)) \right) \dd \nutil(h ) 
}{
 \int \exp\left( - \beta (\calLhat(h) - \calLhat(\fstar)) \right) \dd \nutil(h ) 
} \\
& \leq 
\exp\left( (\tfrac{r}{2} + 1)\beta\epsilonstar^2\right) \int_{h \in \calF_r^c} \exp\left( - \beta (\calLhat(h) - \calLhat(\fstar)) \right) \dd \nutil(h).
\end{align*}
Therefore, it holds that 
\begin{align}
\EE[\pi_\infty(\calF_r^c) \boldone_{\calE_r}] & 
\leq 
\EE\left[\boldone_{\calE_r} %\exp(2 \beta \epsilonstar^2 + \phi_{\beta,\lambda}( \sqrt{\ustar})) 
\exp\left( (\tfrac{r}{2} + 1)\beta\epsilonstar^2\right)
\int_{h \in \calF_r^c} \exp(-\beta(\calLhat(h) - \calLhat(\fstar))) \dd \nutil(h) \right] \notag \\
& 
\leq 
%\exp(2 \beta \ustar + \phi_{\beta,\lambda}( \sqrt{\ustar}))  
\exp\left( (\tfrac{r}{2} + 1)\beta\epsilonstar^2\right)
\int_{h \in \calF_r^c}
\EE\left[ \exp\left(-n \frac{\beta}{n}(\calLhat(h) - \calLhat(\fstar))\right)\right] \dd \nutil(h)  \notag\\
& = 
%\exp(2 \beta \ustar + \phi_{\beta,\lambda}( \sqrt{\ustar}))  
\exp\left( (\tfrac{r}{2} + 1)\beta\epsilonstar^2\right)
\int_{h \in \calF_r^c}
\prod_{i=1}^n \EE_{Z_i}\left[ \exp\left(-\frac{\beta}{n} (\ell(h,Z_i) - \ell(\fstar,Z_i))\right)\right] \dd \nutil(h)  \notag\\
& \leq
%\exp(2 \beta \ustar + \phi_{\beta,\lambda}( \sqrt{\ustar})) 
\exp\left( (\tfrac{r}{2} + 1)\beta\epsilonstar^2\right)
  \nutil( \calF_r^c) ~~~(\because \text{predictor condition of Assumption \ref{ass:BernsteinLightTail}}) \notag\\
& \leq 
2 
\exp\left( (\tfrac{r}{2} + 1)\beta\epsilonstar^2 - \beta \epsilonstar^2 r\right) 
~~~(\because \text{\Eqref{eq:nuFrlowerbound}}) \notag \\
& =
2 
\exp\left( - \tfrac{1}{2}(r - 2)\beta\epsilonstar^2\right).
\notag 
%\exp(2 \beta \ustar + \phi_{\beta,\lambda}( \sqrt{\ustar}) - \beta \epsilonstar^2 r).
%\label{eq:CUpperBound}
\end{align}

{\it \bf Step 3: Proof of $\mathrm{D}$.}

Next, we prove the condition $\mathrm{D}$.
Similarly to $\mathrm{C}$, we have that 
\begin{align*}
& \EE[\pi_\infty(\{W \in \calF_r : \calL(W) - \calL(\fstar) \geq 3 r \ustar \}) \boldone_{\calE_r}] \\
& = 
\EE\left[\frac{\int_{\calL(h) - \calL(\fstar) \geq 3 r \ustar} \exp(-\beta(\calLhat(h) - \calLhat(\fstar))) \dd \nutil(h)}{\int \exp(-\beta(\calLhat(h) - \calLhat(\fstar))) \dd \nutil(h)} \boldone_{\calE_r}\right] \\
& \leq
\EE\left[ \boldone_{\calE_r} %\exp(\tfrac{\CLRD^{s/2}+2}{2}\beta \ustar ) 
\exp\left( (\tfrac{r}{2} + 1)\beta\epsilonstar^2\right)
\int_{\calL(h) - \calL(\fstar) \geq 3 r \ustar} \exp(-\beta(\calLhat(h) - \calLhat(\fstar))) \dd \nutil(h)
\right] \\
& 
\leq 
\EE\left[ \boldone_{\calE_r} 
\exp\left( (\tfrac{r}{2} + 1)\beta\epsilonstar^2\right)
%\exp( \tfrac{\CLRD^{s/2}+2}{2} \beta \ustar) % + \phi_{\beta,\lambda}(\epsilonstar)) 
\int_{\calLhat(h) - \calLhat(\fstar) \geq r \ustar} \exp(-\beta(\calLhat(h) - \calLhat(\fstar))) 
%\exp\left(- \beta  r \ustar \right) 
\dd \nutil(h)
\right] \\
& ~~~~~~~(\because \text{condition $\mathrm{B}$ is satisfied on $\calE_r$})\\
& 
\leq 
\exp\left( (\tfrac{r}{2} + 1)\beta\epsilonstar^2 -r \beta \ustar \right)
%\exp\left[ - \tfrac{1}{2} (r-1 - \CLRD^{s/2} -2) \beta \ustar \right] 
\\
&
\leq 
\exp\left[ - \tfrac{1}{2} (r- 2) \beta \epsilonstar^2 \right].
\end{align*}

{\it \bf Step 4: Integrating all bounds of $\mathrm{A,B,C,D}$.}

%By taking integration, we have 
Finally, we integrate all bounds to obtain an excess risk bound.
\begin{align}
& \EE_{D^n}\left[\int \calL(W) - \calL(\fstar) \dd \pi_\infty(W) \right] \notag \\
= & 
\EE_{D^n}\left[\int_0^\infty  \pi_\infty(\{ W \in \cH:  \calL(W) - \calL(\fstar) > t\})  \dd t\right] \notag\\
= & 
\int_0^\infty  \EE_{D^n}\left[ \pi_\infty(\{ W \in \cH:  \calL(W) - \calL(\fstar) > t\}) \right] \dd t~~~(\text{by Fubuni's theorem}) \notag\\
= & 
3 \ustar
+ 
\int_1^\infty  3 \ustar \EE_{D^n}\left[ \pi_\infty(\{ W \in \cH:  \calL(W) - \calL(\fstar) > 3 r \ustar \}) \right] \dd r \notag\\
\leq 
& 
3 \ustar
+ 
3 \ustar  \int_1^\infty  %\min\{\mathrm{A + B + C + D},1\} 
 \EE_{D^n}\left\{ \boldone_{\calE_r^c}  + \boldone_{\calE_r} [\pi_\infty(\calF_r^c) + \pi_\infty(
\{ W \in \calF_r:  \calL(W) - \calL(\fstar) > 3 r \ustar \})]\right\}
\dd r\notag \\
\leq & 
3 \ustar
+
3 \ustar 
 \int_1^\infty  
\min\left\{
2 e^{- c' \min\{\beta \epsilonstar^2, n \epsilonstar^{2(2-s)}\}r}
+
3 \exp\left( - \tfrac{1}{2}(r - 2)\beta\epsilonstar^2\right)
,1 
\right\}
\dd r \notag\\
\lesssim & \ustar,
\label{eq:LDiffbound}
%\text{(Excess risk)} \lesssim \ustar.
\end{align}
where in the last inequality, we used $\beta \epsilonstar^2 \geq \log(2)$ by \Eqref{eq:epsilonstarPhiBetaLamCond} and $n\epsilonstar^{2(2-s)} \geq 1$ by the definition of $\epsilonstar$.
%!!!!!!!!!!!!!!!!!!!!!!!!!!!!!!!!!!!!!!!!!!!!!!!!!!!!!!!!!!!!
%
%For each $h_i$, 
%\begin{align*}
%P\left(|\calL(h_i) - \calL(\fstar) - \calLhat(h_i) - \calLhat(\fstar)| \geq 
%\sqrt{\frac{2 \Var[ \ell(g_{h_i},Z) - \ell(\fstar,Z)]  }{n} t} + \frac{M}{n}t \right) \leq 
%\exp(-t).
%\end{align*}
%Here, note that by the Bernstein condition, it holds that 
%\begin{align*}
%\Var[ \ell(g_{h_i},Z) - \ell(\fstar,Z)]  \leq B(\calL(g_{h_i}) - \calL(\fstar))^{\betas},
%\end{align*}
%and we obtain 
%\begin{align*}
%|\calL(h_i) - \calL(\fstar) - \calLhat(h_i) - \calLhat(\fstar)| 
%& \leq 
%\sqrt{\frac{2B (\calL(g_{h_i}) - \calL(\fstar))^{\betas}  }{n} t} + \frac{M}{n}t  \\
%& \leq 
%\sqrt{\frac{2B (\calL(g_{h_i}) - \calL(\fstar))^{\betas}  }{n} t} + \frac{M}{n}t \\
%& \leq 
%\frac{1}{4}(\calL(g_{h_i}) - \calL(\fstar))
%+ \left( \frac{2 B }{n} t\right)^{\frac{1}{2-s}} + \frac{M}{n} t,
%\end{align*}
%with probability $1 - e^{-t}$.
%By taking the uniform bound, the following inequality holds uniformly, for all $i=1,\dots,N_\epsilon$, 
%\begin{align*}
%|\calL(h_i) - \calL(\fstar) - \calLhat(h_i) - \calLhat(\fstar)| 
%\leq 
%\frac{1}{4}(\calL(g_{h_i}) - \calL(\fstar))
%+ \left( \frac{4 B  \log( N_\epsilon) }{n} \right)^{\frac{1}{2-s}} + \frac{2 M \log(N_\epsilon)}{n} ,
%\end{align*}
%with probability $1 - e^{- \log(N_\epsilon) }$, where we took $t = 2 \log(N_\epsilon)$.

{\it \bf Step 5: Evaluation of $\ustar$.}

Based on the arguments above, our goal is reduced to evaluating $\ustar$.
We note that 
\begin{align*}
\EE[\hat{r}(u)^2]
& 
\leq u^{s} + \Rbar \psi_{r,\epsilonstar}(u)
\leq u^s + \Rbar \EE\left[\int_{0}^{\hat{r}(u)} \sqrt{\frac{ \beta \epsilonstar^2 r + \left(\epsilon'/\lambdabeta^{1/2} (B(1+RD))^{-1} \right)^{-2\alphatil/\theta}}{n}} \dd \epsilon' \right] \\
& 
\lesssim
u^s + 
\sqrt{\frac{ \beta \epsilonstar^2 r}{n} \EE[\hat{r}^2(u)]}
 +
\frac{1}{\sqrt{n}}
\lambdabeta^{\alphatil/\theta}
\EE[(\hat{r}(u))^{1-\alphatil/\theta} ] \\
& 
\leq 
u^s + 
\sqrt{\frac{ \beta \epsilonstar^2 r}{n} \EE[\hat{r}^2(u)]}
 +
\frac{1}{\sqrt{n}}
\lambdabeta^{\alphatil/\theta}
(\EE[\hat{r}(u)^2])^{(1-\alphatil/\theta)/2}.
%where $\hat{r}(u) := \sup\{ \|\ell(f_h,\cdot) - \ell(\fstar,\cdot)\|_n^2 \mid h \in \calF_r,~\calL(h) - \calL(\fstar) \leq u \}$.
\end{align*}
Therefore, we have that 
\begin{align*}
\EE[\hat{r}(u)^2] \lesssim 
u^s \vee \frac{ \beta}{n}\epsilonstar^2 r
\vee n^{-\frac{1}{1+\alphatil/\theta}} \lambdabeta^{\frac{2\alphatil/ \theta}{1+\alphatil/\theta}}.
\end{align*}
This gives that 
\begin{align*}
\psi_{r,\epsilonstar}(u) 
\lesssim
& \EE\left[\int_{0}^{\hat{r}(u)} \sqrt{\frac{ \beta \epsilonstar^2 r + \left(\epsilon'/\lambdabeta^{1/2} (B(1+RD))^{-1} \right)^{-2\alphatil/\theta}}{n}} \dd \epsilon' \right] \\
\lesssim
& 
\sqrt{\frac{ \beta \epsilonstar^2 r}{n} \EE[\hat{r}(u)^2] }
+ \frac{1}{\sqrt{n}}
\lambdabeta^{\alphatil/\theta}
(\EE[\hat{r}(u)^2])^{(1-\alphatil/\theta)/2} \\
%\lesssim
%& 
%\sqrt{\frac{ \beta \epsilonstar^2 r}{n}
%}
%+ \frac{1}{\sqrt{n}}
%\lambdabeta^{\alphatil/\theta}
%(\EE[\hat{r}(u)^2])^{(1-\alphatil/\theta)/2}
%\\
\lesssim &
%& 
%% \EE[\hat{r}(u)^2] + 
%u^s + 
\frac{ \beta}{n} \epsilonstar^2 r  + n^{-\frac{1}{1+\alphatil/\theta}} \lambdabeta^{\frac{2\alphatil/ \theta}{1+\alphatil/\theta}} 
+
u^{s/2}\sqrt{\frac{ \beta \epsilonstar^2 r}{n} + n^{-\frac{1}{1+\alphatil/\theta}} \lambdabeta^{\frac{2\alphatil/ \theta}{1+\alphatil/\theta}}}
~~~~(\because \text{Young's inequality}).
%\\
%\lesssim
%&
%u^s +\frac{ \beta}{n}\epsilonstar^2 r
%+ n^{-\frac{1}{1+\alphatil/\theta}} \lambdabeta^{\frac{2\alphatil/ \theta}{1+\alphatil/\theta}}.
\end{align*}

Here, we let the upper bound in the right hand side as $\bar{\psi}_{r,\epsilonstar}$, then 
we can easily show the condition \eqref{eq:barpsiConditions}, that is, 
$\bar{\psi}_{r,\epsilonstar}(4u) \leq 2 \bar{\psi}_{r,\epsilonstar}(u)$ and 
 $\frac{\bar{\psi}_{r,\epsilonstar}(u r)}{u r} \leq \frac{\bar{\psi}_{1,\epsilonstar}(u)}{u}$.
%\begin{align*}
%\bar{\psi}_{r,\epsilonstar}(u)
%\lesssim
%& 
% \EE[\hat{r}(u)^2] + 
%u^s + \frac{ \beta}{n} \epsilonstar^2 r  + n^{-\frac{1}{1+\alphatil/\theta}} \lambdabeta^{\frac{2\alphatil/ \theta}{1+\alphatil/\theta}} 
%+
%u^{s/2}\sqrt{\frac{ \beta \epsilonstar^2 r}{n} + n^{-\frac{1}{1+\alphatil/\theta}} \lambdabeta^{\frac{2\alphatil/ \theta}{1+\alphatil/\theta}}}
%~~~~(\because \text{Young's inequality}).
%\end{align*}
Finally, by the definition of $\ustar$, we obtain that 
\begin{align*}
\ustar \lesssim 
\epsilonstar^2 \vee 
\left(\frac{ \beta }{n} \epsilonstar^2  + n^{-\frac{1}{1+\alphatil/\theta}} \lambdabeta^{\frac{2\alphatil/ \theta}{1+\alphatil/\theta}} \right)^{\frac{1}{2-s}}
\vee \frac{1}{n}.
\end{align*}
This yields the assertion.
\end{proof}

\subsection{Proof of fast rate for classification (Theorem \ref{thm:ClassificationFastRate})}
\label{sec:ProofOfClassificationError}

\begin{proof}

Let the convergence rate in the right hand side of \Eqref{eq:ExpectedLossConvRate} in Theorem \ref{thm:ExcessRiskConvRate} be $\ustar$.

Since both $\bar{W}_2(a)$ and $\bar{W}_1(w)$ are bounded and the activation function $\sigma$ is included in the H\"older class $\calC^m(\Real)$, the model $\{f_W \mid W \in \cH\}$ is also included in the H\"older class $\calC^{m}(\calX)$ with regularity $m$ and especially it is included in the Sobolev space $W_2^m(\calX)$: %\ts{Some more explanations}
\begin{align*}
f_W \in W_2^m(\calX). %\calC^{m}(\Omega).
\end{align*}
Moreover, since the logistic loss is $C^\infty$-class and its derivative up to $m$-th order is upper bounded,
the function 
$
x \mapsto \ell(f_W(x), y) 
$
is also included in $W^{m}(\calX)$ for all $y\in \{\pm 1\}$.
Therefore, $\hat{h}_W(x) := \EE_{Y|x}[ \ell(f_W(x), Y)] (=h(f_W(x)|x))$ is also included in $W^{m}(\calX)$.
Moreover, $\|\hat{h}_W\|_{W_2^m(\calX)} \leq C $ uniformly over all $W \in \cH$.
%We denote $\hstar(x) := \inf_{u \in \Real} h(u|x)$.

 If $X_0$ and $X_1$ are a pair of quasi-normed spaces which are continuously embedded in a linear Hausdorff space $\calG$, their $K$-functional is defined for any $f \in X_0+X_1$ by
$$
K(f,t; X_0,X_1) := \inf_{f=f_0 + f_1} \|f_0\|_{X_0} + \|f_1\|_{X_1}.
$$
For each $0 < \theta < 1$, $0 < p \leq \infty$, the {\it interpolation space} $[X_0,X_1]_{q,\theta}$ is the set of all functions $f \in X_0 + X_1$ for which
$$
\|f\|_{[X_0,X_1]_{q,\theta}} := \left(\int_0^\infty (t^{-\theta}K(f,t;X_0,X_1))^q \frac{\dd t}{t} \right)^{1/q}
$$
is finite. For $q = \infty$, the right hand side is properly modified in a usual manner.
As shown by \cite{devore1993besov}, it holds that 
$$
[L_2(\Omega),W_2^m(\calX)]_{1,d/2m} = B_{2,1}^{d/2}(\calX),
$$
where $L_2(\calX)$ is the $L_2$-space on $\calX$ with respect to the Lebesgue measure and 
$B_{2,1}^{d/2}(\calX)$ is the Besov space defined on $\calX$ (see \cite{devore1993besov} for its definition). 
Note that $d/2m < 1$ by the assumption.
From this property, combined the extension theorem of \cite{devore1993besov} and the embedding property of the Besov space \cite{triebel1983theory}, we have that $B_{2,1}^{d/2}(\calX) \hookrightarrow L_\infty(\calX)$. 
Under this condition, 
it is known that the following inequality holds 
\begin{align*}
\|\hat{h}_W - \hstar\|_\infty \leq 
& C \|\hat{h}_W - \hstar\|_{L_2(\calX)}^{1-\frac{d}{2m}}
\|\hat{h}_W - \hstar\|_{W_2^m(\calX)}^{\frac{d}{2m}} \\ \leq 
& C c_0^{-(1-d/2m)} \|\hat{h}_W - \hstar\|_{\LPiPx}^{1-\frac{d}{2m}}
\|\hat{h}_W - \hstar\|_{W_2^m(\calX)}^{\frac{d}{2m}},
\end{align*}
(see \cite{Book:Bennett+Sharpley:88,COLT:Steinwart+etal:2009}).
Combining this with the assumption $\hstar \in W_2^m(\calX)$ and the fact $\|\hat{h}_W\|_{W_2^m(\calX)} \leq C $,
if $\|\hat{h}_W - \hstar\|_{\LPiPx} \leq \epsilon$ for sufficiently small $\epsilon$, we have an $L_\infty$-norm bound as 
$\|\hat{h}_W - \hstar \|_\infty \leq C' \epsilon^{1-\frac{d}{2m}}$.
%Since $|\eta(x) - 1/2| \geq \delta$ (strong low noise condition), 
Thus, if we choose $\epsilon$ so that $\epsilon^{1-\frac{d}{2m}} = \Theta(\delta)$ and 
let $W$ satisfy $\|\hat{h}_W - \hstar\|_{\LPiPx} \leq \epsilon$,
then we can have
\begin{align*}
\|\hat{h}_W - \hstar\|_\infty < \delta/2.
\end{align*}
Then, by the assumption that $\hstar(x) \leq \log(2) - \delta$, it holds that 
\begin{align*}
\hat{h}_W(x)  < \log(2) - \delta/2~~~(\text{a.s.}),
\end{align*}
which indicates that 
\begin{align*}
P_X(\sign(f_W(X)) =\gstar(X))=1.
\end{align*}
Therefore, we only need to bound the quantity $\|\hat{h}_{W_k} - \hstar\|_{\LPiPx}^2$ for $W_k \sim \pi_k$.
%\begin{align}
%\EE\left[\int \|\hat{h}_W - \hstar\|_{\LPiPx}^2 \dd \pi_\infty(W) \right].
%\label{eq:EhhatHstarFirstBound}
%\end{align}

Here, we show that the Bernstein condition (Assumption \ref{ass:BernsteinLightTail}) is satisfied with $s=1$ under Assumptions %\ref{ass:LossBoundedness} and 
\ref{ass:StrongLowNoise}.
By Assumptions %\ref{ass:LossBoundedness} and 
\ref{ass:StrongLowNoise} and $\|f_W\|_\infty \leq R$ for any $W \in \cH$ by the definition of the clipping operator,
it holds that $\|\fstar\|_\infty \leq R$.
Therefore, Lemma \ref{lem:LogisticBernstein} yields the Bernstein condition with $s=1$ and $C_B = 4 + 3R$. % 5R\log(2)$.
%it holds that $|\hstar(x)| = |h(\fstar(x)|x)|$ is also bounded by $R$ (a.s.).
%Here, we fix $x \in \calX$, and assume $\fstar(x) < 0$, that is, $P(Y=1|X=x) < 1/2$, without loss of generality.
%We write $p = P(Y=1|X=x) $.
%Combining the results from (i), (ii) and (iii), we have obtained that the Bernstein condition is satisfied with $s=1$ and $C_B = 5R\log(2)$.
Therefore, 
%The left hand side of \Eqref{eq:EhhatHstarFirstBound} is bounded by
$\|\hat{h}_W - \hstar\|_{\LPiPx}^2$ can be bounded as 
\begin{align*}
& %\EE\left[ \int 
\|\hat{h}_W - \hstar\|_{\LPiPx}^2 \\
%\dd \pi_\infty(W) \right] \\
\leq & 
%\EE\left[ \int 
\EE_Z[(\ell(f_W,Z) - \ell(\fstar,Z))^2 ] \\
%\dd \pi_\infty(W) \right] \\
\leq & 
C_B 
%\EE\left[ \int 
(\calL(W) - \calL(\fstar))^s  
%\dd \pi_\infty(W) \right]  \\
= 
C_B
% \EE\left[ \int 
(\calL(W) - \calL(\fstar)),
% \dd \pi_\infty(W) \right] ,
\end{align*}
where we used Jensen's inequality in the first inequality, and we applied $s=1$ in the last equality.  %Jensen's inequality under the assumption $0 < s < 1$ and $\calL(W) - \calL(\fstar) \geq 0$.
%We have already shown that for sufficiently large $n$, $\beta$ and $\lambda$, the right hand side can be bounded by  
%$\delta/2$.
First, we consider the stationary distribution.
For any $\epsilon' > \ustar$, we have already shown in the proof of Theorem \ref{thm:ExcessRiskConvRate} (See \Eqref{eq:LDiffbound}) that
\begin{align}
%& \EE\left[P^{\pi_\infty}( \|\hat{\eta}_W - \hat{\eta}_{W^*}\|_{\LPiPx} \geq  \epsilon) \right] \\ %\delta^{s}) \right] \\
%\leq &
%\EE\left[P^{\pi_\infty}(  \EE_Z[(\ell(f_W,Z) - \ell(\fstar,Z))^2 ] \geq   \epsilon) \right] \\ %\delta^{s}) \right] \\
%\leq &
%\EE\left[P^{\pi_\infty}(   (\calL(W) - \calL(\fstar))^s  \geq \epsilon) \right] \\ %\geq  \delta^{s}) \right] \\ 
%\leq 
&
\EE_{D^n}\left[\pi_\infty( \{W \in \cH \mid  \calL(W) - \calL(\fstar) \geq  \epsilon'\}) \right] \notag \\
\leq &  C \exp(- c \beta \ustar \times (\epsilon' /\ustar ))
= C \exp(- c \beta \epsilon' ).
%= C \exp(- c' \beta \delta^{1/(1-d/2m)} )
%= C \exp(- c' \beta \delta^{2m/(2m-d)} ).
\label{eq:piinftyCexpBetaExpDash}
\end{align}
Next, we consider the intermediate solution $W_k$. 
Suppose that the sample size $n$ is sufficiently large and $\lambda$ is appropriately chosen 
with sufficiently large $\beta$ so that 
$u^* \ll \delta^{2m/(2m-d)}$\footnote{This is a more precise meaning of  the sentence ``the sample size $n$ is sufficiently large and $\lambda$ is appropriately chosen'' in the statement.}.
The probability of misclassification is bounded by 
%the probability is bounded by 
%\begin{align*}
%& \EE[P_{\pi_k}(\{ W_k \in \cH \mid P_X(\sign(f_{W_k}(X)) = \sign(\fstar(X))) \neq 0\})]  \\
%\leq 
%& \EE\left[P_{W_k \sim \pi_k}[ \calL(W_k) - \calL(\fstar)  \geq  \epsilon^{1/s} ]\right] \\
%\leq &
% \EE\left[\EE_{W_k \sim \pi_k}[ \calL(W_k) - \calL(\fstar)  ]/\epsilon^{1/s} \right] \\
%\leq &
% \EE\left[ \frac{\EE[ \calL(W_k) - \calL(W^{\pi_\infty})  ]  + \EE_{W \sim \pi_\infty}[ \calL(W) - \calL(\fstar)  ]}{ \epsilon^{1/s}}\right] \\
%\leq &
%\frac{\Xi_{k} + c \ustar}{c\delta} \simeq \frac{\exp(- \Lambda_\eta^* \eta k) + \frac{\sqrt{\beta}}{\Lambda^*_0} \eta^{1/2-\kappa} + \ustar}{\epsilon^{1/s}}
%\lesssim \frac{\exp(- \Lambda_\eta^* \eta k) + \frac{\sqrt{\beta}}{\Lambda^*_0} \eta^{1/2-\kappa} + \ustar}{\delta^{2m/(s(2m-d))}}.
%\end{align*}
%%Therefore, 
%
\begin{align*}
& \EE[\pi_k(\{ W_k \in \cH \mid P_X(\sign(f_{W_k}(X)) = \sign(\fstar(X))) \neq 1\})]  \\
\leq 
& \EE\left[P_{W_k \sim \pi_k}[ \calL(W_k) - \calL(\fstar)  \geq  \epsilon/C_B ]\right] \\
%\leq 
%& \EE\left[P_{W_k \sim \pi_k}\{ \calL(W_k) - \EE_{W\sim \pi_\infty}[\calL(W)] - (\calL(\fstar)- \EE_{W\sim \pi_\infty}[\calL(W)])   \geq  \epsilon \}\right] \\
=
& \EE\left[P_{W_k \sim \pi_k,W\sim \pi_\infty}[ \calL(W_k) - \calL(W) - (\calL(\fstar)- \calL(W))   \geq  \epsilon/C_B ]\right] \\
\leq
& \EE\left[P_{W_k \sim \pi_k,W\sim \pi_\infty}[ \calL(W_k) - \calL(W)    \geq  \epsilon/(2C_B) ] \right] 
+
\EE\left[ P_{W\sim \pi_\infty}[ \calL(W) - \calL(\fstar)  \geq \epsilon/(2C_B)] \right] 
\\
\leq 
& \EE\left[\EE_{W_k \sim \pi_k,W\sim \pi_\infty}[ \calL(W_k) - \calL(W)   ]/ (\epsilon/(2C_B))  \right] 
+
\EE\left[ P_{W\sim \pi_\infty}[ \calL(W) - \calL(\fstar) \geq \epsilon/(2C_B)] \right] 
%\EE\left[ \boldone[ \calL(\fstar)- \EE_{W\sim \pi_\infty}[\calL(W)]  \geq  \epsilon/(2C_B)] \right]  
\\
\lesssim &
%\frac{\Xi_{k} + c \ustar}{c\delta} \simeq 
\frac{\Xi_k}{\delta^{2m/(2m-d)}} + \exp(- c' \beta \delta^{2m/(2m-d)}),
\end{align*}
where we used $\epsilon = \Theta(\delta^{2m/(2m-d)})$ and \Eqref{eq:piinftyCexpBetaExpDash} in the last inequality.
Therefore, for a fixed $\delta$, we can obtain the Bayes classifier with high probability by setting $\eta$ sufficiently small and taking sufficiently large $k$.

\vspace{0.2cm}
\noindent {\bf Making the first term as large as the second term.}

We see that the first term in the right hand side is coming from the bound of $\EE\left[P_{W_k \sim \pi_k,W\sim \pi_\infty}[ \calL(W_k) - \calL(W)    \geq  \epsilon/(2C_B) ] \right]$.
To bound this, we used the following bound:
\begin{align*}
P_{W_k \sim \pi_k,W\sim \pi_\infty}[ \calL(W_k) - \calL(W)    \geq  \epsilon/(2C_B) ]
& \leq \EE_{W_k \sim \pi_k,W\sim \pi_\infty}[ \calL(W_k) - \calL(W)   ]/ (\epsilon/(2C_B)) \\
& \lesssim \frac{\Xi_k}{\delta^{2m/(2m-d)}},
\end{align*}
almost surely.
Therefore, if $\Xi_k$ is sufficiently small such that $\frac{\Xi_k}{\delta^{2m/(2m-d)}} \ll 1$, 
then we have
\begin{align*}
P_{W_k \sim \pi_k,W\sim \pi_\infty}[ \calL(W_k) - \calL(W)    \geq  \epsilon/(2C_B) ] \leq 1/2.
\end{align*}
Therefore, by running the algorithm $S$-times and picking up the beset $W_k$ in terms of the validation error
(write it as $W_k^{(S)}$), 
then we have that 
\begin{align*}
P_{W_k \sim \pi_k,W\sim \pi_\infty}[ \calL(W_k^{(S)}) - \calL(W)    \geq  \epsilon/(2C_B) ] \leq 1/2^S.
\end{align*}
Thus, for sufficiently large $S$ such that the right hand side can be smaller than the second term $\exp(- c' \beta \delta^{2m/(2m-d)})$, we have that 
\newcommand{\relmiddle}[1]{\mathrel{}\middle#1\mathrel{}}
\begin{align*}
& \EE
\left\{
P_{W_k^{(S)}}\left[ P_X(\sign(f_{W_k^{(S)}}(X)) = \sign(\fstar(X))) \neq 1 \relmiddle| D^n \right]\right\}  
\lesssim 
\exp(- c' \beta \delta^{2m/(2m-d)}).
\end{align*}
\end{proof}

\begin{Lemma}\label{lem:LogisticBernstein}
Suppose that $\|\fstar\|_\infty \leq R$ and $\sup_W \|f_W\|_\infty \leq R$.
Then, the logistic loss satisfies the Bernstein condition with $s=1$ and $C_B = 4 + 3 R$. %5R\log(2)$.
\end{Lemma}

\begin{proof}

%{lem:LogisticBernstein}
%Here, we show that the Bernstein condition (Assumption \ref{ass:BernsteinLightTail}) is satisfied with $s=1$ under Assumptions %\ref{ass:LossBoundedness} and 
%\ref{ass:StrongLowNoise}.
%By Assumptions %\ref{ass:LossBoundedness} and 
%\ref{ass:StrongLowNoise} and 
Since $\|\fstar\|_\infty \leq R$, % and $\|f_W\|_\infty \leq R$ for any $W \in \cH$ by the assumption, % by the definition of the clipping operator,
it holds that $|\hstar(x)| = |h(\fstar(x)|x)|$ is also bounded by $R$ (a.s.).
Here, we fix $x \in \calX$ and write $p = P(Y=1|X=x)$.
By the optimality of $\fstar(x)$, we have that $p = \frac{1}{1 + \exp(-\fstar(x))}$.
Accordingly, we denote $q = \frac{1}{1 + \exp(-f_W(x))}$ for any $W \in \cH$.

Then, what we need to show is that 
\begin{align}\label{eq:DesiredLowerBoundCB}
p \left[\log\left(\frac{p}{q}\right) \right]^2 + (1-p) \left[\log\left(\frac{1-p}{1-q}\right) \right]^2
\leq C_B \left\{ p \log\left(\frac{p}{q}\right) + (1-p)\log\left(\frac{1-p}{1-q}\right)  \right\}.
\end{align}
%By the convexity of $-\log(\cdot)$, the right hand side 
The right hand side can be rewritten as
\begin{align*}
 p \left[ \log\left(\frac{p}{q}\right) + \frac{1}{p}(q-p)\right] + 
(1-p)\left[ \log\left(\frac{1-p}{1-q}\right) - \frac{1}{1-p}(q-p) \right],
\end{align*}
and by noticing the convexity of $-\log(\cdot)$, each term of the right hand side is non-negative. 
We show the inequality \eqref{eq:DesiredLowerBoundCB} by showing 
\begin{align}
& \left[\log\left(\frac{p}{q}\right) \right]^2 \leq C_B  \left[ \log\left(\frac{p}{q}\right) + \frac{1}{p}(q-p)\right],  
\label{eq:FirstBoundCB}
\\
& 
\left[\log\left(\frac{1-p}{1-q}\right) \right]^2 \leq C_B \left[ \log\left(\frac{1-p}{1-q}\right) - \frac{1}{1-p}(q-p) \right].
\label{eq:SecondBoundCB}
\end{align}
Without loss of generality, we may assume $p \leq 1/2$.

\noindent 
{\it \bf Step 1: Proof of \Eqref{eq:FirstBoundCB}}.
We show the inequality by considering the following four settings (i) $p/2 \leq q \leq p$, 
(ii) $q < p/2$, 
(iii) $p\leq q \leq 2 p$, (iv) $2p < q$. 
Let $f_1(q) = \log(p/q) + \frac{1}{p}(q-p)$ and $f_2(q)= \left[ \log(p/q)\right]^2$.

\noindent (i) ($p/2 \leq q \leq p$)
 Since $f_1(q)$ is a convex function satisfying $ \frac{\dd^2}{\dd q^2}f_1(q) = 1/q^2 \geq 1/p^2~(\forall q \leq p)$, $f_1(p) = 0$ and $f_1(q) \geq 0$, 
it holds that 
$f_1(q) \geq \frac{1}{p^2}(q-p)^2$ for all $q \leq p$.
On the other hand, $0 \leq \log(p/q) \leq \frac{2}{p}(p-q)$ for $p/2 \leq q \leq p$, it holds that $f_2(q) \leq \frac{4}{p^2} (q-p)^2~(p/2 \leq \forall q \leq p)$.
These inequalities yield 
$$
4 f_1(q) \geq f_2(q)~~(p/2 \leq \forall q \leq p).
$$

\noindent (ii)  ($q < p/2$).
Since $f_1(q) \geq \frac{1}{p^2} (p-q)^2~(\forall q \leq p)$
and $-\frac{1}{p}(q-p) \leq 2 \frac{1}{p^2}(q-p)^2 \leq 2 f_1(q)~(\forall q \leq p/2)$, 
we have that 
$$
- \frac{1}{3}\log(p/q) \leq  \frac{1}{p}(q-p) \leq 0~~
%\log(p/q) \geq \frac{1}{p^2} (p-q)^2 - \frac{1}{p}(q-p)
%\geq 
~(\forall q \leq p/2).$$
Therefore, by the definition of $f_1$, we have 
$$
f_1(q) \geq \frac{2}{3} \log(p/q) \geq \frac{2}{3 \log(p/q)}  [\log(p/q) ]^2
\geq \frac{2}{3 \log(1+\exp(R))}f_2(q),
$$
where we used $p \leq 1$ and $q = \frac{1}{1 + \exp(-f_W(x))} \geq \frac{1}{1 + \exp(R)}$.

\noindent (iii) ($p\leq q \leq 2 p$).
In this setting, the convexity of $-\log(\cdot)$ gives $0 \geq \log(p/q) = -\log(q/p) \geq \frac{1}{p}(p - q)$. 
Therefore, it holds that $f_2(q) \leq \frac{1}{p^2}(p-q)^2$. On the other hand, 
since $f_2''(q) = \frac{1}{q^2} \geq \frac{1}{4 p^2}$, it holds that $f_1(x) \geq \frac{1}{4 p^2} (p-q)^2$.
Therefore, we have 
$$
4 f_1(q) \geq f_2(q)~~ (p\leq  \forall q \leq 2 p).
$$

\noindent (iv) ($2 p < q $).
By the convexity of $-\log(q)$, we have that 
\begin{align*}
 - \frac{1}{\log(2)} \log(q/p) 
&  \geq 
 \frac{1}{\log(2)} [- \log(q) +\log(p)] 
 \geq 
  \frac{1}{\log(2)} [ - \log(2p) - \frac{1}{2p}(q-2p) + \log(p)] \\
&   = 
\frac{1}{\log(2)} [ - \log(2) - \frac{1}{2p}(q-2p) ] 
= - 1 - \frac{1}{2 \log(2) p}(q-2p) \\
% \geq - \frac{1}{p \log(2)} (q-2p) - 1
&\geq - \frac{1}{p} (q-2p) - 1 = - \frac{1}{p} (q-p)~~~(\forall q > 2p),
\end{align*}
where we used $2\log(2) \geq 1$.
This yields that 
$$
f_1(q) \geq \left(1 - \frac{1}{\log(2)} \right) \log(p/q) = \frac{1 - \log(2)}{\log(2)} \log(q/p) \geq 0~~(\forall q > 2p).
$$
(Remember that $\log(2) < 1$).
Therefore, 
\begin{align*}
f_1(q) & \geq  %\frac{1 - \log(2)}{\log(2)} \log(q/p) \geq 
\frac{1 - \log(2)}{\log(2) \log(q/p)} [\log(q/p)]^2
\geq 
\frac{1 - \log(2)}{\log(2) \log(1+\exp(R))} [\log(q/p)]^2 \\
& =\frac{1 - \log(2)}{\log(2) \log(1+\exp(R))}  f_2(q)
~~~~(\forall q > 2p),
\end{align*}
where we used $q \leq 1$ and $p \geq \frac{1}{1 + \exp(R)}$.

\noindent {\it \bf Step 2: Proof of \Eqref{eq:SecondBoundCB}}.
This is shown completely in the same manner with the proof of \Eqref{eq:FirstBoundCB}
by setting $p \leftarrow 1-p$ and $q \leftarrow 1 - q$.

\noindent {\it \bf Step 3.}
Combining the results of Step 1 and Step 3, we have the equations \eqref{eq:FirstBoundCB} and \eqref{eq:SecondBoundCB} 
with $$
C_B = \max\left\{ 4, \frac{3}{2} \log(1 + \exp(R)), \frac{\log(2)}{1-\log(2)} \log(1 + \exp(R))\right\}
\leq 4 + 3 R, %\frac{\log(2)}{1-\log(2)} (\log(2) + R),
$$
where we used $\frac{3}{2} \leq \frac{\log(2)}{1-\log(2)} \leq 3$, 
$\log(1 + \exp(R)) \leq \log(2) + R$ by the Lipschitz continuity of $x \mapsto \log(1+\exp(x))$, and $\frac{\log^2(2)}{1-\log(2)} \leq 2$.
Therefore, by resetting $C_B = 4 + 3 R$, we obtain the assertion.
\end{proof}

\subsection{Derivation of the fast rate of regression (\Eqref{eq:RegressionFastRate})}\label{App:RegressionExcessRiskDerivation}

Since $\fstar$ is realized by $f_{\Wstar}$, $\|f_W\|_{\infty} \leq R$ for any $W \in \calH$ and $|\epsilon_i| \leq C$,
we have that 
$$
\ell(Y,f_W(X)) = (Y - f_W(X))^2 = (f_{\Wstar} + \epsilon - f_W(X))^2 \leq 2 [(f_{\Wstar}(X)  - f_W(X))^2 + \epsilon^2]
\leq 2 (4 R^2 + C^2).
$$
Therefore, the assumption $0 \leq \ell(Y,f_W(X)) \leq \Rbar~(\forall W \in \cH)$ (a.s.) is obtained by $\Rbar = 2 (4 R^2 + C^2)$. Other assumptions in

Write $\cHKtilt := \calH_{K^{\theta(\gamma+1)}}$.
As we have stated in the main text, we can show the ``bias'' and ``variance'' terms can be bounded as 
\begin{align*}
& \inf_{h \in \cHKtil: \calL(h) - \calL(\fstar) \leq \epsilon^2} \lambdabeta \|h\|^2_{\cHKtil}
\lesssim \lambdabeta \epsilon^{- \frac{2(1-\theta)}{\theta}}, \\
& - \log \nutil(\{h \in \cH: \|h\|_{\cH} \leq \epsilon \})
\lesssim (\epsilon/\lambdabeta^{1/2} )^{-\frac{2\alphatil}{1 - \alphatil}}.
\end{align*}
The variance term has been already evaluated in \Eqref{eq:SmallBallConcentProb}.
Now, we evaluate the bias term. By the definitions of $\cHKtil$, 
$\Wstar \in \cHKtil$ means that there exists $(a_k)_{k=0}^\infty$ such that 
\begin{align*}
\Wstar = \sum_{k=0}^\infty  \sqrt{\mu_k^{\theta(\gamma+1)}} a_k e_k,~~\text{and}~~\sum_{k=0}^\infty a_k^2 < \infty.
\end{align*}
% and $\cHKtilt$
%the norms $\|\cdot\|_{\cHKtil}$ and $\|\cdot\|_{\cHKtil^\theta}$, 
Here, we denote $Q := \sum_{k=0}^\infty a_k^2$. 
Now, let $\tilde{W} = \sum_{k=0}^N  \sqrt{\mu_k^{\theta(\gamma+1)}} a_k e_k$ for some $N \in \Natural$ as an approximator of $\Wstar$.
Then, its norm in $\cHKtil$ can be evaluated as 
\begin{align}\label{eq:WtilNormBound}
\|\tilde{W}\|^2_{\cHKtil} = \sum_{k=0}^N \mu_k^{-(\gamma+1)} \mu_k^{\theta(\gamma+1)} a_k^2 = 
\sum_{k=0}^N \mu_k^{(\theta - 1)(\gamma+1)} a_k^2.
\end{align}
We evaluate the discrepancy between $\Wstar$ and $\tilde{W}$ and evaluate its norm in $\calH$.
Since $\Wstar - \tilde{W} = \sum_{k=N+1}^\infty \sqrt{\mu_k^{\theta(\gamma+1)}} a_k e_k
$, its $\calH$-norm is given by 
\begin{align}\label{eq:WWtilDisc}
\|\Wstar - \tilde{W}\|_{\calH}^2 = \sum_{k=N+1}^\infty \mu_k^{\theta(\gamma+1)} a_k^2.
\end{align}
Note that $\calL(f_{\tilde{W}}) - \calL(\fstar) = \|f_{\tilde{W}} - \fstar \|^2_{\LPiPx} \leq (1+RD)^2 \|\tilde{W} - \Wstar\|_{\cH}^2$ by Lemma \ref{lem:LipshitzInfty}. 
Therefore, to ensure $\calL(f_{\tilde{W}}) - \calL(\fstar)\leq \epsilon^2$, it suffices to let $(1+RD)^2 \|\tilde{W} - \Wstar\|^2_{\cH }\leq \epsilon^2$.
By \Eqref{eq:WWtilDisc}, this means that $\sum_{k=N+1}^\infty \mu_k^{\theta(\gamma+1)} a_k^2 \leq \epsilon^2/(1+RD)^2$.
Here, note that 
\begin{align*}
\|\Wstar - \tilde{W}\|_{\calH}^2 = \sum_{k=N+1}^\infty \mu_k^{\theta(\gamma+1)} a_k^2  \leq \mu_{N+1}^{\theta(\gamma+1)} \sum_{k=N+1}^\infty a_k^2
\leq c_\mu^{\theta(\gamma + 1)} (N+2)^{- 2\theta(\gamma + 1)} Q.
\end{align*}
Hence, by setting $N \propto \epsilon^{- 1/[\theta(\gamma + 1)]}$, we can let $\calL(f_{\tilde{W}}) - \calL(\fstar) \leq \epsilon^2$. In this setting of $N$, 
by noticing $(k+1)^{-2 (\theta -1)(\gamma + 1)} = (k+1)^{2 (1 - \theta)(\gamma + 1)}$ is monotonically increasing with respect to $k$,
\Eqref{eq:WtilNormBound} gives that
\begin{align*}
 \|\tilde{W}\|_{\cHKtil}^2 & \leq \sum_{k=0}^N c_\mu^{(\theta -1)(\gamma + 1)} (k+1)^{-2 (\theta -1)(\gamma + 1)} a_k^2 \\
& \leq c_\mu^{(\theta -1)(\gamma + 1)} (N+1)^{2 (1 - \theta)(\gamma + 1) } Q 
\lesssim \epsilon^{- 2 (1 - \theta)/\theta },
\end{align*}
which gives the bias term bound.

%Therefore, we obtain the following excess risk bound.
%The covering number of $\calB_{\cHKtil}$ is bounded by 
%$$\tilde{\phi}(\epsilon') \lesssim (\epsilon'/\lambdabeta)^{-2\alphatil /\theta}$$
%$$
%\ustar \lesssim
%\max \left\{  \lambdabeta^{\frac{2\alphatil/\theta}{1 + \alphatil/\theta}} n^{-\frac{1}{1 + \alphatil/\theta}},
%\epsilonstar^2 \right\}.
%$$
%Finally, we have that 

Combining the bias and variance terms, we may choose $\epsilonstar$ as the infimum of $\epsilon$ such that 
$
\lambdabeta \epsilon^{- \frac{2(1-\theta)}{\theta}} + (\epsilon/\lambdabeta^{1/2} )^{-\frac{2\alphatil}{1 - \alphatil}}
\leq \beta \epsilon^2.
$
That is, we have that 
\begin{align*}
\epsilonstar^2 \lesssim \max\left\{\lambdabeta^{-\alphatil} \beta^{-(1-\alphatil)}, \left(\frac{\lambdabeta}{\beta}\right)^{\theta},n^{-\frac{1}{2-s}}\right\}
=  \max\left\{\lambda^{-\alphatil} \beta^{-1}, \lambda^{\theta},n^{-\frac{1}{2-s}}\right\}.
\end{align*}

\end{document}